%% file: main.tex
\documentclass[11pt]{article}

\usepackage{ifpdf}
\usepackage{amsmath,amssymb,amsthm}
\usepackage{libertine}
\usepackage[T1]{fontenc}
\usepackage{xcolor}
\usepackage{xspace}
\usepackage{enumitem}
\usepackage[sort,numbers]{natbib}
\usepackage{algorithm,algorithmicx,algpseudocode}
\usepackage[margin=1in,footskip=24pt]{geometry}
\usepackage{hyperref}
\hypersetup{
  colorlinks=true,
  linkcolor=blue,
  linktoc=all,
  citecolor=violet,
  bookmarksnumbered=true
}
\usepackage{mathtools}
\usepackage{cleveref}
\usepackage{framed}
\usepackage{accents}
\usepackage{xfrac}
\usepackage{marginnote}
\usepackage{thmtools}
\usepackage{thm-restate}
\reversemarginpar

\newtheoremstyle{mythmstyle}
      {6pt}   %
      {6pt}   
      {}            %
      {}            %
      {\bfseries}   %
      {. }          %
      {2.5pt}       %
      {\thmname{#1}\thmnumber{ #2}\thmnote{ \normalfont (#3)}}   %
\theoremstyle{mythmstyle}
\newtheorem{theorem}{Theorem}[section]\numberwithin{equation}{section}

\newtheorem{claim}[theorem]{Claim}

\newtheorem{definition}[theorem]{Definition}
\newtheorem{fact}[theorem]{Fact}

\newtheorem{lemma}[theorem]{Lemma}

\newtheorem{proposition}[theorem]{Proposition}
\newtheorem{prop}[theorem]{Proposition}
\newtheorem*{remark}{Remark}

\newcommand{\AlgorithmName}[1]{\label{alg:#1}}
\newcommand{\AppendixName}[1]{\label{app:#1}}
\newcommand{\ClaimName}[1]{\label{clm:#1}}

\newcommand{\DefinitionName}[1]{\label{def:#1}}
\newcommand{\EquationName}[1]{\label{eq:#1}\text{}}

\newcommand{\FactName}[1]{\label{fact:#1}}

\newcommand{\LemmaName}[1]{\label{lem:#1}}

\newcommand{\PropositionName}[1]{\label{prop:#1}}

\newcommand{\SectionName}[1]{\label{sec:#1}}
\newcommand{\SubsectionName}[1]{\label{sub:#1}}
\newcommand{\TheoremName}[1]{\label{thm:#1}}
    
\newcommand{\Algorithm}[1]{Algorithm~\ref{alg:#1}}
\newcommand{\Appendix}[1]{Appendix~\ref{app:#1}}
\newcommand{\Claim}[1]{Claim~\ref{clm:#1}}

\newcommand{\Definition}[1]{Definition~\ref{def:#1}}
\newcommand{\Equation}[1]{\eqref{eq:#1}}

\newcommand{\Fact}[1]{Fact~\ref{fact:#1}}

\newcommand{\Lemma}[1]{Lemma~\ref{lem:#1}}

\newcommand{\Proposition}[1]{Proposition~\ref{prop:#1}}

\newcommand{\Section}[1]{Section~\ref{sec:#1}}
\newcommand{\Subsection}[1]{Subsection~\ref{sub:#1}}
\newcommand{\Theorem}[1]{Theorem~\ref{thm:#1}}
    
\newenvironment{proofof}[1]{\begin{proof}[\rm\textit{Proof} \,(of #1)]}{\end{proof}}
\newenvironment{proofsketch}{\begin{proof}[\rm\textit{Proof sketch}]}{\end{proof}}

\usepackage{color}

\newcommand{\doob}{\accentset{*}{\Delta}}
\newcommand{\AnytimeValue}{\operatorname{AnytimeNormRegret}}

\newcommand{\bN}{\mathbb{N}}        

\newcommand{\bR}{\mathbb{R}}
\newcommand{\bZ}{\mathbb{Z}}

\newcommand{\cA}{\mathcal{A}}
\newcommand{\cB}{\mathcal{B}}

\newcommand{\cF}{\mathcal{F}}
\newcommand{\cG}{\mathcal{G}}

\newcommand{\inner}[2]{\langle\, #1 ,\, #2 \,\rangle}
\newcommand{\abs}[1]{\lvert #1 \rvert}
\newcommand{\Abs}[1]{\left\lvert #1 \right\rvert}

\newcommand{\smallfrac}[2]{{\textstyle \frac{#1}{#2}}}

\newcommand{\argmax}{\operatornamewithlimits{arg\,max}}
\newcommand{\eps}{\epsilon}

\newcommand{\set}[1]{\left \{ #1 \right \}}                     
\newcommand{\setst}[2]{\left\{\; #1 \,:\, #2 \;\right\}}        
\newcommand{\smallsetst}[2]{\{\; #1 \,:\, #2 \;\}}        
\newcommand{\card}[1]{\abs{#1}}

\newcommand{\prob}[1]{\operatorname{Pr}\left[\,#1\,\right]}               
\newcommand{\probg}[2]{\operatorname{Pr}\left[\,#1 \:\mid\: #2\,\right]}  
\newcommand{\expect}[1]{\operatorname{E}\left[\,#1\,\right]}              
\newcommand{\smallexpect}[1]{\operatorname{E}[\,#1\,]}              
\newcommand{\expectg}[2]{\operatorname{E}\left[\,#1 \:\mid\: #2\,\right]}  
\newcommand{\Var}[1]{\operatorname{Var}\left[\,#1\,\right]}

\renewcommand{\th}{\ifmmode{^{\textrm{th}}}\else{\textsuperscript{th}\ }\fi}

\algnewcommand{\MyState}[1]{\State
\parbox[t]{\dimexpr\linewidth-\ALG@thistlm}{\hangindent=0pt\strut\hangafter=1#1\strut}}
\makeatother
\algdef{SxN}{Myprocedure}{EndMyprocedure}[2]{\algorithmicprocedure\ \textproc{#1}\ifthenelse{\equal{#2}{}}{}{(#2)}}
\algdef{SxN}{MyFor}{MyEndFor}[1]{\algorithmicfor\ #1\ \algorithmicdo}

\newcommand{\dd}{\,\mathrm{d}}

\DeclareMathOperator{\erf}{erf}
\DeclareMathOperator{\erfi}{erfi}
\DeclareMathOperator{\Reg}{\mathrm{Regret}}
\DeclareMathOperator{\Ber}{Ber}

\newcommand{\Regret}[1]{\mathrm{Regret}(#1)}
\newcommand{\ContReg}[1]{\mathrm{ContRegret}(#1)}
\newcommand{\PseudoContReg}[1]{{\mathrm{ContRegret}}(#1)}

\algrenewcomment[1]{\(\triangleright\) #1}
\algnewcommand{\LineComment}[1]{\State \(\triangleright\) #1}
\newcommand{\diffwrt}[1]{\frac{\mathrm{d}}{\mathrm{d}#1}}

\newcommand{\ind}{\mathbf{1}}

\makeatletter
\renewcommand{\paragraph}{%
  \@startsection{paragraph}{4}%
  {\z@}{.7ex \@plus 1ex \@minus .2ex}{-1em}%
  {\normalfont\normalsize\bfseries}%
}
\makeatother

\newenvironment{smallbmatrix}                          
{\left[\begin{smallmatrix}}
{\end{smallmatrix}\right]}

\begin{document}

\title{Optimal anytime regret with two experts}
\author{
Nicholas J. A. Harvey
    \thanks{Email: \texttt{nickhar@cs.ubc.ca}. University of British Columbia, Department of Computer Science.} \and
Christopher Liaw
    \thanks{Email: \texttt{cvliaw@cs.ubc.ca}. University of British Columbia, Department of Computer Science.} \and
Edwin Perkins
    \thanks{Email: \texttt{perkins@math.ubc.ca}. University of British Columbia, Department of Mathematics.} \and
Sikander Randhawa
    \thanks{Email: \texttt{srand@cs.ubc.ca}. University of British Columbia, Department of Computer Science.}
}
\date{}

\maketitle
\thispagestyle{empty}
\input{abstract.tex}

\newpage
\thispagestyle{empty}
\tableofcontents

\newpage \pagestyle{plain}\setcounter{page}{1}

\input{intro.tex}
\input{preliminaries.tex}
\input{upper_bound.tex}
\input{lower_bound.tex}
\input{cts_ub.tex}

\clearpage
\appendix
\input{app_hypergeometric.tex}
\input{app_cts_upper_bound.tex}

\input{app_oblivious.tex}

\bibliographystyle{plain}
\bibliography{references}

\end{document}

%% file: abstract.tex
\begin{abstract}
	We consider the classical problem of prediction with expert advice.
	In the fixed-time setting, where the time horizon is known in advance, algorithms that achieve the optimal regret are known when there are two, three, or four experts
	or when the number of experts is large.
	Much less is known about the problem in the anytime setting, where the time horizon is \emph{not} known in advance.
	No minimax optimal algorithm was previously known in the anytime setting, regardless of the number of experts.
	Even for the case of two experts, Luo and Schapire have left open the problem of determining the optimal algorithm.
	
	We design the first minimax optimal algorithm for minimizing regret in the anytime setting.
	We consider the case of two experts, and prove that the optimal regret is $\gamma \sqrt{t} / 2$ at all time steps $t$, where $\gamma$ is a natural constant that arose 35 years ago in studying fundamental properties of Brownian motion.
	The algorithm is designed by considering a continuous analogue of the regret problem, which is solved using ideas from stochastic calculus.
\end{abstract}

%% file: intro.tex
\section{Introduction}
\SectionName{intro}
We study the
problem of prediction with expert advice, whose origin can be traced back to the 1950s \cite{Hannan57}.
The problem is a sequential game between an adversary and an algorithm as follows.
There are $n$ actions, which are called ``experts''.
At each time step, the algorithm computes a distribution over the experts, then randomly chooses an expert according to that distribution; concurrently, the adversary chooses a cost in $[0,1]$ for each expert, with knowledge of the algorithm's distribution but not its random choice.
The cost of each expert is then revealed to the algorithm, and the algorithm incurs the cost that its chosen expert incurred.
The goal is to design an algorithm whose expected \emph{regret} is small.
That is, the goal is to minimize the difference between the algorithm's expected total cost and the total cost of the best expert.
This problem and its variants
have been a key component in numerous results;
we refer the reader to \cite{AHK12}.

The most well-known algorithm for the experts problem is the celebrated multiplicative weights update algorithm (MWU) \cite{LW94, Vovk90}.
In the fixed-time setting (where a time horizon $T$ is known in advance), MWU suffers a regret of $\sqrt{(T/2) \ln n}$ at time $T$, where $n$ is the number of experts \cite{CFHHSW97, Cesa99}.
This bound on the regret of MWU is known to be tight for any even $n$ \cite{GPS17}.
It is also known that $\sqrt{(T/2) \ln n}$ is asymptotically optimal for large $n$ and $T$.
(A precise statement may be found in the references
\cite[Corollary 3.2.2]{CFHHSW97}
\cite[Theorem 3.7]{CBL}.)
Interestingly, MWU is \emph{not} optimal for small values of $n$.
For $n = 2$, Cover~\cite{Cover66}
observed decades earlier that a natural dynamic programming formulation of the problem leads to a simple analysis showing that the minimax optimal regret is $\sqrt{T / 2\pi}$,
asymptotically for large $T$ (a proof of this can also be found in \cite[\S 3]{CFHHSW97}, \cite[Theorem 18.5.5]{KP17}).

For some applications, the time horizon $T$ is not known in advance;
examples include any sort of online tasks (e.g., online learning),
or tasks requiring convergence over time (e.g., convergence to equilibria).
An alternative model, more suited to those scenarios, is the \emph{anytime setting}\footnote{Other authors have referred to this setting as an ``unknown time horizon'' or ``bounds that hold uniformly over time''.}, in which algorithms are not given $T$ but must bound the regret \emph{for all} $T$.
Yet another model is to assume that $T$ is random with a known distribution \citep{LuoSchapire14}.
For example, the \emph{geometric horizon setting} of Gravin, Peres, and Sivan~\cite{gravin2016towards} assumes that $T$ is a geometric random variable.
In this setting, they gave the optimal algorithm for two and three experts.
Moreover, they propose a conjecture on the relationship between the fixed-time and the geometric horizon settings that could lead to optimal bounds for all $n$.

Our focus is the anytime setting.
One can convert algorithms for the fixed-time setting to the anytime setting by the well-known ``doubling trick'' \cite[\S 4.6]{CFHHSW97}.
This involves restarting the fixed-time horizon algorithm every power-of-two steps with new parameters.
If the fixed-time algorithm has regret $O(T^c)$ at time $T$ for some $c \in (0,1)$ then the doubling trick yields an algorithm with regret $O(t^c)$ at time $t$ for every $t \geq 1$.
On the one hand, this is a conceptually simple and generic reduction.
On the other hand, restarting the algorithm and discarding its state is clearly wasteful and probably not very practical.

Instead of using the doubling trick, one can use variants of MWU with a dynamic step size;
see, e.g., \cite[\S 2.3]{CBL}, \cite[Theorem 1]{Nesterov09}, \cite[\S 2.5]{BubeckNotes}.
This is a much more elegant and practical approach
and is even simpler to implement.
However, the analysis is more subtle than for MWU with a fixed step size.
It is known that, with an appropriate choice of step sizes, MWU can guarantee\footnote{It can be shown, by modifying arguments of \cite{GPS17},
that this is the optimal anytime analysis for MWU with step sizes $c/\sqrt{t}$.}
a regret of $\sqrt{t \ln n}$ for all $t \geq 1$ and all $n \geq 2$ (see \cite[Theorem~2.4]{BubeckNotes} or \cite[Proposition 2.1]{Ger11}).
However, it is unknown whether $\sqrt{t \ln n}$ is the minimax optimal anytime regret, for any value of $n$.
Indeed, Luo and Schapire \cite{LuoSchapire14} have also stated that ``finding the minimax solution to this setting seems to be quite challenging,
	even for the simplest case of $n = 2$''.

\paragraph{Results and techniques.}
This work considers the anytime setting with $n = 2$ experts.
We show that the optimal regret is $\frac{\gamma}{2} \sqrt{t}$, where $\gamma \approx 1.30693$ is a fundamental constant that arises in the study of Brownian motion~\cite{Perkins}.
(Note that $\gamma/2 \approx 0.653 < 0.833 \approx \sqrt{\ln 2}$.)
This also answers a question that has been left open by Luo and Schapire \cite{LuoSchapire14}.
It is not a priori obvious why this fundamental constant should play a role in both Brownian motion and regret.
Nevertheless, some connections are known.
For example, in the fixed-time setting, the optimal algorithms for $n \in \set{2,3,4}$  (see \cite{gravin2016towards}) and the optimal lower bound for $n \rightarrow \infty$ all involve properties of random walks.
Since Brownian motion is a continuous limit of random walks, a connection between anytime regret and Brownian motion is plausible.

Our techniques to analyze the optimal anytime regret are a significant departure from previous work on regret minimization.
First, we define a continuous-time analogue of the problem which
expresses the regret as a stochastic integral.
This allows us to utilize tools from stochastic calculus to arrive at a potential function whose derivative gives the optimal \emph{continuous}-time algorithm.
Remarkably, the optimal \emph{discrete}-time algorithm is the \emph{discrete} derivative of the same potential function.
We note that Freund~\cite{Freund09} has used stochastic differential equations for a continuous-time formulation of the experts problem, although he did not discuss the discrete-time problem.

The potential function that we derive involves a ``confluent hypergeometric function''.
Such functions often arise in solutions to differential equations, and are useful in discrete mathematics
\cite[\S 5.5]{GKP}.

\paragraph{Application.}
An interesting application of our results is to a problem in probability theory that does not involve regret at all.
Let $( X_t )_{t \geq 0}$ be a standard random walk.
Then $\expect{\abs{X_\tau}} \leq \gamma \expect{\sqrt{\tau}}$
for every stopping time $\tau$;
moreover, the constant $\gamma$ cannot be improved.\footnote{
At first glance, the inequality may seem to contradict the Law of the Iterated Logarithm.
However, we remark that if $\tau \coloneqq \inf\{ t > 0 \,:\, |X_t| \geq c\sqrt{t \ln \ln t} \}$ for some $c \in (0, \sqrt{2})$ then $\expect{\sqrt{\tau}} = \infty$ (despite $\tau$ being a.s.~finite) and the inequality is trivial.
}
This result is originally due to Davis~\cite[Eq.~(3.8)]{Davis76},
who proved it first for Brownian motion and later derived the result
for random walks (via the Skorokhod embedding).
We give a new derivation of Davis' result from our results in \Subsection{Application}.

\paragraph{Related work.}
The minimax regret for the experts problem has been well-studied in the fixed-time horizon setting.
As mentioned above, some tight asymptotics of the minimax regret were known decades ago:
for $n=2$, it is $\sqrt{T/2\pi}$ \cite{Cover66},
whereas asymptotically in $n$, it is $\sqrt{T \ln(n)/2}$ \cite{CFHHSW97, Cesa99}.
Recent work, building on results of Gravin et al.~\cite{gravin2016towards}, shows that the minimax regret is $\sqrt{8T/9\pi}$ for $n=3$~\cite{AbbasiBG17} and $\sqrt{\pi T/8}$ for $n=4$ \cite{BEZ20}.
The anytime setting is not as well understood.
In the two-experts setting, Luo and Schapire~\cite{LuoSchapire14} demonstrate that, if the time horizon $T$ is chosen by an adversary and unknown to the algorithm then the algorithm may be forced to incur regret at least $\sqrt{T/\pi}$.
This exceeds the minimax regret of $\sqrt{T/2\pi}$ in the fixed-time setting,
which establishes that the adversary has more power to cause regret in the anytime setting.

Recently, there has been a line of work that makes connections between the experts problem (in the finite-time horizon and geometric-time horizon setting) and PDEs \cite{Andoni,BEZ20,BEZ19,drenska,DK20,KKW19a,KKW19b}.
There is also work connecting regret minimization to option pricing \cite{DeMarzoKM06} and to the Black-Scholes formula \cite{AbernethyFW12}, which is based on Brownian motion and stochastic calculus.
Intuitively, stochastic calculus is a crucial tool to optimally hedge against future costs,
which we exploit too.

Our algorithm chooses the distribution on the experts using the discrete derivative of a potential function.
This idea has also been used in the AdaNormalHedge algorithm~\cite{AdaNormalHedge}, although their potential function was not derived in continuous time.

Our work crucially uses
stopping times for Brownian motion hitting a time-dependent boundary. Such techniques have also been used for non-adversarial bandits to approximate Gittins indices (see, e.g.,~\cite{BL02}).

%% file: preliminaries.tex
\section{Discussion of results and techniques}
\SectionName{prelim}

\subsection{Formal problem statement}
\SubsectionName{FormalProbStatement}

We will formulate the problem in the style of online convex optimization~\cite{SS}, in which at each time step a deterministic algorithm picks a distribution on experts.
An alternative formulation would be to have a randomized algorithm pick a single expert; see, e.g., \cite[Chapter 4]{CBL}.
Using the randomized formulation in the anytime setting has certain subtleties which we discuss in \Subsection{randomization}.

Let $n$ denote the number of experts.
There is a deterministic algorithm $\cA$,
and a deterministic adversary $\cB$ that knows $\cA$.
For each integer $t \geq 1$, there is a prediction task that is said to occur at time $t$.
In this task, $\cA$ picks a probability distribution $x_t \in [0,1]^n$,
and $\cB$ picks a cost vector $\ell_t \in [0,1]^n$.
The coordinate $\ell_{t,j}$ denotes the cost of the $j\th$ expert at time $t$.

After $x_t$ is chosen the vector $\ell_t$ is revealed, so $x_t$ 
depends on $\ell_1,\ldots,\ell_{t-1}$
(and implicitly $x_1,\ldots,x_{t-1}$).
The vector $\ell_t$ depends on $\cA$ and on $\ell_1,\ldots,\ell_{t-1}$
(and implicitly $x_1,\ldots,x_t$, since $\cA$ is deterministic and known to $\cB$).
The game can end whenever $\cB$ wishes, or continue forever.
Since $\cA$ is deterministic and known to $\cB$, the entire sequence of interactions, including the ending time, can be predetermined by $\cB$.

The cost incurred by the algorithm at time $t$ is the inner product $\inner{x_t}{\ell_t}$.
This may be thought of as the ``expected cost'' of the algorithm, although the algorithm is actually deterministic.
The total expected cost of the algorithm up to time $t$ is
$\sum_{i=1}^t \inner{x_i}{\ell_i}$.
For $j \in [n]$, the total cost of the $j\th$ expert up to time $t$
is $L_{t,j} = \sum_{i=1}^t \ell_{i,j}$.
The regret at time $t$ of algorithm $\cA$ against adversary $\cB$
is the difference between the algorithm's total expected cost and the total cost of the best expert,
i.e.,
$$
\Regret{n,t,\cA,\cB} ~=~ \sum_{i=1}^t \inner{x_i}{\ell_i}
 \:-\: \min_{j \in [n]} L_{t,j}.
$$

\paragraph{Anytime setting.}
This work focuses on the anytime setting.
In this setting, one may view the algorithm as running
forever, with the goal of minimizing, for \emph{all} $t$, the regret normalized by $\sqrt{t}$.
Alternatively, one may view the game as ending at a time chosen by the adversary,
and the algorithm must minimize the regret at that ending time.
(It does not matter whether the adversary chooses the ending time in advance or dynamically, since $\cA$ and $\cB$ are deterministic so all interactions are predetermined.)
These two views are equivalent because the algorithm cannot distinguish between them.

Formally, the goal is to design an algorithm which achieves the infimum in the following expression defining the minimax anytime regret.
\begin{equation}\EquationName{AnytimeDef}
\AnytimeValue(n) ~\coloneqq~
\inf_{\cA} \sup_{\cB} \sup_{t \geq 1} \frac{\Regret{n,t,\cA,\cB}}{\sqrt{t}}.
\end{equation}
This precise value was previously unknown even for $n=2$.
The best known bounds at present are
\begin{equation}
\label{eq:KnownTwoExpertBounds}
0.564 ~\approx~ \sqrt{1/\pi} ~\leq~ \AnytimeValue(2) ~\leq~ \sqrt{\ln 2} ~\approx~ 0.833.
\end{equation}
The lower bound, due to \cite{LuoSchapire14}, demonstrates a gap between the anytime setting and the fixed-time setting, where the optimal normalized regret is $\sqrt{1/2\pi}$ \cite{Cover66}.
Our main result is that $\AnytimeValue(2)=\gamma/2 \approx 0.653$ and consequently 
neither inequality in \Equation{KnownTwoExpertBounds} is tight.

As mentioned above, MWU with a dynamic step size shows that
$\AnytimeValue(n) \leq \sqrt{\ln n}$ for all $n \geq 2$ \cite[\S 2.5]{BubeckNotes}.
The lower bound 
$\liminf_{n \rightarrow \infty} \AnytimeValue(n)/\sqrt{\ln n} \geq \sqrt{1/2}$
follows from the bound in the fixed-time setting \cite{CFHHSW97}.
Thus, the upper bound is loose by at most a factor $\sqrt{2}$.

\subsubsection{Randomized formulations}
\SubsectionName{randomization}

Several alternative formulations of the problem arise if $\cA$ selects a single expert $I_t \in [n]$ randomly at each time $t$, and the adversary chooses an ending time $\tau$.
We mention three possibilities, differing in the power of the adversary $\cB$.
\begin{itemize}
    \item The most powerful adversary allows $\ell_t$ to depend on $I_1,\ldots,I_t$.
    In this case it is easy to design $\cB$ with $\Regret{n,t,\cA,\cB} = \Omega(t)$.

    \item An adversary of intermediate power allows the cost vector $\ell_t$ and the event $\tau=t$ to be determined by $I_1,\ldots,I_{t-1}$. This is analogous to the ``non-oblivious opponent'' of \cite[\S 4.1]{CBL}.
    Interestingly, one can design such an adversary $\cB$ for which $\expect{ \frac{\Regret{n,\tau,\cA,\cB}}{\sqrt{\tau \log \log \tau}} } = \Omega(1)$.
    The surprising aspect is the $\sqrt{\log \log \tau}$ in the denominator, which arises due to the law of the iterated logarithm.
    We prove this result in \Appendix{oblivious}.

    \item The weakest adversary requires that $\ell_t$ and $\tau$ depend only on $\cA$ and not its random choices $I_1,I_2,\ldots$.
    This is analogous to the ``oblivious opponent'' of \cite[\S 4.1]{CBL}.
    The expected regret in this model is identical to the regret in the deterministic model described at the start of \Subsection{FormalProbStatement}.
\end{itemize}

We favour this third model because it is consistent with the online convex optimization literature, and moreover its minimax regret has the ideal asymptotics $\Theta(\sqrt{t})$.
It is intriguing that in the anytime setting, the non-oblivious opponent has more power than the oblivious opponent.
In contrast, the two adversaries have the same power in the fixed time setting \cite[\S 4.1]{CBL}.

\subsection{Statement of results}
\SubsectionName{StatementOfResults}

To state our results, we require two definitions.
\begin{equation}\EquationName{Erfi}
\begin{split}
\erfi(x) &~=~ \frac{2}{\sqrt{\pi}} \int_0^x e^{z^2} \,\mathrm{d}z \\
M_0(x)   &~=~ e^x - \sqrt{\pi x} \erfi(\sqrt{x})
\end{split}
\end{equation}
The first is the imaginary error function, a well-known special function that relates to the Gaussian error function.
The second is an example of a confluent hypergeometric function,
a very broad class of special functions that includes, e.g., Bessel functions and Laguerre polynomials.
(See \Subsection{hyper} for formal definitions.)
Our analysis makes use of a few elementary properties of these functions.
A key constant used in this paper is $\gamma$, which is defined to be the smallest\footnote{
In fact, $\gamma$ is the \emph{unique} positive root.
See \Fact{M0_unique_root}.
}
positive root\footnote{
The \emph{roots} of certain confluent hypergeometric functions have appeared in studying some natural phenomena of Brownian motion; for some examples see \cite{Breiman,Davis76,GreenwoodPerkins,Perkins}.} of $M_0(x^2/2)$, i.e.,
\begin{equation}
\EquationName{GammaDef}
\gamma ~:=~ \min \setst{ x>0 }{ M_0(x^2/2)=0 } ~\approx~ 1.3069...
\end{equation}

\begin{theorem}[Main result]
\TheoremName{main}
In the anytime setting with two experts,
the minimax optimal
normalized regret
(over deterministic algorithms $\cA$ and adversaries $\cB$)
is  
\begin{equation}
\EquationName{mainthm}
\AnytimeValue(2) ~=~
\inf_{\cA} \sup_{\cB} \sup_{t \geq 1} \frac{\Regret{2,t,\cA,\cB}}{\sqrt{t}}
~=~ \frac{\gamma}{2}.
\end{equation}
\end{theorem}

The proof of this theorem has two parts:
an upper bound, in \Section{ub}, which exhibits an optimal algorithm, and
a lower bound, in \Section{lb}, which exhibits an optimal randomized adversary.
The algorithm is very short, and it appears below in \Algorithm{minimax}.

One might imagine that some form of duality theory is involved in our matching upper and lower bounds.
Indeed, if the costs are in $\{0,1\}$ one may write $\AnytimeValue(2)$ as the value of an infinite-dimensional linear program, although we do not explicitly adopt this viewpoint.
Instead, $\gamma$ arises in our lower bound as the maximizer in \Equation{SupStoppingTime},
whereas $\gamma$ arises in our upper bound as the minimizer in \Equation{MinimizeBoundaryProblem}.
We are not aware of any direct relationship between those two equations.
Nevertheless, our algorithm and our lower bound can be seen as constructing feasible primal and dual solutions, respectively, to the aforementioned linear program.

\paragraph{Comparison to existing techniques.} 
A duality viewpoint is adopted by Gravin et al.~\cite{gravin2016towards} in the fixed-time and geometric horizon settings using von Neumann's minimax theorem (see also \cite[\S 18.5]{KP17}). Their dual problem is characterized by properties of random walks, which allows one to determine the optimal dual value directly without reference to the primal.
It is conceivable that some form of von Neumann's minimax theorem can be applied for the anytime setting, although it is unclear due to the appearance of the supremum and $1/\sqrt{t}$ in \eqref{eq:mainthm}.
Our results of \Section{lb} may be viewed as using random walks to construct feasible dual solutions of value $\gamma/2-\eps ~\:\forall\eps > 0$, but it is not obvious that these solutions converge to the optimal dual value.
The only way we know of to prove optimality of those dual solutions is to construct an algorithm whose regret is $\gamma \sqrt{t}/2$. This is the more challenging part of this paper, which we discuss in Sections~\ref{sec:ub} and \ref{sec:cts_ub}.

\paragraph{A conjecture for $n$ experts.}
We suspect that the roots of confluent hypergeometric functions may also have a key role to play in designing optimal algorithms when there are $n > 2$ experts.
For $r \in (0, 1]$, we define $\alpha(r)$ as the smallest positive root of the function $x \mapsto M(-r, 1/2, x^2/2)$ defined in \Subsection{hyper}.
With this definition, we have that $\gamma = \alpha(1/2)$ where $\gamma$ is as defined in \Equation{GammaDef}.
A plausible conjecture is that $\lim_{n \to \infty} \frac{\AnytimeValue(n)}{\alpha(1/n)} = 1$.
It can be shown that $\lim_{n \to \infty} \frac{\alpha(1/n)}{\sqrt{\ln n}} = 1/\sqrt{2}$.
Hence, our conjecture roughly states that, for large $n$, the optimal regret for $n$ experts
in the anytime setting is $\sqrt{t\ln(n) / 2}$.
This last bound matches the guarantee in the fixed time setting.

\begin{remark}
    Our lower bound can be strengthened to show that, for any algorithm $\cA$,
    \[
        \sup_{\cB} \:
        \limsup_{t \geq 1}
        \frac{\Regret{2,t,\cA,\cB}}{\sqrt{t}}
        ~\geq~ \frac{\gamma}{2}.
    \]
    The key aspect is here the $\limsup$ rather than a $\sup$.
    In particular, even if $\cA$ is granted a ``warm-up'' period during which its regret is ignored, an adversary can still force it to incur large regret afterwards.
    This strengthened result is proved in \Subsection{regretIO}.
\end{remark}

The algorithm's description and analysis relies heavily on a function $R \colon \bR_{\geq 0} \times \bR \rightarrow \bR$ defined by
\begin{equation}
\EquationName{RDef}
    R(t,g) =
    \begin{cases}
        0 & \text{($t = 0$)} \\
        \frac{g}{2} + \kappa \sqrt{t} \cdot M_0\left(\sfrac{g^2}{2t} \right) & \text{($t > 0$ and $g \leq \gamma \sqrt{t}$)} \\
        \frac{\gamma \sqrt{t}}{2} & \text{($t > 0$ and $g \geq \gamma\sqrt{t}$)}
    \end{cases}
    \qquad\text{where}\qquad
    \kappa = \frac{1}{\sqrt{2\pi} \erfi(\gamma / \sqrt{2})}
\end{equation}
and $M_0$ as defined in \eqref{eq:Erfi}.
The function $R$ may seem mysterious at first, but in fact arises naturally from the solution to a stochastic calculus problem\footnote{
As we will see below, the regret against a random adversary is a stochastic integral.
Viewing this problem in continuous time, then designing a function to minimize the integral leads to a PDE which $R$ solves.
}
in \Section{cts_ub}.
In our usage of this function, $t$ will correspond to the time and $g$ will correspond to the \emph{gap} between (i.e., absolute difference of) the total loss for the two experts.
One may verify that $R$ is continuous on $\bR_{>0} \times \bR$ because the second and third cases agree on the curve
$ \setst{ (t,\gamma \sqrt{t}) }{ t>0 } $ since $\gamma$ satisfies $M_0(\gamma^2/2) = 0$.
We next define a function $p$ to be
\begin{equation}
\EquationName{DiscreteDerivative}
    p(t,g) = \smallfrac{1}{2} \big( R(t, g+1) - R(t, g-1) \big).
\end{equation}
This is the discrete derivative of $R$ at time $t$ and gap $g$.
The algorithm constructs its distribution $x_t$ so that $p(t,g)$ is the probability mass assigned to the expert with the greatest accumulated loss so far at time $t$.
It is shown later that $p(t,g) \in [0,1/2]$ whenever $t \geq 1$ and $g \geq 0$ so that $p$ is indeed a probability and the algorithm is well defined.
\label{IsPaProbability}
We remark that $p(t,0) = 1/2$ (\Lemma{ptg_well_defined}) for all $t \geq 1$ so the algorithm places equal mass on both experts when their cumulative losses are equal.

\begin{algorithm}
\caption{The algorithm achieving the minimax anytime regret for two experts.
At each time step, each expert incurs a cost in the interval $[0,1]$, so the cost vector $\ell_t$ lies in $[0,1]^2$.}
\AlgorithmName{minimax}
\begin{algorithmic}[1]
\State Initialize $L_0 \gets \begin{smallbmatrix} 0 \\0 \end{smallbmatrix}$.
\For{$t = 1, 2, \ldots$}
\State If necessary, swap indices so that $L_{t-1,1} \geq L_{t-1,2}$.
\State The current gap is $g_{t-1} \leftarrow L_{t-1,1} - L_{t-1,2}$.
\State Set $x_t \leftarrow \big[\, p(t,g_{t-1}) ,\, 1 \!-\! p(t, g_{t-1}) \,\big]$,
    where $p$ is the function defined by
    \eqref{eq:DiscreteDerivative}.
\LineComment{Observe cost vector $\ell_t$ and incur cost $\inner{x_t}{\ell_t}$.}
\State $L_{t} \gets L_{t-1} + \ell_t$
\EndFor
\end{algorithmic}
\end{algorithm}

\subsection{Techniques}
\SubsectionName{techniques}

\paragraph{Lower Bound.}
A common approach to prove lower bounds in the experts problem is to consider a random adversary.
With 2 experts, this adversary changes the gap by $\pm 1$ at each step.
In the fixed-time setting, the adversary has no control over the time horizon; it is known to both the adversary and the algorithm beforehand. 
The adversary in the anytime setting has the additional power to choose the time horizon, without informing the algorithm, and therefore it is perhaps unsurprising that an adversary using a fixed time horizon does not provide the optimal anytime lower bound.

To obtain the optimal lower bound, we allow the adversary to select a \emph{random time} $\tau$ at which the game ends. 
In general, a random adversary in the anytime setting could generate an infinitely long sequence of random bits as its costs, then select the ending time $\tau$ as a function of the entire sequence. 
We will consider a weaker random adversary in which $\tau$ is not a function of the entire sequence, but instead the event $\tau=t$ is known at time $t$;
that is, $\tau$ is a stopping time \cite[\S 9.1]{Klenke}.
There is an adversary of this weaker form that is nonetheless optimal, as we discuss next.

First, let us view the regret as a discrete stochastic process.
To analyze this stochastic process, we use an elementary identity known as Tanaka's Formula for random walks, which allows us to write the regret process as
$\Reg(t) = Z_{t} + g_{t}/2$ where $Z_t$ is a martingale with $Z_0 = 0$ and $g_t$ is the current gap at time $t$.
When $\tau$ is a stopping time satisfying certain hypotheses, the \emph{Optional Stopping Theorem} (OST) yields $\expect{Z_{\tau}} = Z_0 = 0$.
We will restrict our attention to adversaries whose stopping times satisfy the hypotheses of the OST.
(The stopping times in the fixed-time and geometric horizon settings trivially satisfy the hypotheses.)
It is not a priori obvious that there is an optimal adversary in the anytime setting satisfying this restriction.

Concretely, we will consider adversaries that select $\tau$ to be the first time that the gap $g_t$ exceeds\footnote{Note that $\tau = \min \setst{t \geq 0}{g_t \geq f(t)}$ is a stopping time.} some time dependent boundary $f(t)$.
This approach follows an established doctrine that connects optimal stopping and stochastic control problems to free-boundary problems \cite{Chernoff,Peskir}.
The conclusion of the OST is then that $\expect{\Reg(\tau)} = \expect{g_{\tau}}/2 \geq \expect{f(\tau)}/2$.
However, the hypotheses of the OST must be respected, otherwise the adversary could just select the boundary $f(t)$ to be arbitrarily large, and the resulting regret lower bound would violate known upper bounds. 
  
To understand what boundaries $f(t)$ to consider, let us discuss the OST hypotheses.
First, it is not sufficient for the stopping time to be \emph{almost surely finite}.
(Otherwise, one could use a boundary $f(t)=\Theta(\sqrt{t \ln \ln t})$ and the Law of the Iterated Logarithm~\cite{Durrett} to prove lower bounds that contradict the $O(\sqrt{t})$ upper bound of Cover or MWU.)
At this point a lucky guess is required:
we will consider boundaries of the form $f(t) = c \sqrt{t}$, since this is consistent with the known $\Theta(\sqrt{t})$ regret bounds.
It is known \cite{Shepp,Breiman} that choosing $c < 1$ is necessary and sufficient to ensure that $\expect{\tau} < \infty$, which is a sufficient hypothesis for the OST. Unfortunately this only yields a 
regret lower bound of $\sqrt{t}/2$, which is trivial. (The algorithm can easily be forced to have regret $1/2$ at time $t=1$.)
Therefore, the condition $\expect{\tau} < \infty$ is too restrictive.

Fortunately there is a less-widely known hypothesis for the OST that leads to optimal results in our setting.
Concretely,
if $Z_t$ is a martingale with bounded increments (i.e.~$\sup_{t \geq 0} |Z_{t+1} - Z_t| \leq K$ for some $K > 0$) and $\tau$ is a stopping time satisfying $\expect{\sqrt{\tau}} < \infty$, then $\expect{Z_{\tau}} = 0$.
The crucial detail is to bound the \emph{expected square root} of $\tau$.
This result is stated formally in \Theorem{OST}.
It remains to choose as large a boundary as possible such that the associated stopping time of hitting the boundary satisfies $\expect{\sqrt{\tau}} < \infty$.
Using classical results of Breiman \cite{Breiman} and Greenwood and Perkins~\cite{GreenwoodPerkins}, we show that the optimal choice of $c$ is $\gamma$.

\paragraph{Upper Bound.}
Our analysis of \Algorithm{minimax}, to prove the upper bound in \Theorem{main}, uses a deceptively simple argument where $R$ defined in \Equation{RDef} acts as a potential function.
Specifically, we show that the change in regret from time $t-1$ with gap $g_{t-1}$
to time $t$ with gap $g_t$ is at most $R(t, g_t) - R(t-1,g_{t-1})$.
By telescoping, this immediately implies that $\max_g R(t,g)$ is an upper bound on the regret
at time $t$.
The analysis has a number of key features.
First, note that the potential function $R$ is bivariate; it depends on both the \emph{time} $t$ as well as the \emph{state} $g_t$.
To deal with this bivariate potential, we use a tool known as the discrete It\^{o} formula.
This formula allows us to relate the regret to the potential $R$, while elegantly handling changes to both time and state.
In fact, the potential $R$ turns out to be an extremely tight approximation to the actual regret.
Previously, there have been several works that make use of bivariate potentials (e.g.~\cite{NormalHedge, AdaNormalHedge}).
However, to the best of our knowledge, our work is the first to use the discrete It\^{o} formula in the setting of regret minimization.

The function $R$ and the use of discrete It\^{o} do not come ``out of thin air'';
they come from considering a continuous-time analogue of the problem.
This continuous viewpoint brings a wealth of analytical tools that do not exist (or are more cumbersome) in the discrete setting.
As discussed in the lower bound section above, in discrete-time it is natural to
assume the gap process evolves as a reflected random walk.
In order to formulate the continuous-time problem, we assume that the continuous adversary evolves the gap between the best and worst expert as a reflected Brownian motion (the continuous-time analogue of a random walk).
Using this adversary, the continuous-time regret becomes a stochastic integral.

The most natural way to analyze an integral is to use the fundamental theorem of calculus (FTC). However, the continuous-time regret is defined by a stochastic integral so the FTC cannot be applied\footnote{
The integrator is reflected Brownian motion, which is not of bounded variation.}.
However there is a stochastic analog of the FTC, namely the (continuous) It\^{o} formula, which we state in \Theorem{ItoFormula}.
We use it to provide an insightful decomposition of the continuous-time regret.
In particular, this decomposition suggests that the algorithm should satisfy an analytic condition
known as the \emph{backwards heat equation}.
A key resulting idea is: if the algorithm satisfies the backward heat equation, then there is a natural potential function that upper bounds the regret of the algorithm.
This enables a systematic approach to obtain an explicit continuous-time algorithm and a potential function that bounds the continuous algorithm's regret.
To go back to the discrete setting, using the \emph{same} potential function,
we replace applications of It\^{o}'s formula with the discrete It\^{o} formula.
Remarkably, this leads to \emph{exactly} the same regret bound as the continuous setting.

\subsection{Application}
\SubsectionName{Application}

As mentioned in \Section{intro}, the following theorem of Davis can be proven as a corollary of our techniques.
Intriguingly, the proof involves regret, despite the fact that regret does not appear in the theorem statement.

\begin{theorem}[{\protect Davis~\cite{Davis76}}]
\TheoremName{DavisRW}
Let $( X_t )_{t \geq 0}$ be a standard random walk.
Then $\expect{\abs{X_\tau}} \leq \gamma \expect{\sqrt{\tau}}$
for every stopping time $\tau$;
moreover, the constant $\gamma$ cannot be improved.
\end{theorem}

\begin{proof}
We begin by proving the first assertion.
Suppose that $\Reg(T)$ is the regret process when \Algorithm{minimax} is used
against a random adversary.
As discussed in \Subsection{techniques} and \Equation{LBRegret}, we can write the regret process
as $\Reg(T) = Z_T + g_T / 2$ where $Z_T$ is a martingale and $g_T$ evolves
as a reflected random walk.\footnote{
    Equality holds because our algorithm satisfies $p(t,0) = 1/2$; this is discussed in the text preceding \Equation{LBRegret}.
}
Moreover, if $\tau$ is a stopping time satisfying $\expect{\sqrt{\tau}} < \infty$,
then $\expect{Z_{\tau}} = 0$ (see \Theorem{OST}).

The upper bound in \Theorem{main} asserts that $\gamma \sqrt{T} / 2 \geq \Reg(T) = Z_T + g_T/2$ simultaneously for all $T \geq 0$.
Hence, $\gamma \expect{\sqrt{\tau}} / 2 \geq \expect{g_{\tau}} / 2$.
Replacing $g_{\tau}$ with $|X_{\tau}|$ (since both $g_t$ and $\abs{X_t}$ are reflected random walks), the proof of the first assertion is complete.

The fact that no constant smaller than $\gamma$ is possible is a direct
consequence of the results of Breiman~\cite{Breiman} and Greenwood and Perkins \cite{GreenwoodPerkins} as mentioned in \Subsection{techniques}
(see also \Section{lb} or \cite{Davis76}).
\end{proof}
\begin{remark}
    Davis \cite{Davis76} proved \Theorem{DavisRW} for both
    random walks and Brownian motion.
    We are also able to recover the result for Brownian motion
    as a corollary of our continuous-time result (\Theorem{ContsMainResult}).
    The proof is very similar to that above.
\end{remark}

\begin{remark}
    In retrospect, our arguments in \Section{cts_ub} have some similarities with Davis' proof of the Brownian Motion version of \Theorem{DavisRW} \cite[Theorem 1.1]{Davis76}.
    For example, in \Section{cts_ub}, we see that the backwards heat equation appears naturally in the design of the continuous algorithm.
    Analogously, the backwards heat equation also appears in the proof of Theorem 1.1 in \cite{Davis76} but as a result of searching for a supermartingale.
\end{remark}

\subsection{An expression for the regret involving the gap}

In our two-expert prediction problem, 
the most important scenario restricts each cost vector $\ell_t$ to be either
$\begin{smallbmatrix}1\\0\end{smallbmatrix}$ or $\begin{smallbmatrix}0\\1\end{smallbmatrix}$.
That is, at each time step, some expert incurs cost $1$ and the other expert incurs no cost.
This restricted scenario is equivalent to the condition
$g_t - g_{t-1} \in \set{\pm 1} ~\forall t \geq 1$,
where $g_t \coloneqq \abs{ L_{t,1} - L_{t,2} }$ is the gap at time $t$.
To prove the optimal lower bound it suffices to consider this restricted scenario.
The optimal upper bound is first proven in the restricted scenario,
then extended to general cost vectors in \Subsection{ub_general}.
With the sole exception of \Subsection{ub_general}, we assume the restricted scenario.

We now present an expression, valid for any algorithm, that emphasizes how the regret depends on the
\emph{change} in the gap.
This expression will be useful in proving both the upper and lower bounds.
Henceforth we write $\Reg(t) \coloneqq \Reg(2, t, \cA, \cB)$ where $\cA$ and $\cB$
are usually implicit from the context.

\begin{prop}
\PropositionName{RegretWithGaps}
Assume the restricted setting in which $g_t - g_{t-1} \in \set{\pm 1}$ for every $t \geq 1$.
When $g_{t-1} \neq 0$, let $p_t$ denote the probability mass assigned by the algorithm to the ``worst expert'',
i.e.,~if $L_{t-1,1} \geq L_{t-1,2}$ then $p_t = x_{t,1}$ and otherwise $p_t = x_{t,2}$.
The quantity $p_t$ may depend arbitrarily on $\ell_1,\ldots,\ell_{t-1}$.
Then
\begin{equation}
    \EquationName{RegretInGapSpace}
    \Regret{T} ~=~ \sum_{t=1}^T p_t \cdot ( g_t - g_{t-1} ) \cdot \ind[g_{t-1} \neq 0] \:+\: \sum_{t=1}^T\inner{x_t}{\ell_t}\cdot \ind [g_{t-1} = 0].  
\end{equation}
Furthermore, assume that if $g_{t-1} = 0$, then $p_t = x_{t,1} = x_{t,2} = 1/2$.
In this case
\begin{equation}
    \EquationName{RegretInGapSpace2}
    \Regret{T} ~=~ \sum_{t=1}^T p_t \cdot ( g_t - g_{t-1} ).
\end{equation}
\end{prop}

\begin{remark}
Observe that \eqref{eq:RegretInGapSpace2} is a discrete analog of a Riemann-Stieltjes integral of $p$ with respect to $g$.
If $(g_t)_{t \geq 0}$ is a random process, then \eqref{eq:RegretInGapSpace2} is called a discrete stochastic integral.
In the specific case that $(g_t)_{t \geq 0}$ is a reflected random walk (the absolute value of a standard random walk),
then \Equation{RegretInGapSpace} is the Doob decomposition \cite[Theorem 10.1]{Klenke}
of the regret process $\big( \Regret{t} \big)_{t \geq 0}$,
i.e., the first sum is a martingale and the second sum is an increasing predictable process.
\end{remark}

\newcommand{\DeltaR}{\Delta_{\mathrm{R}}}
\begin{proof}
Define $\DeltaR(t) = \Regret{t}-\Regret{t-1}$.
The total cost of the best expert at time $t$ is $L_t^* \coloneqq \min \set{L_{t,1}, L_{t,2}}$.
The change in regret at time $t$ is the cost incurred by the algorithm
minus the change in the total cost of the best expert,
so $\DeltaR(t) = \inner{x_t}{\ell_t} - (L_t^* - L_{t-1}^*)$.

\paragraph{Case 1: $g_{t-1} \neq 0$.}
In this case, the best expert at time $t-1$ remains a best expert at time $t$. Note that this uses the assumption that $g_t - g_{t-1} \in \{\pm 1\}$ so $g_{t-1} \geq 1$.
If the worst expert incurs cost $1$, then the algorithm incurs cost $p_t$ and the best expert incurs cost $0$, so $\DeltaR(t) = p_t$ and $g_t-g_{t-1}=1$.
Otherwise, the best expert incurs cost $1$ and the algorithm incurs cost $1-p_t$,
so $\DeltaR(t) = -p_t$ and $g_t-g_{t-1}=-1$.
For either choice of cost, we have $\DeltaR(t) = p_t \cdot (g_t-g_{t-1})$.

\paragraph{Case 2: $g_{t-1} = 0.$}
Both experts are best, but one incurs no cost, so $L_t^* = L_{t-1}^*$ and $\DeltaR(t) = \inner{x_t}{\ell_t}$.

\vspace{0.6em}
The above two cases prove \Equation{RegretInGapSpace}.
For the last assertion, we have that $\inner{x_t}{\ell_t} = 1/2 = p_t\cdot (g_t - g_{t-1}) $ whenever $g_{t-1} = 0$.
Hence, we can collapse the two sums in \Equation{RegretInGapSpace} into one to get \Equation{RegretInGapSpace2}.
\end{proof}

\subsection{Basic facts about confluent hypergeometric functions}
\SubsectionName{hyper}
For any $a,b \in \bR$ with $b \not\in \bZ_{\leq 0}$, the confluent hypergeometric function of the first kind is defined as
\begin{equation}
	\EquationName{ConfluentDef}
	M(a, b, z) = \sum_{n=0}^{\infty} \frac{(a)_n z^n}{(b)_n n!},
\end{equation}
where $(x)_n \coloneqq \prod_{i=0}^{n-1}(x+i)$ is the Pochhammer symbol.
See, e.g., Abramowitz and Stegun~\cite[Eq.~(13.1.2)]{Abramowitz}.

For notational convenience, for $i \in \{0, 1, 2, \ldots,\}$, we write
\begin{equation}
	\EquationName{MiDef}
	M_i(x) = M(i-1/2, i+1/2, x).
\end{equation}

\begin{fact}
	\FactName{chf_derivs}
	If $b \notin \bZ_{\leq 0}$ then $\diffwrt{x} M(a,b,x) = \frac{a}{b} \cdot M(a+1, b+1, x)$.
	Consequently,
	\begin{enumerate}[label=(\arabic*),noitemsep,topsep=0pt]
		\item $M_0'(x) = -M_1(x)$; and
		\item $M_1'(x) = \frac{1}{3} \cdot M_2(x)$.
	\end{enumerate}
\end{fact}
\begin{proof}
	See \cite[Eq.~(13.4.9)]{Abramowitz}.
\end{proof}

\begin{fact}
	\FactName{basic_identities}
	The following identities hold:
	\begin{enumerate}[label=(\arabic*),noitemsep,topsep=0pt]
		\item $M_0(x) = -\sqrt{\pi x} \erfi(\sqrt{x}) + e^x$.
		\label{item:M0}
		\item $M_1(x) = \frac{\sqrt{\pi} \erfi(\sqrt{x})}{2\sqrt{x}}$.
		\label{item:M1}
		\item $M_2(x) = \frac{3(2e^x \sqrt{x} - \sqrt{\pi} \erfi(\sqrt{x}))}{4x^{3/2}}$.
		\label{item:M2}
		\item $\frac{2}{3} \cdot M_2(x) \cdot x + M_1(x) = e^x$.
		\label{item:M1M2}
	\end{enumerate}
\end{fact}
\begin{proof}\mbox{}
	
	(2): See \cite{Abramowitz}, equations (7.1.21) or (13.6.19), and use that $\erfi(x)=-i \erf(ix)$, where $i=\sqrt{-1}$.
	
	(1): Differentiating the right-hand side (using the definition of $\erfi$ in \eqref{eq:Erfi})
	yields $-\frac{\sqrt{\pi} \erfi(\sqrt{x})}{2\sqrt{x}}$.
	So the right-hand side is an anti-derivative of $-M_1(x)$, by part (2).
	Thus, the identity (1) follows from \Fact{chf_derivs}(1) and the initial condition $M_0(0)=1$.
	
	(3): This follows directly by differentiating (2) and \Fact{chf_derivs}(2).

	(4): Immediate from (2) and (3).
\end{proof}

\begin{fact}
	\FactName{M0_decreasing_concave}
	The function $M_0(x)$ is decreasing and concave on $[0, \infty)$.
\end{fact}
\begin{remark}
	In fact, $M_0(x)$ is decreasing and concave on $\bR$ but we will not require this fact.
\end{remark}
\begin{remark}
	The function $M_0(x^2/2)$ is also decreasing and concave on $[0, \infty)$.
	Indeed, the concavity follows from the fact that if $f$ is a non-increasing concave function on $\bR$ and $g$ is a convex function on $\bR$ then $f(g(x))$ is concave.
	Although this fact is crucial for our algorithm, we do not explicitly make reference to this in the paper.
\end{remark}
\begin{proof}
	By \Fact{chf_derivs}, we have $M_0'(x) = -M_1(x)$ and $M_0''(x) = -\frac{1}{3} \cdot M_2(x)$.
	Note that the coefficients of $M_1(x), M_2(x)$ in their Taylor series are all non-negative.
	As $x \geq 0$, we have that $M_0'(x), M_0''(x) \leq 0$ as desired.
\end{proof}
\begin{fact}
	\FactName{M0_unique_root}
	The function $x \mapsto M_0(x^2/2)$ has a unique positive root at $x = \gamma$.
	Moreover $M_0(x^2/2) > 0$ for $x \in (0, \gamma)$ and $M_0(x^2/2) < 0$ for $x \in (\gamma, \infty)$.
\end{fact}
\begin{proof}
	The Maclaurin expansion of $M_0(x^2/2)$ is given by
	\[
	M_0\left( \frac{x^2}{2} \right)
	= 1 - \sum_{k=1}^{\infty} \frac{1}{(2k-1)k!} \frac{x^{2k}}{2^k}.
	\]
	Note that $M_0(0) = 1$.
	It is clear, from the series expansion above (and \Fact{M0_decreasing_concave}), that $M_0(x^2/2)$ is strictly decreasing in $x$ on $(0, \infty)$ and $\lim_{x \to \infty} M_0(x^2/2) = -\infty$.
	Hence, $M_0(x^2/2)$ contains a positive root $\gamma$ and it is unique.
	Finally, it is clear that $M_0(x^2/2)$ is positive on $(0, \gamma)$ and negative on $(\gamma, \infty)$.
\end{proof}

\begin{claim}
	\ClaimName{RootOfM}
	For any $\eps>0$, there exists $a_\eps \in (-1,-1/2)$ such that
	the smallest\footnote{In fact, there is a unique positive root.} positive root $c_\eps$ of $z \mapsto M(a_{\eps},1/2,z^2/2)$
	satisfies $c_\eps \geq \gamma-\eps$.
\end{claim}
\begin{proof}
	Following Perkins' notation \cite{Perkins}, let $\lambda_0(-c,c)$ be such that
	$c$ is the smallest positive root of $x \mapsto M(-\lambda_0(-c,c),1/2,x^2/2)$.
	By \cite[Proposition 1]{Perkins}, the map $c \mapsto \lambda_0(-c,c)$ is strictly decreasing
	and continuous on $\bR_{>0}$, so it has a continuous inverse $\alpha$.
	From \eqref{eq:GammaDef} and \Fact{basic_identities}(1), we see that
	$\lambda_0(-\gamma, \gamma) = 1/2$, hence $\alpha(1/2)=\gamma$.
	By continuity, for all $\eps > 0$, there exists $\delta \in (0,1/2)$ such that
	$\alpha(1/2+\delta) > \gamma-\eps$.
	Then we may take $a_\eps = -(1/2+\delta)$ and $c_\eps = \alpha(1/2+\delta)$.
\end{proof}

%% file: upper_bound.tex
\section{Upper bound}
\SectionName{ub}

In this section, we prove the upper bound in \Theorem{main}
via a sequence of simple steps.
The main ideas of the proof are contained in the restricted setting where the gap changes by $\pm 1$ each step.
This corresponds to each loss vector $\ell_t$ being either $\begin{smallbmatrix}1\\0\end{smallbmatrix}$ or $\begin{smallbmatrix}0\\1\end{smallbmatrix}$.
We first prove the upper bound in \Theorem{main} in this restricted setting.
In \Subsection{ub_general}, we extend the analysis to general loss vectors in $[0,1]^2$ through the use of concavity arguments.

The proof in this section uses the potential function $R$ which,
as explained in \Subsection{techniques}, is defined via continuous-time arguments in \Section{cts_ub}.
Moreover, the structure of the proof is heavily inspired by the proof in the continuous setting.

Moving forward, we need a few observations about the functions $R$ and $p$,
which were defined in equations \eqref{eq:RDef} and \eqref{eq:DiscreteDerivative}.
First, we require two straightforward calculations.
These are special cases of \Lemma{uncap_rtg_deriv_cts}
(with $\tilde{R}_\gamma = \tilde{R}$ and $R_\gamma = R$).
For convenience, we restate them here without the subscript but only prove \Lemma{uncap_rtg_deriv_cts} later in the paper.
\begin{lemma}
	\LemmaName{f_derivative}
	Consider the function $\tilde{R}(t,g) = \frac{g}{2} + \kappa \sqrt{t} M_0\left( \frac{g^2}{2t} \right)$.
	Then $\frac{\partial}{\partial g} \tilde{R}(t,g) = \frac{1}{2}\left( 1 - \frac{\erfi(g/\sqrt{2t})}{\erfi(\gamma/\sqrt{2})} \right)$.
\end{lemma}
\noindent Note that $R(t,g) = \tilde{R}(t,g)$ for $g \leq \gamma \sqrt{t}$ and $R(t,g) = \tilde{R}(t, \gamma \sqrt{t}) = \gamma \sqrt{t} / 2$ for $g \geq \gamma \sqrt{t}$.
So one should think of $\tilde{R}$ as the ``untruncated'' version of $R$.
\begin{lemma}
	\LemmaName{rtg_derivative}
	$\frac{\partial}{\partial g} R(t,g) = \frac{1}{2} \left( 1 - \frac{\erfi(g/\sqrt{2t})}{\erfi(\gamma/\sqrt{2})} \right)_+$.
\end{lemma}

\begin{lemma}
    \LemmaName{rtg_increasing_concave}
    For any $t>0$, $R(t,g)$ is concave and non-decreasing in $g$.
\end{lemma}
\begin{proof}
	The fact that $R(t,g)$ is non-decreasing in $g$ follows from \Lemma{rtg_derivative}.
	The concavity of $R(t,g)$ (in $g$) follows from the fact that $\erfi$ is non-decreasing, so $\frac{\partial}{\partial g} R(t,g)$ is non-increasing in $g$.
\end{proof}
As a consequence of \Lemma{rtg_increasing_concave}, we can easily get the maximum value of $R(t,g)$ for any $t$.
\begin{lemma}
    \LemmaName{rtg_maximum}
    For any $t > 0$, we have $R(t,g) \leq \gamma\sqrt{t}/2$.
\end{lemma}
\begin{proof}
    \Lemma{rtg_increasing_concave} shows that $R(t,g)$ is non-decreasing in $g$.
    By definition, $R(t,g)$ is constant for $g \geq \gamma\sqrt{t}$.
    It follows that $\max_g R(t,g) \leq R(t, \gamma\sqrt{t}) = \gamma\sqrt{t} / 2$.
\end{proof}

In the definition of the prediction task, the algorithm must produce a probability vector $x_t$.
Recalling the definition of $x_t$ in \Algorithm{minimax}, it is not a priori clear whether $x_t$ is indeed a probability vector.
We now verify that it is, since \Lemma{ptg_well_defined} implies that $p(t,g) \in [0,1/2]$ for all $t,g$.
\begin{lemma}
    \LemmaName{ptg_well_defined}
    Fix $t \geq 1$.
    Then
    \begin{enumerate}[label=(\arabic*),noitemsep,topsep=0pt]
        \item $p(t,0) = 1/2$;
        \item $p(t,g)$ is non-increasing in $g$; and
        \item $p(t,g) \geq 0$.
    \end{enumerate}
\end{lemma}
\begin{proof}
    For the first assertion, we have
    \[
        p(t,0) = \frac{1}{2}(R(t,1) - R(t,-1)) = \frac{1}{2}\left( \frac{1}{2} + \kappa \sqrt{t} M_0(1/2t) + \frac{1}{2} - \kappa \sqrt{t} M_0(1/2t) \right) = \frac{1}{2}.
    \]
    For the second equality, we used that $1 \leq \gamma \leq \gamma \sqrt{t}$ for all $t \geq 1$.
    The second assertion follows from concavity of $R$, which was shown in \Lemma{rtg_increasing_concave},
    and an elementary property of concave functions (\Fact{concave_decreasing}).
    The final assertion holds because $R$ is non-decreasing in $g$, which is also shown in \Lemma{rtg_increasing_concave}.
\end{proof}

\subsection{Analysis when gap increments are \texorpdfstring{$\pm 1$}{pm 1}}
\SubsectionName{ub_integer_gaps}

In this subsection we prove the upper bound of \Theorem{main} for a restricted class of adversaries
(that nevertheless capture the core of the problem).
The analysis is extended to all adversaries in \Subsection{ub_general}.

\begin{theorem}
\TheoremName{UBSpecial}
Let $\cA$ be the algorithm described in \Algorithm{minimax}.
For any adversary $\cB$ such that each cost vector $\ell_t$ is either
$\begin{smallbmatrix}1\\0\end{smallbmatrix}$ or $\begin{smallbmatrix}0\\1\end{smallbmatrix}$, we have
\[
\sup_{t \geq 1} \frac{\Regret{2,t,\cA,\cB}}{\sqrt{t}} ~\leq~ \frac{\gamma}{2}.    
\]
\end{theorem}

Our analysis relies on an identity known as the discrete It\^{o} formula,
which is the discrete analogue of It\^{o}'s formula from stochastic analysis (see \Theorem{ItoFormula}).
To make this connection (in addition to future connections) more apparent, we define the discrete derivatives of a function $f$ to be
\begin{align*}
&f_g(t,g) = \frac{f(t,g+1) - f(t,g-1)}{2}, \\
&f_t(t,g) = f(t,g) - f(t-1,g), \\
&f_{gg}(t,g) = \big(f(t,g+1) + f(t,g-1)\big) - 2f(t,g).
\end{align*}
It was remarked earlier that $p(t,g)$ (see \Equation{DiscreteDerivative}) is the discrete derivative of $R$,
and this is because
\begin{equation}
\EquationName{PIsDiscreteDeriv}
    p(t,g) ~=~ R_g(t,g).
\end{equation}

\begin{lemma}[Discrete It\^{o} formula]
    \LemmaName{discrete_ito}
    Let $g_0, g_1, \ldots$ be any sequence of real numbers
    (not necessarily random) satisfying $|g_t - g_{t-1}| = 1$.
    Then for any function $f$ and any fixed time $T \geq 1$, we have
    \begin{equation}
    \EquationName{discrete_ito}
    \begin{split}
        f(T, g_T) - f(0, g_0) &~=~ \sum_{t=1}^T f_g(t, g_{t-1}) \cdot (g_t - g_{t-1})
        \:+\: \sum_{t=1}^T \left(\frac{1}{2} f_{gg}(t, g_{t-1}) + f_t(t, g_{t-1}) \right).
    \end{split}
    \end{equation}
\end{lemma}

This lemma is a small generalization of \cite[\S 2]{Fujita08} and \cite[Theorem 2]{Kudzma82}
to accommodate a bivariate function $f$ that depends on $t$.
The proof is essentially identical.
\begin{proof}
	By telescoping, $f(T,g_T) - f(0,g_0) = \sum_{t=1}^T \big(f(t,g_t) - f(t-1,g_{t-1})\big).$ Consider a fixed $t \in [T]$.
	We can write
	\begin{equation}
		\begin{split}
			f(t,g_t) - f(t-1,g_{t-1})
			& = \left( f(t,g_t) - \frac{f(t,g_{t-1}+1) + f(t,g_{t-1}-1)}{2} \right) \\
			& +
			\left( \frac{f(t,g_{t-1}+1) + f(t,g_{t-1}-1)}{2} - f(t-1,g_{t-1}) \right).
		\end{split}
	\end{equation}
	For the first bracketed term, by considering the cases $g_t = g_{t-1} + 1$ and $g_t = g_{t-1} - 1$, we have
	\begin{equation}
		\EquationName{discrete_first_deriv}
		\begin{split}
			f(t,g_t) - \frac{f(t,g_{t-1}+1) + f(t,g_{t-1}-1)}{2}
			& = \frac{f(t,g_{t-1}+1) - f(t,g_{t-1}-1)}{2} \cdot (g_t - g_{t-1}) \\
			& = f_g(t,g_{t-1}) \cdot (g_t - g_{t-1}).
		\end{split}
	\end{equation}
	Note that the above step is the only place where the assumption that $|g_t - g_{t-1}| = 1$ is used.
	For the second bracketed term, we have
	\begin{equation*}
		\begin{split}
			\frac{f(t,g_{t-1}+1) + f(t,g_{t-1}-1)}{2} - f(t-1,g_{t-1})
			& = \frac{f(t,g_{t-1} + 1) + f(t,g_{t-1} - 1) - 2f(t,g_{t-1})}{2} \\
			& + (f(t,g_{t-1}) - f(t-1,g_{t-1})) \\
			& = \frac{1}{2} f_{gg}(t,g_{t-1}) + f_t(t,g_{t-1}).
		\end{split}
	\end{equation*}
	This gives the desired formula.
\end{proof}

Now we show that the regret involves a discrete integral of the same form as \eqref{eq:discrete_ito}.
Let us recall that
\Lemma{ptg_well_defined}(1) guarantees $p(t,0) = 1/2$, i.e.,~$x_{t} = [1/2, 1/2]$.
Hence, \eqref{eq:RegretInGapSpace2} gives
\begin{equation}
    \EquationName{regret_formula_ub_01}
    \Reg(T) = \sum_{t=1}^T p(t,g_{t-1})\cdot (g_t - g_{t-1})
\end{equation}
where $g_0 = 0$ and $g_t \geq 0$ for all $t \geq 1$.

\paragraph{Key technical step.} The following is the most non-obvious step of the proof.
We apply the discrete It\^{o} formula to \Equation{regret_formula_ub_01}, taking $f=R$.
Since $p = R_g = f_g$, observe that the main difference between \Equation{discrete_ito} and \Equation{regret_formula_ub_01} is the absence of $\frac{1}{2} f_{gg}(t, g_{t-1}) + f_t(t, g_{t-1})$ in \Equation{regret_formula_ub_01}.
In the continuous setting,
we will see that a key idea is to try to obtain a solution satisfying
$(\frac{1}{2} \partial_{gg} + \partial_t) f = 0$;
this is the well-known backwards heat equation.
In the discrete setting, by a remarkable stroke of luck, we have the following analogous property.

\begin{lemma}[Discrete backwards heat inequality]
    \LemmaName{ub_discrete_bhe}
    $\frac{1}{2} R_{gg}(t,g) + R_t(t, g) \geq 0$
    for all $t \geq 1$ and $g \geq 0$.
\end{lemma}

\noindent This lemma is the most technical part of the discrete analysis. 
Its proof appears in \Subsection{ub_technical}.

\begin{proofof}{\Theorem{UBSpecial}}
Apply \Lemma{discrete_ito} to the function $R$ and the sequence $g_0,g_1,\ldots$
of (integer) gaps produced by the adversary $\cB$.
Then, for any time $T \geq 0$,
\begin{alignat*}{2}
&R(T, g_T) - R(0,g_0) \\
&\quad=~    \sum_{t=1}^T R_g(t, g_{t-1}) \cdot (g_t - g_{t-1})
            + \sum_{t=1}^T \Big( \frac{1}{2} R_{gg}(t, g_{t-1}) + R_t(t, g_{t-1}) \Big)
            &&\qquad\text{(by \Lemma{discrete_ito})}\\
&\quad\geq~ \sum_{t=1}^T p(t, g_{t-1}) \cdot (g_t - g_{t-1})
            &&\qquad\text{(by \eqref{eq:PIsDiscreteDeriv} and \Lemma{ub_discrete_bhe})}\\
&\quad=~    \Reg(T)
            &&\qquad\text{(by \eqref{eq:regret_formula_ub_01}).}
\end{alignat*}
Since $g_0 = 0$ and $R(0,0) = 0$, applying \Lemma{rtg_maximum} shows that
$\Reg(T) \leq R(T, g_T) \leq \gamma\sqrt{T}/2$.
\end{proofof}

This completes the proof of \Theorem{UBSpecial}.
However, the proof does not reveal why \Algorithm{minimax} is optimal.
The constant $\gamma$ in the regret bound appears due to properties of the function $R$, whose definition has yet to be explained.
In \Section{cts_ub}, we will define the function $R$ specifically to obtain $\gamma$ in the preceding analysis.
In \Section{lb}, we prove a matching lower bound and show that $\gamma$ is indeed the right constant.

We remark that the proof of \Theorem{UBSpecial} may also be viewed as an amortized analysis, in which the algorithm incurs amortized regret at most $\frac{\gamma}{2}(\sqrt{t}-\sqrt{t-1}) \approx \sfrac{\gamma}{4\sqrt{t}}$ at each time step $t$.
This viewpoint may be helpful to see how the potential function used in our setting relates to the potential functions in traditional algorithm design 
\cite[\S 17.3]{CLRS}.

\subsection{Proof of \Lemma{ub_discrete_bhe}}
\SubsectionName{ub_technical}
\begin{lemma}
	\LemmaName{M_geq_sqrt}
	For all $u \in [0,1/2]$, we have $M_0(u) \geq \sqrt{1 - 2u}$.
\end{lemma}
\begin{proof}
	The Maclaurin expansion of $M_0(u)$ is given by
	\[
	M_0(u) = 1 - \sum_{k=1}^{\infty} \frac{1}{(2k-1)k!} u^k.
	\]
	Note that $\frac{\mathrm{d}^k}{\mathrm{d} x^k} \sqrt{1 - 2x} = -\frac{(2k-3)!!}{(1-2x)^{(2k-1)/2}}$,
	where $(n)!!$ denotes the double factorial (note that $(-1)!! = 1$).\footnote{If $n \in \bZ_{\geq 0}$, we define $(n)!! = \prod_{k=0}^{\lceil n/2 \rceil - 1}(n-2k)$. If $n \in \bZ_{< 0}$, we define $(n)!!$ via the recursive relation $(n)!! = \frac{(n+2)!!}{n+2}$ so that $(-1)!! = \frac{(1)!!}{1} = 1$.}
	Hence, the Maclaurin expansion of $\sqrt{1-2u}$ is
	\[
	\sqrt{1-2u} = 1 - \sum_{k=1}^{\infty} \frac{(2k-3)!!}{k!} u^k.
	\]
	It is not hard to verify that $(2k-3)!! \geq \frac{1}{2k-1}$.
	This implies that $M_0(u) \geq \sqrt{1-2u}$.
\end{proof}

\begin{lemma}
	\LemmaName{rtg_ineq}
	For all $z \in [0,1)$ and $x \in \bR$, we have
	\[
	M_0\left( \frac{(x+z)^2}{2} \right) + M_0\left( \frac{(x-z)^2}{2} \right) \geq 2 \sqrt{1-z^2} M_0\left( \frac{x^2}{2(1-z^2)} \right).
	\]
\end{lemma}
\begin{proof}
	Fix $z \in [0,1)$ and consider the function
	\[
	h_z(x) = M_0\left( \frac{(x+z)^2}{2} \right) + M_0\left( \frac{(x-z)^2}{2} \right) - 2 \sqrt{1-z^2} M_0\left( \frac{x^2}{2(1-z^2)} \right).
	\]
	Note that $h_z(0) \geq 0$ by applying \Lemma{M_geq_sqrt} with $u = z^2/2$.
	We will show that $x = 0$ is the minimizer of $h_z$ which implies the lemma.
	
	Indeed, computing derivatives, we have
	\begin{align*}
		h_z'(x) = -M_1\left( \frac{(x+z)^2}{2} \right) \cdot (x+z) - M_1\left( \frac{(x-z)^2}{2} \right) \cdot (x-z)
		+ 2 M_1\left( \frac{x^2}{2(1-z^2)} \right) \cdot \frac{x}{\sqrt{1-z^2}}.
	\end{align*}
	As $h_z'(0) = 0$, $x = 0$ is a critical point of $h_z$.
	We will now show that $h_z$ is convex which certifies that $x = 0$ is indeed a minimizer.
	
	To obtain $h_z''$, we differentiate term-by-term. Let $u = \frac{(x+z)^2}{2}$.
	Then
	\begin{align*}
		\diffwrt{x} M_1\left( \frac{(x+z)^2}{2} \right) \cdot (x+z)
		& = \frac{M_2\left( \frac{(x+z)^2}{2} \right) \cdot (x+z)^2}{3} + M_1\left( \frac{(x+z)^2}{2} \right) \\
		& = \frac{2M_2(u)\cdot u}{3} + M_1(u) \\
		& = \frac{2u(2e^u\sqrt{u} - \sqrt{\pi}\erfi(\sqrt{u}))}{4u^{3/2}} + \frac{\sqrt{\pi}\erfi(\sqrt{u})}{2\sqrt{u}} \\
		& = e^u = \exp\left( \frac{(x+z)^2}{2} \right).
	\end{align*}
	The first equality is by \Fact{chf_derivs} and the third equality is by
	identities \ref{item:M1} and \ref{item:M2} in \Fact{basic_identities}. 
	We can similarly show that
	\[
	\diffwrt{x} M_1\left( \frac{(x-z)^2}{2} \right) \cdot (x-z) = \exp\left( \frac{(x-z)^2}{2} \right).
	\]
	Finally, for the last term, we have
	\begin{align*}
		\diffwrt{x} M_1\left( \frac{x^2}{2(1-z^2)} \right) \cdot \frac{x}{\sqrt{1-z^2}}
		& = \frac{1}{3} M_2\left( \frac{x^2}{2(1-z^2)} \right) \cdot \frac{x^2}{(1-z^2)^{3/2}} + M_1\left( \frac{x^2}{2(1-z^2)} \right) \cdot \frac{1}{\sqrt{1-z^2}} \\
		& = \frac{1}{\sqrt{1-z^2}} \left(\frac{2}{3} M_2\left( \frac{x^2}{2(1-z^2)} \right) \cdot \frac{x^2}{2(1-z^2)} + M_1\left( \frac{x^2}{2(1-z^2)} \right) \right) \\
		& = \frac{ \exp\left( \frac{x^2}{2(1-z^2)} \right)}{\sqrt{1-z^2}},
	\end{align*}
	where the first equality uses \Fact{chf_derivs} and the last equality is by identity \ref{item:M1M2} in \Fact{basic_identities}.
	
	Hence, we have
	\[
	h_z''(x) = \frac{ 2e^{x^2/2(1-z^2)} - (e^{(x+z)^2/2} + e^{(x-z)^2/2}) \sqrt{1-z^2}}{\sqrt{1-z^2}}.
	\]
	So to check that $h_z''(x) \geq 0$ for all $x \in \bR$, it suffices to check that
	\[
	\frac{(e^{(x+z)^2/2} + e^{(x-z)^2/2})\sqrt{1-z^2}}{2} \leq e^{x^2/2(1-z^2)}.
	\]
	Indeed, we have
	\begin{align*}
		\frac{(e^{(x+z)^2/2} + e^{(x-z)^2/2})\sqrt{1-z^2}}{2}
		& \leq \frac{(e^{(x+z)^2/2} + e^{(x-z)^2/2}) e^{-z^2/2}}{2}\\
		& = e^{x^2/2} \frac{(e^{xz} + e^{-xz})}{2} \\
		& \leq e^{x^2/2} e^{x^2z^2/2} \\
		& = e^{x^2(1+z^2)/2} \\
		& \leq e^{x^2/2(1-z^2)},
	\end{align*}
	where the first inequality is because $1-a \leq e^{-a}$ for all $a \in \bR$,
	the second inequality is because $(e^a+e^{-a})/2 = \cosh(a) \leq e^{a^2/2}$ for all $a \in \bR$,
	and the last inequality is because $1+a \leq 1/(1-a)$ for all $a < 1$.
	This proves that $h_z$ is convex which concludes the proof that $x = 0$ is a minimizer for $h_z$ and hence, completes the proof of the lemma.
\end{proof}

\begin{proofof}{\Lemma{ub_discrete_bhe}}
	The inequality $R_t(t,g) + \frac{1}{2} R_{gg}(t,g) \geq 0$ is equivalent to
	\begin{equation}
		\EquationName{discrete_bhe_ineq}
		R(t,g+1) + R(t, g-1) \geq 2R(t-1, g).
	\end{equation}
	We first prove the claim for $t = 1$.
	In this case, the RHS of \Equation{discrete_bhe_ineq} is identically $0$.
	On the other hand, the LHS of \Equation{discrete_bhe_ineq} is non-decreasing in $g$ by \Lemma{rtg_increasing_concave}.
	Hence, it suffices to prove the inequality for $g = 0$.
	With $t = 1$ and $g = 0$, we have
	\[
	R(1,1) + R(1,-1) = 2\kappa M_0(1/2).
	\]
	As $M_0$ is decreasing (\Fact{M0_decreasing_concave}) and $1/2 \leq \gamma^2/2$, we have $M_0(1/2) \geq M_0(\gamma^2/2) = 0$.
	So \Equation{discrete_bhe_ineq} holds for $t = 1$ and $g \geq 0$.
	
	For the remainder of the proof, we assume that $t > 1$.
	We consider a few cases depending on the value of $g$ and $t$.
	\paragraph{Case 1: $g \leq \min\{\gamma \sqrt{t} - 1, \gamma \sqrt{t-1}\}$.}
	In this case, $g+1 \leq \gamma \sqrt{t}$, $g \leq \gamma \sqrt{t-1}$, and $g-1 \leq \gamma \sqrt{t}$.
	Hence,
	\begin{align*}
		& R(t,g+1) = \frac{g+1}{2} + \kappa \sqrt{t} \cdot M_0\left( \frac{(g+1)^2}{2t} \right) \\
		& R(t,g-1) = \frac{g-1}{2} + \kappa \sqrt{t} \cdot M_0\left( \frac{(g-1)^2}{2t} \right) \\
		& R(t-1,g) = \frac{g}{2} + \kappa \sqrt{t} \cdot M_0\left( \frac{g^2}{2(t-1)} \right).
	\end{align*}
	So \Equation{discrete_bhe_ineq} is equivalent to
	\begin{equation}
		\EquationName{bhe_to_verify}
		\sqrt{t}\cdot M_0\left( \frac{(g+1)^2}{2t} \right) + \sqrt{t} \cdot M_0\left( \frac{(g-1)^2}{2t} \right) \geq
		2\sqrt{t-1} \cdot M_0\left( \frac{g^2}{2(t-1)} \right),
	\end{equation}
	or rearranging, is equivalent to
	\[
	M_0\left( \frac{(g+1)^2}{2t} \right) + M_0\left( \frac{(g-1)^2}{2t} \right) \geq
	2\sqrt{1-1/t} \cdot M_0\left( \frac{g^2}{2(t-1)} \right).
	\]
	The latter inequality is true by \Lemma{rtg_ineq} using $x = g/\sqrt{t}$ and $z = 1/\sqrt{t} \in (0, 1)$.
	\paragraph{Case 2: $\gamma \sqrt{t} - 1 \leq g \leq \gamma\sqrt{t-1}$.}
	Let $\tilde{R}$ be the function defined in \Lemma{f_derivative}.
	In this case, we have
	\[
	R(t,g+1) = \gamma\sqrt{t}
	= \tilde{R}(t, \gamma \sqrt{t})
	\geq \tilde{R}(t, g+1)
	= \frac{g+1}{2} + \kappa \sqrt{t} \cdot M_0\left( \frac{(g+1)^2}{2t} \right).
	\]
	The inequality is by \Lemma{f_derivative} which implies that $\tilde{R}(t,g+1)$ is \emph{non-increasing} for $g \in (\gamma\sqrt{t}-1, \infty)$.
	Using the lower bound on $R(t,g+1)$,
	\Equation{discrete_bhe_ineq} is again implied by \Equation{bhe_to_verify} and
	we have already verified that \Equation{bhe_to_verify} is true.
	\paragraph{Case 3: $\gamma \sqrt{t-1} \leq g \leq \gamma\sqrt{t}-1$.}
	In this case
	\begin{align*}
		& R(t,g+1) = \frac{g+1}{2} + \kappa \sqrt{t} \cdot M_0\left( \frac{(g+1)^2}{2t} \right) \\
		& R(t,g-1) = \frac{g-1}{2} + \kappa \sqrt{t} \cdot M_0\left( \frac{(g-1)^2}{2t} \right) \\
		& R(t-1,g) = \frac{\gamma}{2} \sqrt{t-1}.
	\end{align*}
	As $g \leq \gamma \sqrt{t}-1$, we have $M_0\left( \frac{(g-1)^2}{2t} \right) \geq M_0\left( \frac{(g+1)^2}{2t} \right) \geq M_0\left( \frac{\gamma^2}{2} \right) = 0$.
	Here, the first two inequalities are because $M_0$ is decreasing (\Fact{M0_decreasing_concave}).
	Hence,
	\[
	R(t,g+1) + R(t,g-1) \geq g \geq \gamma \sqrt{t-1} = 2R(t-1, g),
	\]
	which is precisely \Equation{discrete_bhe_ineq}.
	\paragraph{Case 4: $\max\{\gamma\sqrt{t-1}, \gamma\sqrt{t}-1\} \leq g$.}
	In this case, $R(t-1,g)$ and $R(t,g+1)$ are constant in $g$ but $R(t,g-1)$ is non-decreasing in $g$.
	Hence, it suffices to check \Equation{discrete_bhe_ineq} for $g = \max\{\gamma \sqrt{t-1}, \gamma \sqrt{t}-1\}$ which holds by either case 2
	(if $\gamma \sqrt{t}-1 \leq \gamma \sqrt{t-1}$) or case 3 (if $\gamma \sqrt{t-1} \leq \gamma \sqrt{t} - 1$).
\end{proofof}

\subsection{Analysis of \Algorithm{minimax} for general cost vectors}
\SubsectionName{ub_general}
In this section, we prove the upper bound of \Theorem{main} in full generality.

\begin{theorem}
	\TheoremName{UBGeneral}
	Let $\cA$ be the algorithm described in \Algorithm{minimax}.
	For any adversary $\cB$ (allowing any cost vectors $\ell_t \in [0,1]^2$), we have
	\[
	\sup_{t \geq 1} \frac{\Regret{2,t,\cA,\cB}}{\sqrt{t}} ~\leq~ \frac{\gamma}{2}.    
	\]
\end{theorem}

In \Subsection{ub_integer_gaps}, since the gap was integer-valued, the identity of the best expert could only change when the gap is exactly $0$ (at which time there are two best experts).
In general, the gap can be real-valued, so the best expert can switch abruptly, which affects our formula for the regret.
We will need to generalize \Proposition{RegretWithGaps} to deal with this possibility.
Let $\DeltaR(t) = \Reg(t) - \Reg(t-1)$.

\begin{prop}
	\PropositionName{delta}
	Let $g_{t-1}$ be the gap after time $t-1$ but before playing an action at time $t$. Let $g_t$ be the gap after time $t$.
	Let $p(t, g_{t-1})$ denote the probability mass assigned to the worst expert at time $t$.
	Suppose that $p(t, 0) = 1/2$ for all $t \geq 1$.
	\begin{enumerate}[noitemsep,nosep=0pt]
		\item
		If a best expert at time $t-1$ remains a best expert at time $t$ then
		\[
		\DeltaR(t) = (g_{t} - g_{t-1}) p(t,g_{t-1}).
		\]
		\item If a best expert at time $t-1$ is no longer a best expert at time $t$ then
		\[
		\DeltaR(t) = g_{t} - (g_{t} + g_{t-1}) p(t,g_{t-1}).
		\]
		Moreover, $g_t + g_{t-1} \leq 1$.
	\end{enumerate}
\end{prop}
The proof of this is very similar to that of \Proposition{RegretWithGaps} and appears in \Appendix{prop_delta_proof}
\begin{remark}
	Note that, at any specific time, the set of best experts may have size either one or two so the choice of the best expert in \Proposition{delta} may be ambiguous.
	However, note that if $g_{t-1} = 0$ (i.e., there are two best experts at time $t-1$) then $p(t, g_{t-1}) = 1/2$ so both formulas give $\DeltaR(t) = \frac{1}{2} g_{t}$.
	On the other hand, if $g_t = 0$ (i.e., there are two best experts at time $t$) then both formulas give $\DeltaR(t) = -g_{t-1} p(t, g_{t-1})$.
	Hence there is no issue with the ambiguity.
\end{remark}

We will need the following identity which is essentially the same as \Lemma{discrete_ito} but without specializing to the case where $|g_t - g_{t-1}| = 1$.
\begin{lemma}
	\LemmaName{gen_discrete_ito}
	Let $g_0, g_1, \ldots$ be a sequence of real numbers.
	Then for any function $f$ and any fixed time $T \geq 1$, we have
	\begin{equation}
		\EquationName{gen_discrete_ito}
		\begin{split}
			f(T, g_T) - f(0, g_0) &~=~
			\sum_{t=1}^T f(t, g_t) - \frac{f(t, g_{t-1} + 1) + f(t, g_{t-1} - 1)}{2} \\
			& \:+\: \sum_{t=1}^T \left(\frac{1}{2} f_{gg}(t, g_{t-1}) + f_t(t, g_{t-1}) \right).
		\end{split}
	\end{equation}
\end{lemma}
\begin{proof}
	The proof is identical to the proof of \Lemma{discrete_ito} except that we do not perform the simplification in \Equation{discrete_first_deriv}.
\end{proof}
When we assumed the gaps were integer-valued, we had
\[
\DeltaR(t) = R(t, g_t) - \frac{R(t, g_{t-1} + 1) + R(t, g_{t-1} - 1)}{2}
\]
because both sides were equal to $R_g(t, g_{t-1}) \cdot (g_t - g_{t-1})$; see \Equation{RegretInGapSpace2} and \Equation{discrete_first_deriv}.
This does not hold in the general setting, but we will be able to prove the following inequality.
\begin{lemma}
	\LemmaName{delta_R_ineq}
	For all $t \geq 1$,
	\[
	\DeltaR(t) \leq R(t, g_t) - \frac{R(t, g_{t-1} + 1) + R(t, g_{t-1} - 1)}{2}.
	\]
\end{lemma}
The proof of \Lemma{delta_R_ineq} appears in \Appendix{delta_R_ineq_proof}.
Given \Lemma{delta_R_ineq}, we can now prove our upper bound in general.
\begin{proofof}{\Theorem{UBGeneral}}
	Fix any $T \geq 1$. Then
	\begin{alignat*}{2}
		R(T, g_T) - R(0, g_0)
		&~=~
		\sum_{t=1}^T R(t, g_t) - \frac{R(t, g_{t-1} + 1) + R(t, g_{t-1} - 1)}{2}  \\
		& \:+\: \sum_{t=1}^T \left(\frac{1}{2} R_{gg}(t, g_{t-1}) + R_t(t, g_{t-1}) \right) && \quad\text{(\Lemma{gen_discrete_ito})} \\
		&~\geq~\sum_{t=1}^T \DeltaR(t) && \quad\text{(\Lemma{delta_R_ineq} and \Lemma{ub_discrete_bhe})} \\
		&~=~ \Reg(T).
	\end{alignat*}
	As $g_0 = 0$ and $R(0, 0) = 0$,
	we have $\Reg(T) \leq R(T, g_T) \leq \gamma \sqrt{T} / 2$, where the last inequality is by \Lemma{rtg_maximum}.
\end{proofof}

\subsubsection{Proof of \Proposition{delta}}
\AppendixName{prop_delta_proof}
\begin{proofof}{\Proposition{delta}}
	Fix $t$ and for notational convenience, let $p = p(t, g_{t-1})$ throughout the proof.
	In addition, throughout the proof, we use expert 1 to refer to the worst expert at time $t-1$ (chosen arbitrarily if the choice of worst expert is not unique) and use expert 2 to refer to the other expert.
	Let $\ell_{t,1}, \ell_{t,2} \in [0,1]$ be the respective losses at time $t$ and $L_{t,1}, L_{t,2}$ be the respective \emph{cumulative} losses up to time $t$.
	Note that $g_{t-1} = L_{t-1,1} - L_{t-1,2}$.
	Finally, we set $L_t^* = \min_{i \in \{1,2\}} L_{t,i}$.
	By assumption, $L_{t-1}^* = L_{t-1,2}$.
	
	For the first assertion we have $L_t^* = L_{t,2}$ (because a best expert remains a best expert).
	Note that $\ell_{t,1} - \ell_{t,2} = (L_{t,1} - L_{t,2}) - (L_{t-1,1} - L_{t-1,2}) = g_t - g_{t-1}$.
	So the change in the cost of the algorithm can be written as
	\[
	p\ell_{t,1} + (1-p)\ell_{t,2} = p(\ell_{t,1} - \ell_{t,2}) + \ell_{t,2} = p(g_t - g_{t-1}) + \ell_{t,2}. 
	\]
	On the other hand, the change in the cost of the best expert is $L_{t}^* - L_{t-1}^* = L_{t,2} - L_{t-1,2} = \ell_{t,2}$.
	Subtracting this from the above equation gives $\DeltaR(t) = (g_{t}-g_{t-1})p$.
	
	In the second assertion, we have $L_t^* = L_{t,1}$, so $g_t = L_{t,2} - L_{t,1}$.
	Again, the algorithm incurs cost $p\ell_{t,1} + (1-p)\ell_{t,2}$.
	This time, note that $\ell_{t,1} - \ell_{t,2} = (L_{t,1} - L_{t,2}) - (L_{t-1,1} - L_{t-1,2}) = -g_t - g_{t-1}$.
	So the algorithm incurs cost $-p(g_t + g_{t-1}) + \ell_{t,2}$.
	On the other hand, the change in the cost of the best expert is
	\[
	L_t^* - L_{t-1}^* = L_{t,1} - L_{t-1,2} = L_{t,1} - L_{t-1,1} + L_{t-1,1} - L_{t-1,2} = \ell_{t,1} + g_{t-1} = \ell_{t,2} - g_{t},
	\]
	where the last equality uses the identity $\ell_{t,1} - \ell_{t,2} = -g_t - g_{t-1}$.
	Subtracting this last quantity with the change in the algorithm's cost gives $\DeltaR(t) = g_{t} - p(g_t + g_{t-1})$.
	
	To complete the proof for the second assertion, it remains to check that $g_t + g_{t-1} \leq 1$.
	From above, we have the identity, $g_t + g_{t-1} = \ell_{t,2} - \ell_{t,1} \leq \ell_{t,2} \leq 1$, as desired.
\end{proofof}

\subsubsection{Proof of \Lemma{delta_R_ineq}}
\AppendixName{delta_R_ineq_proof}
\begin{proofof}{\Lemma{delta_R_ineq}}
	Fix $t \geq 1$.
	We will consider the two cases corresponding to the two cases in \Proposition{delta}.
	
	\paragraph{Case 1: A best expert at time $t-1$ remains a best expert at time $t$.}
	In this case, $\DeltaR(t) = (g_t - g_{t-1}) p(t, g_{t-1})$,
	so it suffices to check that
	\begin{equation}
		\EquationName{gen_ineq_1}
		p(t,g_{t-1})\cdot(g_t - g_{t-1})
		\leq 
		R(t, g_t) - \frac{R(t, g_{t-1} + 1) + R(t, g_{t-1}-1)}{2}.
	\end{equation}
	Rearranging, the above inequality is equivalent to
	\[
	R(t, g_t) - \frac{R(t, g_{t-1} + 1) + R(t, g_{t-1}-1)}{2}
	- p(t, g_{t-1}) \cdot (g_t - g_{t-1}) \geq 0.
	\]
	If $g_{t-1}$ is fixed then notice that the LHS of the above expression is concave in $g_t$.
	To see this, \Lemma{rtg_increasing_concave} implies that $R(t,g_t)$ is concave in $g_t$, the second term is constant in $g_t$, and the last term is linear in $g_t$.
	Hence, it suffices to verify the inequality when $g_t = g_{t-1} \pm 1$ (\Fact{concave_ineq_check}).
	Indeed, if $|g_t - g_{t-1}| = 1$ then, as in \Equation{discrete_first_deriv}
	\begin{align*}
		R(t,g_t) - \frac{R(t, g_{t-1}+1) + R(t, g_{t-1}-1)}{2}
		& = \frac{R(t, g_{t-1}+1) - R(t, g_{t-1}-1)}{2} \cdot (g_t - g_{t-1}) \\
		& = p(t, g_{t-1}) \cdot (g_t - g_{t-1}),
	\end{align*}
	where the second equality used the definition of $p$.
	
	\paragraph{Case 2: A best expert at time $t-1$ is no longer a best expert at time $t$.}
	This case is nearly identical to the previous case but in this case $\DeltaR(t) = g_t - (g_t + g_{t-1}) p(t, g_{t-1})$ with the promise that $g_t + g_{t-1} \leq 1$.
	Hence, the inequality we need to verify is that
	\begin{equation}
		\EquationName{gen_ineq_2}
		g_t - (g_t + g_{t-1}) p(t, g_{t-1}) \leq R(t, g_t) - \frac{R(t, g_{t-1} + 1) + R(t, g_{t-1} - 1)}{2}.
	\end{equation}
	Once again, we do this via a concavity argument.
	Fix $g_{t-1} \in [0,1]$.
	Since $g_t + g_{t-1} \leq 1$, we have $g_t \in [0, 1-g_{t-1}]$.
	Notice that the LHS of \Equation{gen_ineq_2} is linear in $g_t$ and the RHS of \Equation{gen_ineq_2} is concave in $g_t$ (by \Lemma{rtg_increasing_concave}).
	Hence, again by \Fact{concave_ineq_check}, it suffices to check the inequality assuming $g_t \in \{0, 1-g_{t-1}\}$.
	Note that the case $g_t = 0$ is handled by case 1 since the LHS of \Equation{gen_ineq_1} and \Equation{gen_ineq_2} are identical (see also the remark after \Proposition{delta}).
	
	Now assume that $g_t = 1-g_{t-1}$.
	Then \Equation{gen_ineq_2} becomes
	\[
	1-g_{t-1} - p(t, g_{t-1})
	\leq R(t, 1-g_{t-1}) - \frac{R(t, g_{t-1} + 1) + R(t, g_{t-1} - 1)}{2} 
	\]
	Recall that $p(t, g) = \frac{R(t, g+1) - R(t,g-1)}{2}$ so that the above inequality is equivalent to
	\[
	1 - g_{t-1} - \frac{R(t, g_{t-1} + 1) - R(t, g_{t-1} - 1)}{2} 
	\leq R(t, 1-g_{t-1}) - \frac{R(t, g_{t-1} + 1) + R(t, g_{t-1} - 1)}{2}.
	\]
	Rearranging the inequality becomes
	\[
	1 \leq g_{t-1} + R(t, 1 - g_{t-1}) - R(t, g_{t-1}-1). 
	\]
	Note that $g_{t-1} \leq 1 \leq \gamma \sqrt{t}$ (since $t \geq 1$ and $\gamma \geq 1$).
	Hence, by definition of $R$, the RHS of the above inequality is
	\begin{align*}
		g_{t-1} + R(t, 1 - g_{t-1}) - R(t, g_{t-1}-1)
		& = g_{t-1} + \frac{1-g_{t-1}}{2} + \kappa \sqrt{t} M_0\left( \frac{(1-g_{t-1})^2}{2} \right) \\
		& - \frac{g_{t-1} - 1}{2} - \kappa \sqrt{t} M_0\left( \frac{(g_{t-1}-1)^2}{2} \right) \\
		& = 1,
	\end{align*}
	and obviously, $1 \leq 1$.
	This proves that the desired inequality holds with equality.
\end{proofof}

%% file: lower_bound.tex
\section{Lower bound}
\SectionName{lb}

The main result of this section is the following theorem, which implies the lower bound in \Theorem{main}.

\begin{theorem}
\TheoremName{LB}
For any algorithm $\cA$ and any $\epsilon > 0$, there exists an adversary $\cB_\epsilon$ such that
\begin{equation}
\EquationName{GoalInLB}
\sup_{t \geq 1} \frac{\Regret{2,t,\cA,\cB_\epsilon}}{\sqrt{t}} ~\geq~ \frac{\gamma-\eps}{2}.    
\end{equation}
\end{theorem}
\noindent As remarked earlier, the $\sup$ can be replaced by a $\limsup$; see \Subsection{regretIO}.

It is common in the literature for regret lower bounds to be proven by
random adversaries; see, e.g., \cite[Theorem 3.7]{CBL}.
We will also consider a random adversary, but the novelty is the use of a non-trivial stopping time at which it can be shown that the regret is large.

\paragraph{A random adversary.}
Suppose an adversary produces a sequence of cost vectors
$\ell_1, \ell_2, \ldots \in \set{0,1}^2$ as follows.
For all $t \geq 1$,
\begin{itemize}[topsep=0pt,itemsep=0pt]
    \item If $g_{t-1}>0$ then $\ell_t$ is randomly chosen to be one of the vectors $\begin{smallbmatrix}1\\0\end{smallbmatrix}$ or $\begin{smallbmatrix}0\\1\end{smallbmatrix}$,
          uniformly and independent of $\ell_1,\ldots,\ell_{t-1}$.
          Thus $g_t-g_{t-1}$ is uniform in $\set{\pm 1}$.
    \item If $g_{t-1}=0$ then $\ell_t = \begin{smallbmatrix}1\\0\end{smallbmatrix}$ if $x_{t,1} \geq 1/2$, and $\ell_t = \begin{smallbmatrix}0\\1\end{smallbmatrix}$ if $x_{t,2} > 1/2$.
          In both cases $g_t=1$.
\end{itemize}
As remarked above, the process $( g_t )_{t \geq 0}$ has the same distribution as the absolute value of a standard random walk (which is also known as a reflected random walk).

We now obtain from \eqref{eq:RegretInGapSpace} a lower bound on the regret of any algorithm against this adversary. 
The adversary's behavior  when $g_{t-1}=0$ ensures that $\inner{x_t}{\ell_t} \geq 1/2$, showing that
\[
\Regret{T} ~\geq~ \underbrace{\sum_{t=1}^T p_t \left( g_t - g_{t-1} \right )\cdot \ind[g_{t-1} \neq 0]}_{\text{martingale}}
 ~+~ \frac{1}{2} \underbrace{\sum_{t=1}^T \ind [g_{t-1} = 0]}_{\text{local time}}
 \qquad\forall T \in \bN.
\]
(Equality holds if the algorithm sets $x_t=[1/2,1/2]$ whenever $g_{t-1}=0$.)
The first sum is a martingale indexed by $T$.
(This holds because $g_t-g_{t-1}$ has conditional expectation $0$ when $g_{t-1} \neq 0$,
and $\ind[g_{t-1} \neq 0]=0$ when $g_{t-1}=0$.)
The second sum is called the local time of the random walk.
Using Tanaka's formula \cite[Ex.~10.8]{Klenke}, the local time
can be written as $\sum_{s=1}^t \ind [g_{s-1} = 0] = g_t - Z'_t$ where $Z'_t$ is a martingale with uniformly bounded increments
and $Z'_0=0$.
Thus, combining the two martingales, we have
\begin{equation}
\EquationName{LBRegret}
\Regret{t} ~\geq~ Z_t + \frac{g_t}{2}
\qquad\forall t \in \bZ_{\geq 0},
\end{equation}
where $Z_t$ is a martingale with uniformly bounded increments and $Z_0=0$.

\paragraph{Intuition for a stopping time.} 
Optional stopping theorems assert that, under some hypotheses,
the expected value of a martingale at a stopping time equals the value at the start.
Using such a theorem, at a stopping time $\tau$ it would hold that
$\expect{\Regret{\tau}} \geq \expect{ g_{\tau} } / 2$
(under some hypotheses on $\tau$ and $Z$).
Thus it is natural to design a stopping time $\tau$ that maximizes $\expect{ g_{\tau} }$ and satisfies the hypotheses.
We know from \eqref{eq:KnownTwoExpertBounds} that the optimal anytime regret at time $t$ is $\Theta(\sqrt{t})$,
so one reasonable stopping time would be
$$
\tau(c) ~:=~ \min \setst{ t>0 }{ g_t \geq c \sqrt{t} }
$$
for some constant $c$ yet to be determined.
If $\tau(c)$ and $Z$ satisfy the hypotheses of the optional stopping theorem,
then it will hold that $\expect{\Regret{\tau(c)}} \geq \frac{c}{2} \smallexpect{ \sqrt{\tau(c)} }$.
From this, it follows, fairly easily,
that $\AnytimeValue(2) \geq c/2$; this will be argued more carefully later.

\paragraph{An optional stopping theorem.}
The optional stopping theorems appearing in standard references require one of the following hypotheses:
(i) $\tau$ is almost surely bounded, or
(ii) $\expect{\tau}$ is bounded and the martingale has bounded increments, or
(iii) the martingale is almost surely bounded and $\tau$ is almost surely finite.
See, e.g., \cite[Theorem 5.33]{BreimanBook}, \cite[Theorem 4.8.5]{Durrett}, \cite[Theorem 10.11]{Klenke}, \cite[Theorem 12.5.1]{GS}, \cite[Theorem II.57.4]{RogersWilliamsI}, or \cite[Theorem 10.10]{Williams}.
These will not suffice for our purposes.
For example, condition (ii) is the only useful hypothesis for our setting. 
It is known \cite{Breiman, Shepp} that $\expect{\tau(c)} < \infty$, with $\tau(c)$ as above, if and only if $c < 1$;
this yields a weak lower bound on the regret.
Instead, we will require the following theorem, which has a weaker hypothesis (due to the square root).
We are unable to find a reference for this theorem, although it is presumably folklore,
so we provide a proof of this theorem.

\begin{theorem}
\TheoremName{OST}
Let $Z_t$ be a martingale and $K>0$ a constant such that $|Z_t - Z_{t-1}| \leq K$ almost surely for all $t$.
Let $\tau$ be a stopping time.
If $\expect{\sqrt{\tau}} < \infty$ then
$\expect{Z_{\tau}} = \expect{Z_0}$.
\end{theorem}

Before we prove \Theorem{OST}, some preliminary definitions are required.
For a martingale $(X_t)_{t \in \bN}$, define
its maximum process $X^*_t = \max_{0 \leq s \leq t} \abs{X_s}$
and its quadratic variation process $[X]_t = \sum_{1 \leq s \leq t} (X_s-X_{s-1})^2$.

\begin{theorem}[Davis~\cite{Davis}]
	\TheoremName{BDG}
	There exists a constant $C$ such that for any martingale $(X_t)_{t \in \bN}$ with $X_0=0$,
	$\expect{ X^*_\infty } \leq C \expect{[X]_\infty^{1/2}}$.
\end{theorem}
\begin{proof}[Proof of \Theorem{OST}]
	Define the stopped process $Z_{t \wedge \tau}$,
	which is also a martingale \cite[Theorem 10.15]{Klenke}.
	Since $\expect{\sqrt{\tau}} < \infty$ we have
	$\prob{\tau<\infty}=1$.
	On the event $\set{\tau < \infty}$,
	$(Z_{t \wedge \tau})_{t \geq 0}$ has a well-defined limit,
	which is used as the almost sure definition of $Z_\tau$.
	
	We claim that $Z_{t \wedge \tau} \xrightarrow{L_1} Z_{\tau} \in L_1$
	from which the theorem concludes as follows.
	By optional stopping \cite[Lemma~10.10]{Klenke}, since $\tau \wedge t \leq t$,
	$\expect{Z_{t \wedge \tau}} = \expect{Z_0}$.
	This last equality holds for any fixed $t \geq 0$.
	Hence, $\expect{Z_{\tau}} = \lim_{t \to \infty} \expect{Z_{t \wedge \tau}} = \expect{Z_0}$.
	
	It remains to show that $Z_{\tau \wedge t} \xrightarrow{L_1} Z_{\tau} \in L_1$.
	The $L_1$ convergence is proven using the dominated convergence theorem~\cite[Corollary 6.26]{Klenke},
	which requires exhibiting a random variable that bounds $\abs{Z_{t \wedge \tau}}$
	for all $t$ and has finite expectation.
	For notational convenience, let $X_t = Z_{t \wedge \tau}$.
	Clearly $\abs{X_t} \leq X_t^* \leq X^*_\infty$,
	so it remains to show that $\expect{X^*_\infty} < \infty$.
	Using \Theorem{BDG} and that $Z$ has increments bounded by $K$,
	\[
	\expect{ X^*_\infty }
	~\leq~ C \expect{[X]_\infty^{1/2}}
	~=~ C \expect{\Big(\sum_{1 \leq s \leq \tau} (Z_s-Z_{s-1})^2\Big)^{1/2}}
	~\leq~ C K \expect{\tau^{1/2}}
	~<~ \infty.
	\]
	The dominated convergence theorem states that $Z_{t \wedge \tau} \overset{L_1}{\longrightarrow} Z_\tau \in L_1$, as required.
\end{proof}

\paragraph{Optimizing the stopping time.}
Since the martingale $Z_t$ defined above has bounded increments, 
\Theorem{OST} may be applied so long as $\smallexpect{\sqrt{\tau(c)}} < \infty$,
in which case the preceding discussion yields $\AnytimeValue(2) \geq c/2$.
We reiterate that the condition $\smallexpect{\sqrt{\tau(c)}} < \infty$ is a stronger assumption than $\tau(c)$ being almost surely finite.
So it remains to determine
\begin{equation}
	\EquationName{SupStoppingTime}
	\sup \smallsetst{ c \geq 0 }{ \smallexpect{\sqrt{\tau(c)}} < \infty },
\end{equation}
where $\tau(c)$ is the first time at which a standard random walk crosses the two-sided boundary $\pm c \sqrt{t}$.
We will use the following result,
in which $M$ is the confluent hypergeometric function defined in \Subsection{hyper}.

\begin{theorem}[{\cite[Theorem 2]{Breiman}, \cite[Theorem 5]{GreenwoodPerkins}}]
	\TheoremName{Breiman}
	Let $c>1$ and $a<0$ be such that $c$ is the smallest positive root of the function $x \mapsto M(a,1/2,x^2/2)$.
	Then
	$\prob{ \tau(c) > u } = u^a \pi(u)$, where $\pi$ is a slowly-varying function, i.e.~$\lim_{x \to \infty} \pi(ax) \pi(x)^{-1}$ for all $a > 0$.
\end{theorem}
\begin{fact}[{\cite[Lemma VIII.8.2]{Feller}}]
	\FactName{feller}
	Let $\pi$ be a slowly-varying function.
	Then for all $\eps > 0$ there exists $M_{\eps}$ such that $\pi(x) \leq x^{\eps}$ for all $x \geq M_{\eps}$.
\end{fact}
By combining \Theorem{Breiman} and \Fact{feller}, we see that if $c$ is the smallest positive root of the function
$x \mapsto M(a, 1/2, x^2/2)$ then for any $\delta > 0$, there exists a constant $C_{\delta}$ such that $\prob{\tau(c) > u} \leq C_{\delta} u^{a+\delta}$.

Recall the definition of $\gamma$ in \eqref{eq:GammaDef}.
For intuition, let us apply \Theorem{Breiman} with $c=\gamma$,
which is defined so that it is the root for $a=-1/2$ (see \Equation{MiDef} and \Fact{basic_identities}).
It then follows that (ignoring the slowly varying function for now),
\begin{align*}
	\expect{ \sqrt{\tau(\gamma)} }
	~=~ \int_0^{\infty} \prob{ \sqrt{\tau(\gamma)}>s } \dd s
	~=~ \int_0^{\infty} \prob{ \tau(\gamma)>s^2 } \dd s
	~\sim~ K \int_0^{\infty} s^{-1} \dd s,
\end{align*}
by \Theorem{Breiman}.
This integral is infinite, so \Theorem{OST} cannot be applied to $\tau(\gamma)$.
However, the integral is on the cusp of being finite. 
By slightly decreasing $a$ below $-1/2$, and slightly modifying $c$ to be the new root,
we should obtain a finite integral, showing that $\smallexpect{ \sqrt{\tau(c)} }$ is finite.
The following proof uses analytic properties of $M$ to show that this is possible.

\begin{proof}[Proof of \Theorem{LB}]
	Fix any $\eps > 0$ that is sufficiently small.
	Consider the random adversary and the stopping times $\tau(c)$
	described above.
	By \Claim{RootOfM}, there exists $a_\eps \in (-1,-1/2)$ and $c_\eps \geq \gamma-\eps$ such that
	$c_\eps$ is the unique positive root of $z \mapsto M(a_\eps,1/2,z^2/2)$.
	Let $\delta > 0$ be a constant such that $a_{\eps} + \delta < -1/2$.
	Then for some constant $C_{\delta}$,
	\begin{equation}
		\EquationName{FiniteIntegral}
		\expect{ \sqrt{\tau(c_\eps)} }
		~=~ \int_0^{\infty} \prob{ \tau(c_\eps)>s^2 } \dd s \\
		~\leq~ C_{\delta} \int_0^{\infty} s^{2 (a_\eps + \delta)} \dd s
		~<~ \infty,
	\end{equation}
	since $a_\eps + \delta<-1/2$.
	It follows that $\tau(c_\eps)$ is almost surely finite, and therefore 
	$\Regret{\tau(c_\eps)}$ and $g_{\tau(c_\eps)}$ are almost surely well defined.
	Applying \Theorem{OST} to the martingale $Z_t$ appearing in \Equation{LBRegret},
	we obtain that
	\begin{align*}
		\expect{\Regret{\tau(c_\eps)}}
		& ~\geq~ \expect{Z_{\tau(c_\eps)} + \frac{g_{\tau(c_\eps)}}{2}} \\
		& ~=~ \expect{Z_{0}} + \frac{1}{2} \expect{g_{\tau(c_\eps)}} \\
		& ~=~ \frac{1}{2} \expect{g_{\tau(c_\eps)}} \\
		& ~=~ \frac{1}{2} \expect{c_\eps \sqrt{\tau(c_\eps)}},
	\end{align*}
	where the second equality is because $Z_0 = 0$ deterministically.
	By the probabilistic method, there exists a finite sequence of cost vectors
	$\ell_1,\ldots,\ell_t$ (depending on $\cA$ and $\eps$)
	for which the regret of $\cA$ at time $t$ is at least $c_\eps \sqrt{t} /2$.
	The adversary $\cB_\eps$ (which knows $\cA$) provides this sequence of cost vectors to algorithm $\cA$, thereby proving \eqref{eq:GoalInLB}.
\end{proof}

\subsection{Large regret infinitely often}
\SubsectionName{regretIO}
In this subsection, we prove the following extension of \Theorem{LB}, which shows that one can achieve regret $\gamma \sqrt{t} / 2$ infinitely often.
\begin{theorem}
	\TheoremName{regretIO}
	For any algorithm $\cA$ and any $\eps > 0$, there exists an adversary
	$\cB_{\eps}$ such that
	\begin{equation}
		\EquationName{GoalInLBIO}
		\limsup_{t \geq 1} \frac{\Regret{2,t,\cA,\cB_\epsilon}}{\sqrt{t}} ~\geq~ \frac{\gamma-\eps}{2}.    
	\end{equation}
\end{theorem}
The basic idea of the proof of \Theorem{regretIO} is quite simple.
Initially, we run a reflected random walk starting at the origin and wait until it crosses the $(\gamma-\eps) \sqrt{t}$ boundary.
By the arguments in \Theorem{LB}, we know that, in expectation, the regret is large at the first instant when the random
walk crosses the boundary.
We then ``restart'' the random walk except now the starting position is the current position of the random walk instead of the origin.
The key observation is that \Theorem{Breiman} is only sensitive to the asymptotics of the boundary and not the starting position.
Thus, essentially the same arguments in \Theorem{LB} can be used to show that (i) the random walk crosses the $(\gamma - \eps) \sqrt{t}$
boundary a \emph{second} time and (ii) the regret is large at the time when the random walk crosses the boundary for the second time.

To formally prove \Theorem{regretIO}, we need a more general version of \Theorem{OST}.
Let $\cF$ be a $\sigma$-algebra and let $(\cF_t)_{t \in \bZ_{\geq 0}}$ be a filtration (i.e.~$\cF_0 \subseteq \cF_1 \subseteq \ldots$
and $\cF_t \subseteq \cF$ for all $t \geq 0$).
For a stopping time $\tau$, the stopped $\sigma$-algebra is defined as
$\cF_\tau \coloneqq \{A \in \cF \,:\, A \cap \{\tau \leq t\} \in \cF_t\, \forall t \in \bZ_{\geq 0} \}$
\cite[Definition 9.19]{Klenke}.
Finally, let $\cG \subseteq \cF$ be a sub $\sigma$-algebra.
For a random variable $X$, the conditional expectation of $X$ given $\cG$, denoted $\expectg{X}{\cG}$,
is a random variable $Y$ satisfying $\expect{ Y\ind_A } = \expect{ X \ind_A }$ for all $A \in \cG$ \cite[Definition 8.11]{Klenke}.
Here, $\ind_A$ is the indicator of the event $A$.

\begin{theorem}
	\TheoremName{OST2}
	Let $( Z_t )_{t \in \bZ_{\geq 0}}$ be a martingale with respect to a filtration $\{\cF_{t}\}$
	and $K>0$ a constant such that $|Z_t - Z_{t-1}| \leq K$ almost surely for all $t$.
	Let $\sigma \leq \tau$ be stopping times and suppose that $\expect{\sqrt{\tau}} < \infty$.
	Then the random variables $Z_{\sigma}, Z_{\tau}$ are almost surely well-defined
	and $\expectg{Z_{\tau}}{\cF_{\sigma}} = Z_{\sigma}$.
\end{theorem}
\begin{proof}
	Define the stopped process $Z_{t \wedge \tau}$,
	which is also a martingale \cite[Theorem 10.15]{Klenke}.
	Since $\expect{\sqrt{\tau}} < \infty$ we have
	$\prob{\tau<\infty}=1$.
	On the event $\set{\tau < \infty}$,
	$(Z_{t \wedge \tau})_{t \geq 0}$ has a well-defined limit,
	which is used as the almost sure definition of $Z_\tau$.
	As $\set{\tau < \infty} \subseteq \set{\sigma < \infty}$,
	the same argument shows that $(Z_{t \wedge \sigma})_{\geq 0}$
	has a well-defined limit, and we use this as the almost sure
	definition of $Z_{\sigma}$.
	
	The arguments in the proof of \Theorem{OST} show that $Z_{t \wedge \tau} \xrightarrow{L_1} Z_{\tau} \in L_1$
	and $Z_{t \wedge \sigma} \xrightarrow{L_1} Z_{\sigma} \in L_1$.
	By the definition of conditional expectation, we need to check that
	$\expect{Z_{\tau} \ind_{A}} = \expect{Z_{\sigma} \ind_{A}}$
	for all $A \in \cF_{\sigma}$.
	To that end, fix $A \in \cF_{\sigma}$ and note that
	$A \cap \set{\sigma \leq t} \in \cF_{\sigma \wedge t}$.
	For any fixed $t$, $t \wedge \sigma \leq t$ and $\tau \leq t$,
	so the optional sampling theorem \cite[Theorem 10.11]{Klenke}
	applied to the stopped process yields
	$\expectg{Z_{t \wedge \tau}}{\cF_{t \wedge \sigma}} = Z_{t \wedge \sigma}$.
	Hence,
	\begin{equation}
		\EquationName{OSTProof1}
		\expect{ Z_{\tau \wedge t} \ind_{A} \ind_{\set{\sigma \leq t}} } = \expect{Z_{\sigma \wedge t} \ind_{A} \ind_{\set{\sigma \leq t}}}.
	\end{equation}
	Since $Z_{\tau \wedge t} \xrightarrow{L_1} Z_{\tau} \in L_1$, it follows that
	$Z_{\tau \wedge t} \ind_{A} \ind_{\set{\sigma \leq t}} \xrightarrow{L_1} Z_{\tau} \ind_{A} \ind_{\set{\sigma < \infty}}$.
	This is because
	\begin{align*}
		\operatorname{E} \big[ |Z_{\tau \wedge t} \ind_{A} \ind_{\sigma \leq t} - Z_{\tau} \ind_{A} \ind_{\sigma < \infty} | \big]
		& \leq
		\expect{\abs{Z_{\tau \wedge t} \ind_{A} \ind_{\sigma \leq t} - Z_{\tau} \ind_{A} \ind_{\sigma \leq t}}}
		+ \expect{\abs{Z_{\tau} \ind_{A} \ind_{\sigma < \infty} - Z_{\tau} \ind_{A} \ind_{\sigma \leq t}}} \\
		& \leq \expect{\abs{Z_{t\wedge \tau}-Z_{\tau}}} + \expect{\abs{Z_{\tau}} \ind_{ t < \sigma < \infty}}.
	\end{align*}
	The quantity $\expect{\abs{Z_{t \wedge \tau} - Z_{\tau}}} \to 0$
	because $Z_{t \wedge \tau} \xrightarrow{L_1} Z_{\tau}$.
	Next, $Z_{\tau} \in L_1$ and $\ind_{t < \sigma < \infty} \to 0$
	a.s.~so $\expect{\abs{Z_{\tau}} \ind_{t < \sigma < \infty}} \to 0$ by dominated convergence.
	Finally, note that $Z_{\tau} \ind_{A}\ind_{\sigma < \infty} = Z_{\tau} \ind_{A}$ as $\ind_{\sigma < \infty} = 1$ a.s.
	Hence,
	\begin{equation}
		\EquationName{OSTProof2}
		\expect{Z_{\tau \wedge t} \ind_{A} \ind_{\set{\sigma \leq t}}} \xrightarrow{t \to \infty} \expect{Z_{\tau}\ind_{A}}.
	\end{equation}
	Similarly,
	\begin{equation}
		\EquationName{OSTProof3}
		\expect{Z_{\sigma \wedge t} \ind_A \ind_{\set{\sigma \leq t}}} \xrightarrow{t \to \infty} \expect{Z_{\sigma}\ind_A}.
	\end{equation}
	Combining \Equation{OSTProof1}, \Equation{OSTProof2}, and \Equation{OSTProof3}
	gives $\expect{Z_{\tau} \ind_A} = \expect{Z_{\sigma} \ind_A}$ as desired.
\end{proof}

The proof of \Theorem{regretIO} makes use of the following result which is a generalization of \Theorem{Breiman}
to the setting where the boundary is \emph{asymptotically} a square root curve.
This will allow us to consider a random walk hitting a square root boundary but where both the boundary
and the starting position of the particle may not be at the origin.

\begin{theorem}[{\cite[Theorem 5]{GreenwoodPerkins}}]
	\TheoremName{GreenwoodPerkins}
	Let $c>1$ and $a<0$ be such that $c$ is the smallest positive root of the function $x \mapsto M(a,1/2,x^2/2)$.
	Let $f(t)$ be a function such that $\lim_{t \to \infty} f(t) t^{-1/2} = c$.
	Let $\tau = \inf\setst{t > 0}{g_t \geq f(t)}$.
	Then
	$\prob{ \tau > u } = u^a \pi(u)$, where $\pi$ is a slowly-varying function.
\end{theorem}

\begin{proof}[Proof of \Theorem{regretIO}.]
	We use the same adversary as in \Theorem{LB} so that
	\[
	\Reg(t) \geq Z_t + \frac{g_t}{2},
	\]
	where $Z_t$ is a martingale with $Z_0 = 0$ and $g_t$ evolves as a reflected random walk.
	Let $\cF_t \coloneqq \sigma(g_0, \ldots, g_t)$ be the natural filtration.
	Finally, let $c_{\eps} \geq \gamma - \eps$ and $a_{\eps}$ be as in the proof of \Theorem{LB}.
	
	Define the stopping times $\tau_0 \coloneqq 0$ and
	$\tau_i \coloneqq \inf \setst{t > \tau_{i-1}}{g_t \geq c_{\eps} \sqrt{t}}$
	for $i \geq 1$.
	Note that, by the strong Markov property, for each $i \geq 1$,
	the process $\{g_{\tau_{i-1} + t}\}_{t \geq 0}$ is a reflected random walk
	started at position $g_{\tau_{i-1}} > 0$.
	Moreover, observe that $\tau_i$ is similar to the stopping time used in \Theorem{LB}
	in that the asymptotics of the boundary are the same but
	the boundary itself and starting point may be perturbed by a (random) additive constant.
	
	Let us assume that $\expect{\sqrt{\tau_{i-1}}} < \infty$ and we now show that $\expect{\sqrt{\tau_i}} < \infty$.
	Let $\delta > 0$ be a constant such that $a_{\eps} + \delta < -1/2$.
	On the event that $\{\tau_{i-1} < \infty\}$,
	\Theorem{GreenwoodPerkins} and \Fact{feller} imply that there is a (random) constant $C_{\delta}$,
	which may depend on $\tau_{i-1}$ and $g_{\tau_{i-1}}$ (which are both $\cF_{\tau_{i-1}}$-measurable), such that
	$\probg{\tau_i - \tau_{i-1} > u}{\cF_{\tau_{i-1}}} \leq C_{\delta} u^{a_{\eps} + \delta}$.
	Hence, following the proof of \Theorem{LB},
	this implies that $\expectg{\sqrt{\tau_i - \tau_{i-1}}}{\cF_{\tau_{i-1}}} < \infty$.
	Since $\expect{\sqrt{\tau_{i-1}}} < \infty$, this implies that $\expect{\sqrt{\tau_i}} < \infty$.
	Hence, we can apply \Theorem{OST2} to obtain that
	$\expectg{Z_{\tau_i}}{\cF_{\tau_{i-1}}} = Z_{\tau_{i-1}}$ for all $i \geq 1$.
	
	We will now inductively construct a sequence of events which satisfy the
	conclusions of the theorem.
	To that end, define the events
	\[
	A_i = \set{\tau_i < \infty, Z_{\tau_i} \geq \ldots \geq Z_{\tau_1} \geq 0}.
	\]
	For the base case, we have $A_1 = \set{\tau_1 < \infty, Z_{\tau_1} \geq 0}$.
	In the proof of \Theorem{LB}, we have already verified that $\prob{A_1} > 0$
	(this also follows from the previous paragraph).
	For the inductive step, suppose that $\prob{A_{i-1}} > 0$.
	The condition that
	$\expectg{Z_{\tau_i}}{\cF_{\tau_{i-1}}} = Z_{\tau_{i-1}}$
	implies that, for any $B \in \cF_{\tau_{i-1}}$ with $\prob{B} > 0$,
	the event $B \cap \set{\tau_i < \infty, Z_{\tau_i} \geq Z_{\tau_{i-1}}}$ has positive probability.
	Taking $B = A_{i-1}$ implies that $\prob{A_i} > 0$.
	
	To conclude, for any $n \geq 1$, the event $A_n$ has positive probability.
	Hence, there exists a sequence of times $T_1, \ldots, T_n < \infty$
	and loss vectors up to time $T_n$ that guarantee
	$g_{T_i} \geq c_{\eps} \sqrt{T_i}$ for all $i \in [n]$ and
	$Z_{T_n} \geq \ldots \geq Z_{T_1} \geq 0$.
	In particular, for all $i \in [n]$,
	\[
	\Reg(T_i) \geq Z_{T_i} + \frac{g_{T_i}}{2} \geq \frac{c_{\eps}}{2} \sqrt{T_i}.
	\]
	As $n \geq 1$ was arbitrary, the theorem follows.
\end{proof}

%% file: cts_ub.tex
\section{Derivation of a continuous-time analogue of {\protect \Algorithm{minimax}}}
\SectionName{cts_ub}

The purpose of this section is to show how the potential function
$R$ defined in \Equation{RDef} arises naturally as the solution
of a stochastic calculus problem.
The derivation of that function is accomplished by defining, then solving,
an analogue of the regret minimization problem in continuous time.
The main advantage of considering this continuous setting is the wealth
of analytic methods available, such as stochastic calculus.

\subsection{Defining the continuous regret problem}
\paragraph{Continuous time regret problem.}
The continuous regret problem is inspired by \Equation{RegretInGapSpace2}.
Notice that, when the adversary chooses {cost vectors in $\{ \begin{smallbmatrix}1\\0\end{smallbmatrix},\begin{smallbmatrix}0\\1\end{smallbmatrix} \}$,
the sequence of gaps $g_0, g_1, g_2, \ldots$ live in the support of a reflected random walk.
The goal in the discrete case is to find an algorithm $p$ that bounds
the regret over all possible sample paths of a reflected random walk.
In continuous time it is natural to consider a stochastic integral with
respect to reflected Brownian motion, denoted $|B_t|$, instead.
Our goal now is to find a continuous-time algorithm whose regret
is small for almost all reflected Brownian motion paths.

\begin{definition}[Continuous Regret]
\DefinitionName{ContRegret}
Let $p : \bR_{>0} \times \bR_{\geq 0} \rightarrow [0,1]$ be a continuous function that satisfies $p(t,0) = 1/2$ for every $t > 0$.
Let $B_t$ be a standard one-dimensional Brownian motion.
Then, the \emph{continuous regret} of $p$ with respect to $B$
is the stochastic integral
\begin{equation}
\EquationName{ContinuousRegret}    
\ContReg{T,p,B} ~=~ \int_{0}^T p(t, \Abs{B_t}) \dd \Abs{B_t}.
\end{equation}  
\end{definition}

\begin{remark}
The condition $p(t,0) = 1/2$ is due to \Equation{ContinuousRegret} being inspired by \Equation{RegretInGapSpace2}, which requires this condition.
\end{remark}

In this definition we may think of $p$ as a continuous-time algorithm and $B$ as a continuous-time adversary.
The goal for the remainder of this section is to prove the following result.

\begin{theorem}
\TheoremName{ContsMainResult}
There exists a continuous-time algorithm $p^*$ such that
\begin{equation}
    \EquationName{ContsMainResult}
    \ContReg{T,p^*,B} ~\leq~ \frac{\gamma \sqrt{T}}{2}~
    \quad\forall T \in \bR_{\geq 0},
    ~\text{almost surely}.
 \end{equation}
\end{theorem}

\begin{remark}
A natural question arises upon reviewing the definition of continuous regret: What role does Brownian motion play in \Definition{ContRegret} and is it the ``correct'' stochastic process to consider in order to uncover the optimal algorithm?
In the analysis that follows, the only properties of reflected Brownian motion that we use are its non-negativity and that its \emph{quadratic variation} is $t$.
It turns out that one can generalize \Theorem{ContsMainResult} by allowing any non-negative, continuous semi-martingale $X$ to control the gap process, and by letting time grow at the rate of the \emph{quadratic variation} of $X$. See 
\Appendix{cont_reg_semi_mart} for more details. 
\end{remark}

\subsection{Connections to stochastic calculus and the backward heat equation}
Since $\ContReg{T}$ evolves as a stochastic integral with respect to a semi-martingale\footnote{A semi-martingale is a stochastic process that can written as the sum of a local martingale and a process of finite variation.} (namely reflected Brownian motion), It\^o's lemma provides an insightful decomposition.
The following statement of It\^o's lemma is a specialization of
\cite[Theorem~IV.3.3]{RY13}
for the special case of reflected Brownian motion.\footnote{
Specifically, we are using the statement of It\^{o}'s formula that appears in Remark 1
after Theorem~IV.3.3 in \cite{RY13} with $X_t = |B_t|$ and $A_t = t$.
Note that $y$ in their notation is $t$ in ours and $\inner{|B|}{|B|}_t = t$.
}

\paragraph{Notation.}
Up to now, we have used the symbol $g$ as the second parameter to the bivariate functions $p$ and $R$.
Henceforth, it will be more consistent with the usual notation in the literature to use $x$ to denote $g$.
We will also use the notation $C^{1,2}$ to denote the class of bivariate functions that are continuously differentiable in their first argument and twice continuously differentiable in their second argument.

\begin{theorem}[It\^o's formula]
\TheoremName{ItoFormula}
Let $f \colon \bR_{\geq 0} \times \bR \rightarrow \bR$ be $C^{1,2}$.
Then, almost surely,
\begin{align}\EquationName{ItoFormula}
    f(T, \Abs{B_T}) - f(0, \Abs{B_0}) &~=~ \int_{0}^T \partial_x f(t, \Abs{B_t}) \dd \Abs{B_t} + \int_{0}^T \Big [
    \underbrace{\partial_t f(t, \Abs{B_t}) + \smallfrac{1}{2}   \partial_{xx}f(t, \Abs{B_t})}_{\eqqcolon \doob f(t,\abs{B_t})} \Big ] \dd t.
\end{align}
\end{theorem}
The integrand of the second integral is an important quantity arising in PDEs and stochastic processes
(see, e.g., \cite[pp.~263]{DoobBook}).
We will denote it by 
$\doob f(t,x) \coloneqq \partial_t f(t, x) + \frac{1}{2} \partial_{xx} f(t,x)$.
Some discussion about the statement of \Theorem{ItoFormula} appears in \Appendix{ito_discussion}.

\paragraph{Applying It\^o's formula to the continuous regret.}
Comparing these equations, it is natural to assume that $p = \partial_x f$ for a function
$f$ that is $C^{1,2}$ with $f(0,0) = 0$, $\partial_x f \in [0,1]$, and $\partial_x f (t,0) = 1/2$;
the latter two conditions are needed for 
\Definition{ContRegret} to be applicable. 
It\^o's formula then yields
\begin{align}
    \EquationName{cont_reg_ito_mix}
    \ContReg{T, p = \partial_x f, B} 
        ~=~ \int_{0}^{T} \partial_xf(t, \Abs{B_t}) \dd \Abs{B_t} 
        ~=~ f(T,\Abs{B_T}) - \int_{0}^T \doob f(t,\abs{B_t}) \dd t.
\end{align}

\paragraph{Path independence and the backward heat equation.}
At this point a useful idea arises:
as a thought experiment, suppose that $\doob f = 0$.
Then the second integral would vanish, and we would have the appealing expression
$\ContReg{T, p, B} = f(T,\abs{B_T})$.
Moreover, since $f$ is a deterministic function, the right-hand side depends only on $\abs{B_T}$ rather than the entire Brownian path $B|_{[0,T]}$.
Thus, the same must be true of the left-hand side:
at time $T$, the continuous regret of the algorithm $p$ depends only on $T$ and $\abs{B_T}$ (the gap).
We say that say that such an algorithm has
\emph{path independent regret}.
Our supposition that led to these attractive consequences was only that $\doob f = 0$, which turns out to be a well studied condition.

\begin{definition}
\DefinitionName{BHE}
Let $f \colon \bR_{> 0} \times \bR \rightarrow \bR$ be a $C^{1,2}$ function.
If $\doob f(t,x) = 0$ for all $(t,x) \in \bR_{> 0} \times \bR$ then we say that
$f$ satisfies the \emph{backward heat equation}.
A synonymous statement is that $f$ is \emph{space-time harmonic}.
\end{definition}

We may summarize the preceding discussion with the following proposition. 

\begin{proposition}
\PropositionName{FindingF}
Let $f : \bR_{> 0} \times \bR  \rightarrow \bR$ be a $C^{1,2}$ function that satisfies $\doob f = 0$ everywhere with $f(0,0)=0$.
Let $p = \partial_x f$.
Then, 
\begin{equation}
    \EquationName{FindingFConsequence}
  \int_0^T p(t,\Abs{B_t}) \dd \Abs{B_t} ~=~ f(T, \Abs{B_T}). \end{equation}
\end{proposition}

Suppose that a function $f$ satisfies the hypothesis of \Proposition{FindingF} \emph{and in addition} $p = \partial_x f \in [0,1]$ with $p(t,0) = 1/2$.
Then, we would have 
\begin{equation}
    \EquationName{ContRegIntermediate}
     \ContReg{T, p , B} = f(T, \Abs{B_T}). 
\end{equation}
We are unable to derive a function that satisfies the properties required for \Equation{ContRegIntermediate} to hold along with $\max_{x\geq 0} f(T,x) \leq \gamma \sqrt{T} / 2$.
Instead, we will begin by relaxing the constraint that $p(t,x) \in [0,1]$ and allow $p(t,x)$ to be negative.
We will overload the notation $\ContReg{\cdot}$ to include such functions. 
In the next section, we will derive a family of such functions that all achieve $\PseudoContReg{T, p, \Abs{B_T}} = f(T, \Abs{B_T}) = O(\sqrt{T})$.
This is done by setting up and solving the backwards heat equation.
Next, we use a ``smoothing'' argument to obtain a family of functions
that all achieve $\ContReg{T, p, \Abs{B_T}} = O(\sqrt{T})$,
and that \emph{do} satisfy $p(t,x) \in [0,1]$.
Finally, we will optimize $\ContReg{T,\cdot, \Abs{B_T}}$ over this family of functions to prove \Theorem{ContsMainResult}. The constant $\gamma$ will appear as a consequnce of this optimization problem.

\subsubsection{Satisfying the backward heat equation}
The main result of this section is the derivation of a family of functions $\tilde{p} : \bR_{>0} \times \bR \rightarrow \bR$
that satisfy $\tilde{p}(t,x) \leq 1$, $\tilde{p}(t,0) =1/2$ and
\begin{equation}
\EquationName{PRegretResult}
    \PseudoContReg{T, \tilde{p}, B} ~=~ f(T, \Abs{B_T}) ~=~ O( \sqrt{T} ),
\end{equation}
but do not necessarily satisfy $\tilde{p}(t,x) \geq 0$.

The first step is to find a function $f$ which satisfies the partial differential equation $\doob f = 0.$
Since the boundary condition $\tilde{p}(t,0)=1/2$ is a condition on $\tilde{p} = \partial_x f$, not on $f$ itself, it will be convenient to solve a PDE for $\tilde{p}$ instead, and then to derive $f$ by integrating.
However, some care is needed since not all antiderivates of $\tilde{p}$ (in $x$) will satisfy the backwards heat equation.
Fortunately, we have a useful lemma showing that if $\tilde{p}$ satisfies the backward heat equation, then we can construct an $f$ that also does.
\begin{lemma}
\LemmaName{BHEDerivative}
Suppose that $h : \bR_{> 0} \times \bR \rightarrow \bR$ is a $C^{1,2}$ function. 
Define $$f(t,x) := \int_{0}^x h(t,y)
\dd y - \frac{1}{2}\int_{0}^t \partial_x h(s,0) \dd s.$$
Then,
\begin{enumerate}[label=(\arabic*),noitemsep,topsep=0pt]
\item $f \in C^{1,2}$,
\item If $\doob h = 0$ over $\bR_{>0} \times \bR$ then
$\doob f = 0$ over $\bR_{>0} \times \bR$,
\item $h = \partial_x f$. 
\end{enumerate}
\end{lemma}
\begin{proof}{\Lemma{BHEDerivative}}
	First, we check that $f \in C^{1,2}.$ Let $(t,x) \in \bR_{>0} \times \bR$. It is easy to check via standard applications of the Dominated Convergence Theorem (DCT) and the Fundamental Theorem of Calculus (FTC) that
	\begin{enumerate}[label=(\arabic*),noitemsep,topsep=0pt]
		\item $\partial_t f (t,x) = \int_0^x \partial_t h(t,y) \dd y - \frac{1}{2} \partial_x h(t,0),$
		\item $\partial_x f(t,x) = h(t,x),$ and
		\item $\partial_{x x} f(t,x) = \partial_x h(t,x).$
	\end{enumerate}All of the above partial derivatives are clearly continuous since $h$ is $C^{1,2}$.
	
	Next, we show that if $\doob h(t,x) = 0 $ for all $(t,x) \in \bR_{>0}\times  \bR$, then $\doob f(t,x) = 0$ for all $\bR_{>0} \times \bR$. Indeed,
	\begin{align*}
		\doob f(t,x)
		&~=~ \left( \partial_t + \frac{1}{2} \partial_{xx} \right) f(t,x) \\
		&~=~ \int_0^x \partial_t h(t,y) \, \dd y - \frac{1}{2} \partial_x h(t,0) + \frac{1}{2} \partial_x h(t,x) &&\qquad\text{(by (1) and (3))} \\
		&~=~ \int_0^x \underbrace{\left ( \partial_t h(t,y) + \frac{1}{2}\partial_{xx} h(t,y) \right )}_{=0} \, \dd y &&\qquad\text{(by FTC)} \\
		&~=~ 0,
	\end{align*}
	as claimed.
\end{proof}

\paragraph{Defining boundary conditions for $p$.}
Obtaining a particular solution to the backward heat equation requires sufficient boundary conditions in order to uniquely identify $\tilde{p}$.
The boundary condition mentioned above is that $\tilde{p}(t,0) = 1/2$ for all $t$.
This condition together with the backward heat equation clearly do not suffice to uniquely determine $\tilde{p}$. Therefore, we impose some reasonable boundary conditions on $\tilde{p}$.

What should the value be at the boundary?
Intuitively, $x \mapsto \tilde{p}(t,x)$ should be a decreasing function because $\tilde{p}$ represents the weight placed on the worst expert as a function of the gap.
Therefore, it is natural to consider an ``upper boundary'' which specifies the point at which the difference in experts' total costs is so great that the algorithm places zero weight on the worst expert. 
The upper boundary can be specified by a curve, $\setst{(t,\phi(t))}{t > 0}$ for some continuous function $\phi : \bR_{>0} \rightarrow \bR_{> 0}.$
We will incorporate this idea by requiring $\tilde{p}(t, \phi(t)) = 0$ for all $t > 0$.

Where should the boundary be?
One reasonable choice for the boundary is to use $\phi_{\alpha}(t) = \alpha \sqrt{t}$ for some constant $\alpha>0$, as this is similar to the boundary used by the random adversary in the lower bound of \Section{lb}.
For now, we leave $\alpha$ as an unknown parameter whose value can be optimized later.
These conditions are combined into the following partial differential equation:
\begin{alignat}{2}
\EquationName{BHEU}
    \text{\small(backward heat equation)}\qquad
    &\partial_t u(t,x) + \smallfrac{1}{2}\partial_{xx} u(t,x) ~=~0 &&\qquad\text{for all $(t,x) \in \bR_{>0 }\times \bR$}\\
    \text{\small(upper boundary)}\qquad
    \EquationName{boundary_constraint}
    & u(t, \alpha\sqrt{t}) ~=~ 0
    &&\qquad\text{for all $t > 0$}\\
    \EquationName{x_axis_constraint}
    \text{\small(lower boundary)}\qquad
    & u(t,0) ~=~ \smallfrac{1}{2}
    &&\qquad\text{for all $t > 0$}.
\end{alignat}
Next we show that the following function solves this PDE.
Define $\tilde{p}_{\alpha} : \bR_{> 0} \times \bR \rightarrow \bR$ by 
\begin{equation}
\EquationName{palpha_def}
\tilde{p}_{\alpha}(t,x)
    ~\coloneqq~
    \frac{1}{2}\left ( 1 - \frac{\erfi \left ( \sfrac{x}{\sqrt{2t}}\right )}{\erfi\left( \sfrac{\alpha}{\sqrt{2}}\right)} \right ).
\end{equation}

\begin{lemma}
\LemmaName{p_pde_solution}
$\tilde{p}_{\alpha}$ satisfies the following properties:
\begin{enumerate}[label=(\arabic*),noitemsep,topsep=0pt]
\item $\tilde{p}_{\alpha}$ is $C^{1,2}$ over $\bR_{>0} \times \bR$,
\item $\tilde{p}_{\alpha}$ satisfies the constraints in \Equation{BHEU}, \Equation{boundary_constraint} and \Equation{x_axis_constraint}, and
\item For all $t>0$ and all $x \geq 0$, $\tilde{p}_{\alpha}(t,x) \leq 1/2$.
\end{enumerate}
\end{lemma}
\begin{proof}{\Lemma{p_pde_solution}}
	Let us assume that we can write $u(t,x) = v(x/\sqrt{t})$.
	Then, we have $\partial_t u(t,x)  = - \frac{x}{2 t^{3/2}} v'(x/\sqrt{t}),$ and $\frac{1}{2}\partial_{xx} u(t,x) = \frac{1}{2t}v^{''}(x/\sqrt{t}).$ The backward heat equation enforces that $v''(x/\sqrt{t}) = \frac{x}{\sqrt{t}}v'(x/\sqrt{t})$.
	By a change of variables $(z = x/\sqrt{t})$, we obtain the following ordinary differential equation
	\begin{equation}
		\EquationName{BHEODE}
		v''(z) ~=~ z\cdot v'(z).
	\end{equation}
	Hence, $v'(z) = C\cdot e^{\frac{z^2}{2}}$ for some constant $C$. We can then integrate to obtain $v(z) = \int_{0}^{z}C e^{y^2/2} \,\dd y  + D = \int_0^{z/\sqrt{2}} \sqrt{2} C e^{r^2}\, \dd r + D$, for some constant $D$.
	For the last equality, we made the change of variables $r = y/\sqrt{2}$ in the integral.
	Therefore, by the definition of $\erfi$ (and a different constant $C$), we have $v(z) = C \erfi(z/\sqrt{2}) + D$.
	Hence, for some constants $C, D \in \bR$, we have
	\[
	u(t,x) ~=~ C \erfi(x/\sqrt{2t}) + D.
	\]
	Plugging in the boundary condition at $x = 0$ and recalling that $\erfi(0) = 0$ we see that $D = 1/2.$ Plugging in the boundary condition that $u(t,\alpha\sqrt{t}) = 0$ and using that $D=1/2$ we see that $C = -\frac{1}{2\erfi\left( \alpha / \sqrt{2} \right )}.$ Therefore, we have that the following function
	\[
	u(t,x) ~=~ \frac{1}{2}\left( 1 - \frac{\erfi \left( x/\sqrt{2t} \right )}{\erfi \left (\alpha/\sqrt{2} \right )} \right )
	\]
	satisfies the backwards heat equation and the boundary conditions.
	Moreover, $u \in C^{1,2}$ on $\bR_{> 0} \times \bR$.
\end{proof}
\Lemma{p_pde_solution} shows that $\tilde{p}_{\alpha}(t,x)$ nearly defines a valid continuous time algorithm,
in that it satisfies the conditions of \Definition{ContRegret} except for non-negativity.
Next, we will integrate $\tilde{p}_{\alpha}$ as described in \Lemma{BHEDerivative}.
Define the function $\tilde{R}_{\alpha} \colon \bR_{> 0} \times \bR \to \bR$ as
\begin{align}
    \EquationName{RtildeDef}
    \tilde{R}_{\alpha}(t,x) = 
    \frac{x}{2} + \kappa_{\alpha}\sqrt{t} \cdot M_0\left(\frac{x^2}{2t}\right) 
    \qquad\text{where}\qquad
    \kappa_{\alpha} = \frac{1}{\sqrt{2 \pi} \erfi(  \alpha/\sqrt{2})}.
\end{align}
\begin{lemma}
    \LemmaName{RIntegral}
    $\tilde{R}_{\alpha}(t,x) = \int_0^x \tilde{p}_{\alpha}(t,y)\, \dd y - \frac{1}{2} \int_0^t \partial_x \tilde{p}_{\alpha}(s,0) \, \dd s$.
\end{lemma}
First we need to compute some derivatives.
\begin{lemma}
	\LemmaName{uncap_rtg_deriv_cts}
	The following identities hold for every $\alpha > 0$.
	\begin{enumerate}
		\item $\partial_x \tilde{R}_{\alpha}(t,x) = \tilde{p}_{\alpha}(t,x) = \frac{1}{2} \left( 1 - \frac{\erfi(x/\sqrt{2t})}{\erfi(\alpha/\sqrt{2t})}\right)$.
		\item $\partial_{xx} \tilde{R}_{\alpha}(t,x) = \partial_x \tilde{p}_{\alpha}(t,x) = -\kappa_{\alpha} \cdot \frac{\exp(x^2/2t)}{\sqrt{t}}$.
	\end{enumerate}
\end{lemma}
\begin{proof}
	The proof is a straightforward calculation.
	We have
	\begin{align*}
		\partial_x \tilde{R}_{\alpha}(t,x)
		& = \frac{1}{2} - \kappa_{\alpha} \frac{x}{\sqrt{t}} \cdot M_1\left( \frac{x^2}{2t} \right) \\
		& = \frac{1}{2} - \frac{1}{\sqrt{2\pi} \erfi(\alpha/\sqrt{2})} \cdot \frac{x}{\sqrt{t}} \cdot \frac{\sqrt{\pi} \erfi(x/\sqrt{2t})}{2 \cdot x/\sqrt{2t}} \\
		& = \frac{1}{2} \left( 1 - \frac{\erfi(x/\sqrt{2t})}{\erfi(\alpha/\sqrt{2})} \right),
	\end{align*}
	where the first equality uses \Fact{chf_derivs}
	and the second equality uses the identity \ref{item:M1} in \Fact{basic_identities}.
	This proves the first identity.
	
	For the second identity, using the definition of $\erfi(\cdot)$, we have
	\[
	\partial_{xx} \tilde{R}_{\alpha} = \partial_x \tilde{p}_{\alpha}(t,x)
	= -\frac{\exp(x^2/2t)}{\sqrt{2\pi} \erfi(\alpha/\sqrt{2}) \sqrt{t}}
	= -\kappa_{\alpha} \cdot \frac{\exp(x^2/2t)}{\sqrt{t}}. \qedhere
	\]
\end{proof}

\begin{proofof}{\Lemma{RIntegral}}
	By the first identity in \Lemma{uncap_rtg_deriv_cts}, we have
	\begin{equation}
		\EquationName{RIntegral1}
		\int_0^{x} \tilde{p}_{\alpha}(t,y) \, \dd y = \tilde{R}_{\alpha}(t,x) - \tilde{R}_{\alpha}(t,0)
	\end{equation}
	Note that $\tilde{R}_{\alpha}(t,0) = \kappa_{\alpha} \sqrt{t}$.
	Next, the second identity of \Lemma{uncap_rtg_deriv_cts}
	implies that $-\partial_x \tilde{p}_{\alpha}(s,0) = \frac{\kappa_{\alpha}}{\sqrt{s}}$.
	Hence,
	\begin{equation}
		\EquationName{RIntegral2}
		-\frac{1}{2}\int_0^t \partial_x \tilde{p}_{\alpha}(s,0) \, \dd s = \kappa_{\alpha}\sqrt{t} = \tilde{R}_{\alpha}(t,0).
	\end{equation}
	Summing \Equation{RIntegral1} and \Equation{RIntegral2} gives
	\[
	\int_0^x \tilde{p}_{\alpha}(t,y) \, \dd y - \frac{1}{2}\int_0^t \partial_x \tilde{p}_{\alpha}(s,0) \, \dd s = \tilde{R}_{\alpha}(t,x) - \tilde{R}_{\alpha}(t,0) + \tilde{R}_{\alpha}(t,0) = \tilde{R}_{\alpha}(t,x). \qedhere
	\]
\end{proofof}

By \Lemma{p_pde_solution}, the function $\tilde{p}_{\alpha}$ satisfies the hypothesis of the function $h$ in \Lemma{BHEDerivative}.
Hence, we can apply \Lemma{BHEDerivative} with $h = \tilde{p}_{\alpha}$ and $f = \tilde{R}_{\alpha}$ to assert the following properties on $\tilde{R}_{\alpha}$.
\begin{lemma}
\LemmaName{RSatisfiesBHE}
$\tilde{R}_{\alpha}$ satisfies the following properties:
\begin{enumerate}[label=(\arabic*),noitemsep,topsep=0pt]
\item $\tilde{R}_{\alpha}$ is $C^{1,2}$,
\item $\tilde{R}_{\alpha}$ satisfies $\doob \tilde{R}_\alpha = 0$ over $\bR_{>0} \times \bR$,
\item $\partial_x \tilde{R}_{\alpha}(t,x) = \tilde{p}_\alpha(t,x)$.
\end{enumerate}
\end{lemma}
\Lemma{RSatisfiesBHE} shows that $\tilde{R}_{\alpha}$ satisfies the hypotheses of \Proposition{FindingF}.
Hence, we have
\[
    \PseudoContReg{T, \tilde{p}_{\alpha}, B} ~=~ \tilde{R}_{\alpha}(T, \Abs{B_T}).
\]
Since $\erfi(\cdot)$ is a strictly increasing function with $\erfi(0) = 0$, observe that $\partial_x \tilde{R}_{\alpha} = \tilde{p}_{\alpha}$ has exactly one root at $\alpha\sqrt{t}$.
In particular, for any fixed $T > 0$, the function $\tilde{R}_{\alpha}(T, x)$ is maximized at $x = \alpha \sqrt{T}$.
Therefore, for every $T$ we have
\[
    \tilde{R}_{\alpha}(T, \Abs{B_T}) ~\leq~ \max_{x \geq 0}\tilde{R}_{\alpha}(T, x) ~\leq~
    \tilde{R}_{\alpha}(T, \alpha \sqrt{T}) ~=~ \left ( \frac{\alpha}{2} + \kappa_{\alpha}M_0\left(\frac{\alpha^2}{2}\right) \right ) {\sqrt{T}},
\]
where the equality is by definition of $\tilde{R}_{\alpha}$ in \Equation{RtildeDef}.
To summarize, we have shown that
\begin{equation}
\EquationName{R_bound}
    \PseudoContReg{T, \tilde{p}_{\alpha}, B} ~\leq~ \left ( \frac{\alpha}{2} + \kappa_{\alpha}M_0\left(\frac{\alpha^2}{2}\right) \right ) {\sqrt{T}}.
\end{equation} 
This establishes \eqref{eq:PRegretResult}, as desired.

\subsubsection{Resolving the non-negativity issue}
\SubsectionName{truncation}
The only remaining step is to modify $\tilde{p}_{\alpha}$ so that it lies in the interval $[0,1/2]$.
We modify $\tilde{p}_{\alpha}$ in the most natural way: by modifying all negative values to be zero.
Specifically, we set
\begin{equation}
    \EquationName{p_definition}
    p_{\alpha}(t,x)
    ~\coloneqq~
    \begin{cases}
    0 &~~\text{($t = 0$)} \\
    (\tilde{p}_{\alpha}(t,x))_+
    &~~\text{($t > 0$)}
    \end{cases}
    ~=~
    \begin{cases}
    0 &~~\text{($t = 0$)} \\
    \frac{1}{2}\left( 1 - \frac{\erfi(\sfrac{x}{\sqrt{2t}})}{\erfi(\sfrac{\alpha}{\sqrt{2}})} \right)_+ &~~\text{($t > 0$)}
    \end{cases}.
\end{equation}
Here, we use the notation $(x)_+ = \max\{0,x\}$.
Note that $p_{\alpha}(t,0) = 1/2$ for all $t > 0$ and $p_{\alpha}(t,x) \in [0,1/2]$ for all $t,x \geq 0$.
So $p_{\alpha}$ defines a valid continuous-time algorithm.
From \Equation{p_definition}, we obtain a truncated version of $\tilde{R}_{\alpha}$ as
\begin{equation}
    \EquationName{R_definition}
    R_{\alpha}(t,x)
    ~\coloneqq~
    \begin{cases}
        0 &\quad(t = 0) \\
        \tilde{R}_{\alpha}(t,x) &\quad(t > 0 \wedge x \leq \alpha \sqrt{t}) \\
        \tilde{R}_{\alpha}(t,\alpha\sqrt{t}) &\quad(t > 0 \wedge x \geq \alpha \sqrt{t})
    \end{cases}.
\end{equation}
It is straightforward to verify that $\partial_x R_{\alpha} = p_{\alpha}$.
This is because for $x \leq \alpha \sqrt{t}$, $p_{\alpha}(t,x) = \tilde{p}_{\alpha}(t,x)$ and $R_{\alpha}(t,x) = \tilde{R}_{\alpha}(t,x)$ (we have computed the derivatives in \Lemma{RSatisfiesBHE}).
In addition, $R_{\alpha}(t,x)$ is constant (in $x$) for $x \geq \alpha\sqrt{t}$ so its derivative (in $x$) is $0$.

If $R_{\alpha}$ were sufficiently smooth then we could immediately apply It\^{o}'s formula (\Theorem{ItoFormula})
to obtain a formula for the regret
of $p_{\alpha}$.
For $x < \alpha \sqrt{t}$, we have $\doob R_{\alpha}(t,x) = 0$ by \Lemma{RSatisfiesBHE}
and for $x > \alpha \sqrt{t}$, it is not difficult to verify that $\doob R_{\alpha}(t,x) > 0$.
It\^{o}'s formula would then suggest that $\ContReg{T, p_{\alpha}, B} \leq \Reg(T, |B_T|)$.
The only flaw is that $\partial_{xx} R_{\alpha}$ is
not well-defined on the curve $\setst{(t, \alpha \sqrt{t})}{t > 0}$
so $R_{\alpha}$ is not in $C^{1,2}$ and \Theorem{ItoFormula} cannot be applied directly.
The reader who believes that this issue is unlikely to be problematic may wish to take \Lemma{ContRegUB} on faith and skip ahead to \Subsection{OptBdry}.

\begin{figure}[ht]
    \centering
    \includegraphics[scale=0.109]{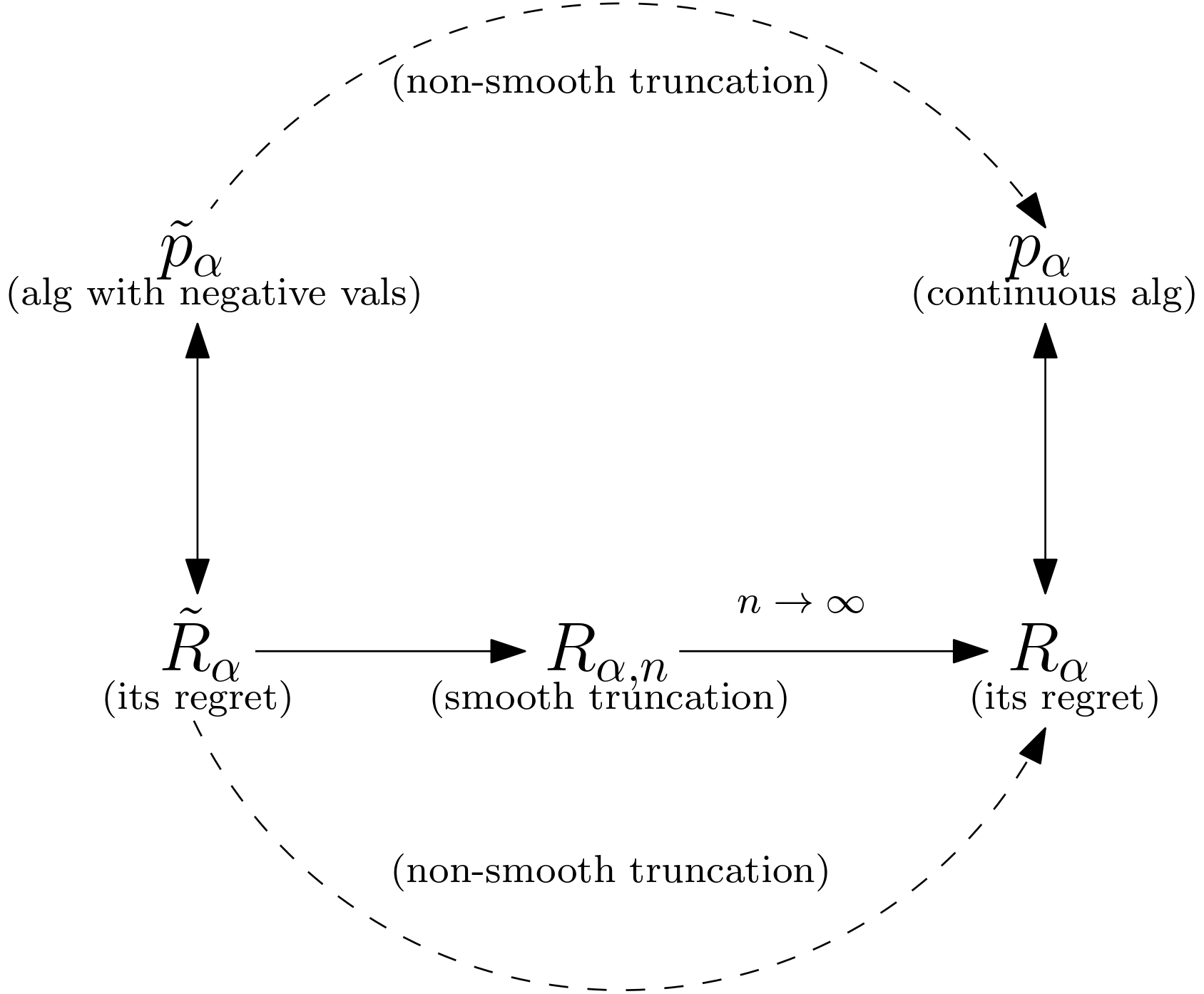}
    \caption[The relationships between $\tilde{p}_{\alpha}$, $\tilde{R}_{\alpha}$, $R_{\alpha,n},$  $p_{\alpha}$, and $R_{\alpha}$.]{
        The relationships between $\tilde{p}_{\alpha}$, $\tilde{R}_{\alpha}$, $R_{\alpha,n},$  $p_{\alpha}$, and $R_{\alpha}$.
        Since $R_{\alpha}$ is not sufficiently smooth, It\^{o}'s formula (\Theorem{ItoFormula}) cannot be applied.
        Instead, we show that $R_{\alpha}$ is the limit of $R_{\alpha, n}$ which are smooth truncations of $\tilde{R}_{\alpha}$.
        Since each $\tilde{R}_{\alpha, n}$ is smooth, It\^{o}'s formula can be applied to each of them.
    }
    \label{fig:my_label}
\end{figure}

\begin{restatable}{lemma}{ContRegUB}
    \LemmaName{ContRegUB}
    Fix $\alpha > 0$.
    Then, almost surely, for all $T \geq 0$, $\ContReg{T, p_{\alpha}, B} \leq R_{\alpha}(T, |B_T|)$.
\end{restatable}
Here, we will present a high-level overview of the proof of this lemma;
the details can be found in \Appendix{ApproximationArgument}.
Let $\phi(x)$ be a smooth function satisfying $\phi(x) = 1$ for $x \leq 0$ and $\phi(x) = 0$ for $x \geq 1$.
For $n \in \bN$, define $\phi_n(x) = \phi(nx)$ and the approximations
\[
    R_{\alpha,n}(t,x) \coloneqq \tilde{R}_{\alpha}(t,x) \phi_n(x - \alpha \sqrt{t})
    + \tilde{R}_{\alpha}(t, \alpha \sqrt{t}) (1-\phi_n(x - \alpha \sqrt{t})).
\]
It is relatively straightforward to check that $R_{\alpha,n}(t,x) \xrightarrow{n \to\infty} R_{\alpha}(t,x)$ pointwise and
similarly for the derivatives.
The important property is that $R_{\alpha,n}$ is smooth so It\^{o}'s formula may be applied.
\Lemma{ContRegUB} is then proved by taking limits and controlling the error terms.

The remainder of this section proves
\Theorem{ContsMainResult} by setting $p^* = p_{\alpha}$ for the optimal $\alpha$.

\begin{remark}
	The definition of $p_{\alpha}$ has an interesting interpretation.
	Let $B$ be a Brownian Motion.
	Fix a time $t$ and a position $x > 0$.
    Now let $\tau = \inf\setst{s > t}{|B_s| \geq \alpha \sqrt{s}}$.
	It is known \cite{Donchev} that $p_{\alpha}(t,x) = \probg{B_\tau < 0}{B_t = x}$.
	In words, $p_{\alpha}(t,x)$ is the probability that a Brownian Motion started at time $t$ and position $x$ crosses the bottom $-\alpha \sqrt{t}$ boundary
	before crossing the top $+\alpha \sqrt{t}$ boundary.
	As a sanity check, one may observe that $p_{\alpha}(t, \alpha \sqrt{t}) = 0$ and $p_{\alpha}(t,0) = 0.5$.
	Interestingly, the optimal algorithms for two experts in both the finite-time horizon setting \cite{Cover66}
	and the geometric time horizon setting have a similar interpretation \cite{GPS17}.
	In both cases, the optimal algorithm is to assign the probability that a random walk started at position $x > 0$ at time $t$ remains positive at the stopping time.
	In the finite-time case, the stopping time is a deterministic quantity $T$ whereas in the geometric-time case, the stopping time is a geometric random variable.
	A similar connection also exists for three and four experts \cite{GPS17, BEZ20}.
\end{remark}

\subsection{Optimizing the boundary to minimize  continuous regret}
\SubsectionName{OptBdry}
By \Lemma{ContRegUB}, $\ContReg{T, \partial_x R_{\alpha}, B} \leq R_{\alpha}(T, |B_T|) \leq R_{\alpha}(T, \alpha \sqrt{T})$,
where the last inequality is because $\partial_x R_{\alpha}(t,x) = p_{\alpha}(t,x)$ is positive for $x \in [0, \alpha\sqrt{t})$ and $0$ for $x \geq \alpha \sqrt{t}$.
As observed in \Equation{R_bound}, we have the formula $R_{\alpha}(T, \alpha \sqrt{T}) = \left( \sfrac{\alpha}{2} + \kappa_{\alpha} M_0(\sfrac{\alpha^2}{2}) \right) \sqrt{T}$.
Thus, to minimize $R_{\alpha}(T, \alpha \sqrt{T})$, it is convenient to define
\[
    h(\alpha) \coloneqq R_{\alpha}(1,\alpha) = \frac{\alpha}{2} + \kappa_{\alpha} M_0(\alpha^2/2).
\]
The only remaining task is now to solve the following optimization problem.
\begin{equation}
    \EquationName{MinimizeBoundaryProblem}
    \min_{\alpha > 0} h(\alpha) = \min_{\alpha > 0} \left\{ \frac{\alpha}{2} + \kappa_{\alpha}\cdot M_0\left( \frac{\alpha^2}{2} \right) \right\}
\end{equation}

The following lemma verifies that there exists some $\alpha$ for which $\ContReg{T, \partial_x R_{\alpha}, B} \leq \frac{\gamma\sqrt{T}}{2}$, completing the proof of \Theorem{ContsMainResult}.
\begin{lemma}
\LemmaName{BoundaryOptSolution} 
The function $h(\alpha)$ is minimized at $\alpha = \gamma$ and $h(\gamma) = \gamma / 2$.
Consequently, for any fixed $T > 0$, $\min_{\alpha} R_{\alpha}(T, \alpha\sqrt{T}) = R_{\gamma}(T, \gamma\sqrt{T}) = \frac{\gamma \sqrt{T}}{2}$.
\end{lemma}
\Lemma{BoundaryOptSolution} follows easily from the following claim.
\begin{claim}
    \ClaimName{boundary_prob_derivative}
    $h'(\alpha) = -\frac{\exp(\alpha^2/2)}{\pi \erfi(\alpha/\sqrt{2})} \cdot M_0(\alpha^2/2)$.
    In particular, $h'(\alpha) < 0$ for $\alpha \in (0, \gamma)$, $h'(\gamma) = 0$, and $h'(\alpha) > 0$ for $\alpha \in (\gamma, \infty)$.
\end{claim}
\begin{proof}
	Recall that $h(\alpha) = \frac{\alpha}{2} + \frac{M_0(\alpha^2/2)}{\sqrt{2\pi} \erfi(\alpha/\sqrt{2})}$
	and that $\frac{\dd}{\dd x} \erfi(x / \sqrt{2}) = \sqrt{\frac{2}{\pi}} e^{x^2/2}$.
	Hence,
	\begin{align*}
		h'(\alpha)
		& = \frac{1}{2} - \frac{\alpha \cdot M_1(\alpha^2/2)}{\sqrt{2\pi} \erfi(\alpha/\sqrt{2})}
		- \frac{\exp(\alpha^2/2) \cdot M_0(\alpha^2/2)}{\pi \erfi(\alpha/\sqrt{2})^2} &&\quad\text{(by \Fact{chf_derivs})} \\
		& = - \frac{\exp(\alpha^2/2) \cdot M_0(\alpha^2/2)}{\pi \erfi(\alpha/\sqrt{2})^2} &&\quad\text{(by \Fact{basic_identities}\ref{item:M1})}.
	\end{align*}
	This proves the first assertion.
	
	Next, observe that $\frac{\exp(\alpha^2/2)}{\erfi(\alpha/\sqrt{2})^2}$ is positive for all $\alpha > 0$.
	Hence, by \Fact{M0_unique_root}, we have that $h'(\alpha) < 0$ for $\alpha \in (0, \gamma)$, $h'(\gamma) = 0$, and $h'(\alpha) > 0$ for $\alpha \in (\gamma, \infty)$.
\end{proof}
\begin{proof}[Proof of \Lemma{BoundaryOptSolution}]
\Claim{boundary_prob_derivative} implies that $\gamma$ is the global minimizer for $h(\alpha)$.
Since $\gamma$ is a root of $M_0(\alpha^2 / 2)$, it follows that $h(\gamma) = \gamma / 2$.
This proves the first assertion.
Next, for every $\alpha > 0$, we have $R_{\alpha}(T, \alpha \sqrt{T}) = \sqrt{T} \cdot h(\alpha) \geq \sqrt{T} \cdot h(\gamma) = \gamma \sqrt{T} / 2$,
which proves the second assertion.
\end{proof}

%% file: app_hypergeometric.tex
\section{Standard concavity facts}
\AppendixName{standard}
\begin{fact}
    \FactName{concave_decreasing}
    Suppose $f \colon \bR \to \bR$ is concave.
    Then for any $\alpha < \beta$, the function $g(t) = f(t + \beta) - f(t + \alpha)$ is non-increasing.
\end{fact}
\begin{fact}
    \FactName{concave_ineq_check}
    Suppose that $f \colon \bR \to \bR$ is concave.
    Let $\alpha < \beta$.
    Then $f(x) \geq \min\{f(\alpha), f(\beta)\}$ for all $x \in [\alpha, \beta]$.
\end{fact}

%% file: app_cts_upper_bound.tex
\section{Additional proofs for \Section{cts_ub}}
\AppendixName{app_cts_ub}

\subsection{Proof of \Lemma{ContRegUB}}
\AppendixName{ApproximationArgument}
The main idea of the proof is that we will approximate $R_{\alpha}$ by a sequence of smooth functions (i.e.~functions in $C^{2,2}$).

Fix $\alpha > 0$.
Recall that $\tilde{R}_{\alpha}(t,x) = \frac{x}{2} + \kappa_{\alpha}\sqrt{t} \cdot M_0\left( \frac{x^2}{2t} \right)$ for $t > 0, x \in \bR$,
where $\kappa_{\alpha} = \frac{1}{\sqrt{2\pi} \erfi(\sfrac{\alpha}{\sqrt{2}})}$.
(For $t = 0$, it suffices to define $\tilde{R}_{\alpha}(t,x) = 0$.)
We also have the truncated version, $R_{\alpha}$, defined as
\[
    R_{\alpha}(t,x) =
    \begin{cases}
        \tilde{R}_{\alpha}(t,x) & t > 0 \wedge x \leq \alpha \sqrt{t} \\
        \tilde{R}_{\alpha}(t,\alpha\sqrt{t}) & t > 0 \wedge x \geq \alpha \sqrt{t} \\
        0 & t = 0
    \end{cases}.
\]
Recall also that $p_{\alpha} = \partial_x R_{\alpha}$.
For convenience, we restate the lemma.
\ContRegUB*

For the remainder of this section,
we will write $\tilde{f} = \tilde{R}_{\alpha}$ and $f = R_{\alpha}$.
Let $\phi(x)$ be any non-increasing $C^2$ function satisfying $\phi(x) = 1$ for $x \leq 0$ and $\phi(x) = 0$ for $x \geq 1$.
For concreteness, we may take
\begin{equation}
    \EquationName{phi_def} 
    \phi(x) =
    \begin{cases}
        1 & x \leq 0 \\
        (1-x) + \frac{1}{2\pi} \sin(2\pi x) & x \in [0,1] \\
        0 & x \geq 1
    \end{cases}.
\end{equation}
We leave it as an easy calculus exercise to verify that $\phi$ is indeed a non-increasing $C^2$ function.

Next, define $\phi_n(x) = \phi(nx)$ and
\[
    f_n(t,x) = \tilde{f}(t,x) \cdot \phi_n(x - \alpha \sqrt{t})
    + f(t, \alpha \sqrt{t}) \cdot \left( 1 - \phi_n(x - \alpha \sqrt{t}) \right).
\]
Note that $f_n \in C^{2,2}$ on $\bR_{> 0} \times \bR$ for all $n$.
The function $f_n$ is a smooth approximation to $f$ and its limit is exactly $f\,(= R_{\alpha})$.
\begin{claim}
    \ClaimName{fn_to_f}
    For every $t>0, x \in \bR$, $\lim_{n \to \infty} f_n(t,x) = f(t,x)$.
\end{claim}
\begin{proof}
    If $x \leq \alpha \sqrt{t}$ then $\phi_n(x - \alpha \sqrt{t}) = 1$ so
    $f_n(t,x) = \tilde{f}(t,x) = f(t,x)$.
    In particular, this also holds for the limit.
    Next, suppose that $a = x - \alpha\sqrt{t} > 0$.
    If $n > 1/a$ then $\phi_n(x - \alpha \sqrt{t}) = 0$ so $f_n(t,x) = \tilde{f}(t, \alpha \sqrt{t}) = f(t,x)$.
\end{proof}
Recall that our goal is to
relate $f(T, |B_T|)$ and $\int_0^T \partial_x f(t, |B_t|)\, \dd |B_t|$.
However, one cannot apply It\^{o}'s formula to $f$ directly as it is not in $C^{1,2}$.
Instead, we will apply It\^{o}'s formula to the smoothed version of $f$, namely $f_n$, and then take limits.
The remainder of this section does this limiting argument carefully.

For technical reasons (namely that $\tilde{f}(t,x)$ has a pole when $t \to 0$ and $x \neq 0$), we will not be able to start the stochastic integral at $0$.
Hence, we will fix $\eps > 0$ and, at the end of the proof, we will allow $\eps \to 0$.

The following lemma bounds the stochastic integral of $\partial_x f_n$ with respect to $|B_t|$.
\begin{lemma}
    \LemmaName{fn_stochastic}
    Almost surely, for every $T \geq \eps$
    \begin{equation}
    \EquationName{fn_ito_bound}
    \begin{split}
        \int_{\eps}^T \partial_x f_n(t, |B_t|)\, \dd |B_t|
        & \leq
        f_n(T, |B_T|) - f_n(\eps, |B_{\eps}|) \\
        & - \int_{\eps}^T \frac{\alpha}{2\sqrt{t}} \cdot \phi_n'(|B_t| - \alpha\sqrt{t}) \cdot
        \left( f(t, \alpha \sqrt{t}) - \tilde{f}(t, |B_t|) \right) \, \dd t \\
        & - \frac{1}{2} \int_{\eps}^T \phi_n''(|B_t| - \alpha\sqrt{t}) \cdot
        \left( f(t, \alpha\sqrt{t}) - \tilde{f}(t, |B_t|) \right) \, \dd t.
    \end{split}
    \end{equation}
\end{lemma}
\begin{proof}
    The proof is by It\^{o}'s formula (\Theorem{ItoFormula}) applied to $f_n$.
    We have, for all $T \geq \eps$,
    \begin{equation}
        \EquationName{fn_ito1}
        f_n(T, |B_T|) - f_n(\eps, |B_{\eps}|)
        =
        \int_{\eps}^T \partial_x f_n(t, |B_t|) \, \dd |B_t|
        + \int_{\eps}^T \partial_t f_n(t, |B_t|) + \frac{1}{2} \partial_{xx} f_n(t, |B_t|)\, \dd t.
    \end{equation}
    Computing derivatives of $f_n$, we have
    \begin{align}
    \begin{split}
        \EquationName{partialt_fn}
        \partial_t f_n(t, x)
        & = (\partial_t \tilde{f}(t,x)) \cdot \phi_n(x - \alpha \sqrt{t})
          - \frac{\alpha}{2\sqrt{t}} \tilde{f}(t,x) \phi_n'(x - \alpha \sqrt{t}) \\
        & + \partial_t (f(t, \alpha \sqrt{t})) \cdot (1 - \phi_n(x - \alpha \sqrt{t})) + \frac{\alpha}{2\sqrt{t}} f(t, \alpha \sqrt{t}) \cdot \phi_n'(x - \alpha\sqrt{t})
    \end{split} \\
    \begin{split}
        \EquationName{partialg_fn}
        \partial_x f_n(t,x)
        & = (\partial_x \tilde{f}(t,x)) \cdot \phi_n(x - \alpha \sqrt{t})
          + \tilde{f}(t,x) \phi_n'(x - \alpha \sqrt{t})
          - f(t, \alpha \sqrt{t}) \phi_n'(x - \alpha \sqrt{t})
    \end{split} \\
    \begin{split}
        \EquationName{partialgg_fn}
        \partial_{xx} f_n(t,x)
        & = (\partial_{xx} \tilde{f}(t,x)) \cdot \phi_n(x - \alpha \sqrt{t})
          + 2(\partial_x\tilde{f}(t,x)) \phi_n'(x - \alpha \sqrt{t}) \\
        & + \left(\tilde{f}(t,x) - f(t, \alpha \sqrt{t}) \right) \phi_n''(x - \alpha \sqrt{t}).
    \end{split}
    \end{align}
    Recalling the notation $\doob = \partial_t + \frac{1}{2} \partial_{xx}$, we have
    \begin{equation}
    \EquationName{doob_fn}
    \begin{split}
        \doob f_n(t,x)
        & = \left( \doob \tilde{f}(t,x) \right) \cdot \phi_n(x - \alpha \sqrt{t})
          + \partial_t (f(t, \alpha \sqrt{t})) \cdot (1 - \phi_n(x- \alpha\sqrt{t})) \\
        & + (\partial_x \tilde{f}(t,x)) \phi_n'(x - \alpha \sqrt{t}) \\
        & + \frac{\alpha}{2 \sqrt{t}} \cdot (f(t, \alpha \sqrt{t}) - \tilde{f}(t, x)) \cdot \phi_n'(x - \alpha\sqrt{t})
          + \frac{1}{2}\left( \tilde{f}(t,x) - f(t, \alpha \sqrt{t}) \right) \phi_n''(x - \alpha \sqrt{t}).
    \end{split}
    \end{equation}
    By \Lemma{RSatisfiesBHE}, $\doob \tilde{f} = 0$.
    By \Claim{partialt_f_bdry} below, $\partial_t ( f(t, \alpha \sqrt{t})) > 0$.
    Next, observe that $(\partial_x \tilde{f}(t,x)) \cdot \phi_n'(x - \alpha \sqrt{t}) \geq 0$.
    To see this, if $x \leq \alpha \sqrt{t}$ then $\phi_n'(x - \alpha\sqrt{t}) = 0$.
    On the other hand, if $x > \alpha \sqrt{t}$ then $\phi_n'(x - \alpha \sqrt{t}) \leq 0$ because $\phi_n$ is non-increasing and $\partial_x \tilde{f}(t,x) \leq 0$ by \Lemma{RSatisfiesBHE} and \Equation{palpha_def}.
    Hence, we can lower bound \Equation{doob_fn} by
    \begin{equation}
        \EquationName{doobfn_lb}
        \doob f_n(t,x)
        \geq \frac{\alpha}{2 \sqrt{t}} \cdot (f(t, \alpha \sqrt{t}) - \tilde{f}(t, x)) \cdot \phi_n'(x - \alpha\sqrt{t})
          + \frac{1}{2}\left( \tilde{f}(t,x) - f(t, \alpha \sqrt{t}) \right) \phi_n''(x - \alpha \sqrt{t}).
    \end{equation}
    Plugging \Equation{doobfn_lb} into \Equation{fn_ito1} gives
    \begin{equation}
    \EquationName{fn_ito_lb}
    \begin{split}
        f_n(T, |B_T|) - f_n(\eps, |B_{\eps}|)
        & \geq \int_{\eps}^T \partial_x f_n(t, |B_t|)\, \dd |B_t| \\
        & + \int_{\eps}^T \frac{\alpha}{2\sqrt{t}} \cdot \phi_n'(|B_t| - \alpha\sqrt{t}) \cdot
        \left( f(t, \alpha \sqrt{t}) - \tilde{f}(t, |B_t|) \right) \, \dd t \\
        & + \frac{1}{2} \int_{\eps}^T \phi_n''(|B_t| - \alpha\sqrt{t}) \cdot
        \left( f(t, \alpha\sqrt{t}) - \tilde{f}(t, |B_t|) \right) \, \dd t.
    \end{split}
    \end{equation}
    Rearranging \Equation{fn_ito_lb} gives the lemma.
\end{proof}
\begin{claim}
    \ClaimName{partialt_f_bdry}
    If $t > 0$ then $\partial_t ( \tilde{f}(t, \alpha \sqrt{t}) ) > 0$.
\end{claim}
\begin{proof}
Note that
\[
    \tilde{f}(t, \alpha\sqrt{t}) = \sqrt{t} \cdot \left( \frac{\alpha}{2} + \frac{M_0(\sfrac{\alpha^2}{2})}{\sqrt{2\pi} \erfi(\sfrac{\alpha}{\sqrt{2}})} \right)
    = \sqrt{t} \cdot f(1,\alpha).
\]
So it suffices to check that $\tilde{f}(1,\alpha) > 0$.
To see this, note that $\tilde{f}(1,0) = \kappa_{\alpha} > 0$ and $\partial_x \tilde{f}(1,x) \geq 0$ as long as $x \leq \alpha$ (by the first identity of \Lemma{uncap_rtg_deriv_cts}).
Hence, $\tilde{f}(1,\alpha) > 0$.
\end{proof}

At this point, we would like to take limits on both sides of \Equation{fn_ito_bound}.
This is achieved by the following two lemmas.
\begin{lemma}
    \LemmaName{fn_converge}
    Almost surely, for every $T \geq \eps$,
    \begin{enumerate}[itemsep=2pt, topsep=2pt]
    \item $\lim_{n \to \infty} \int_{\eps}^T \frac{\alpha}{2\sqrt{t}} \cdot \phi_n'(|B_t| - \alpha\sqrt{t}) \cdot
    \left( f(t, \alpha \sqrt{t}) - \tilde{f}(t, |B_t|) \right) \, \dd t = 0$; and
    \label{item:err_bound1}
    \item $\lim_{n \to \infty} \int_{\eps}^T \phi_n''(|B_t| - \alpha\sqrt{t}) \cdot
    \left( f(t, \alpha\sqrt{t}) - \tilde{f}(t, |B_t|) \right) \, \dd t = 0$.
    \label{item:err_bound2}
    \end{enumerate}
\end{lemma}
\begin{lemma}
    \LemmaName{stoch_int_converge}
    For every $T \geq \eps$,
    \[
        \int_{\eps}^T \partial_x f_n(t, |B_t|)\, \dd |B_t|
        \xrightarrow[]{L^2} \int_{\eps}^T \partial_x f(t, |B_t)\, \dd |B_t|
    \]
    as $n \to \infty$.
\end{lemma}
Within this section, $X_n \xrightarrow{L^2} X$ means that $\expect{(X_n-X)^2} \rightarrow 0$ as $n \to \infty$.
We relegate the proofs of \Lemma{fn_converge} and \Lemma{stoch_int_converge} to \Appendix{fn_converge}.
We now take limits on both sides of \Equation{fn_ito_bound} to obtain the following bound on the stochastic integral of $\partial_x f$.
\begin{lemma}
    \LemmaName{fn_bound_eps}
    Almost surely, for every $T \geq \eps$,
    \begin{equation}
        \EquationName{fn_bound_eps}
        \int_{\eps}^T \partial_x f(t, |B_t|)\, \dd |B_t|
        \leq f(T, |B_T|) - f(\eps, |B_{\eps}|).
    \end{equation}
\end{lemma}
\begin{proof}
    By \Lemma{stoch_int_converge}, for every $T \geq \eps$,
    \[
        \int_{\eps}^T \partial_x f_n(t, |B_t|)\, \dd |B_t|
        \xrightarrow[]{L^2} \int_{\eps}^T \partial_x f(t, |B_t)\, \dd |B_t|.
    \]
    Hence, there exists a subsequence $n_k$ such that
    \[
        \int_{\eps}^T \partial_x f_{n_k}(t, |B_t|)\, \dd |B_t|
        \xrightarrow[]{\text{a.s.}} \int_{\eps}^T \partial_x f(t, |B_t)\, \dd |B_t|.
    \]
    Using \Lemma{fn_stochastic} to bound the left-hand-side
    and then \Lemma{fn_converge} to take limits gives that \Equation{fn_bound_eps} holds for any fixed $T \geq\eps$.
    Hence, almost surely, \Equation{fn_bound_eps} holds for all rational $T \geq \eps$.
    As both sides of \Equation{fn_bound_eps} are continuous as a function of $T$, \Equation{fn_bound_eps} holds for all $T \geq \eps$.
\end{proof}
\begin{proofof}{\Lemma{ContRegUB}}
    We will work in the probability 1 set where \Lemma{fn_bound_eps} holds (for every rational $\eps > 0$) and $t \mapsto B_t$ is continuous.

    Fix $T > 0$.
    Note that $\ContReg{T, \partial_x f, B}$ is defined because $\partial_x f \in [0,1/2]$ and $\partial_x f(t,0) = 1/2$ for all $t > 0$ (see \Equation{p_definition}).
    Recalling \Definition{ContRegret}, we have, for $\eps \leq T$,
    \begin{align*}
        \ContReg{T, \partial_x f, B}
        & = \int_0^T \partial_x f(t, |B_t|)\, \dd |B_t| \\
        & = \int_{\eps}^T \partial_x f(t, |B_t|)\, \dd |B_t|
          + \int_{0}^{\eps} \partial_x f(t, |B_t|)\, \dd |B_t| \\
        & \leq f(T, |B_T|) - f(\eps, |B_{\eps}|) + \int_0^{\eps} \partial_x f(t, |B_t|)\, \dd |B_t| && \qquad \text{(\Lemma{fn_bound_eps})}.
    \end{align*}
    The right-hand-side is continuous in $\eps$ so taking $\eps \to 0$
    (and recalling that $f(0,0) = 0$),
    gives
    \[
        \ContReg{T, \partial_x f, B} \leq f(T, |B_T|).  \qedhere
    \]
\end{proofof}

\subsection{Additional proofs from \Appendix{ApproximationArgument}}
\AppendixName{fn_converge}
Before we prove \Lemma{fn_converge}, we will need one key observation.
\begin{lemma}
    \LemmaName{f_taylor_error}
    Fix $\eps > 0$.
    Then there is a constant $C_{\eps} > 0$ (depending also on $\alpha$) such that for $t > 0$ and $x$ satisfying $\abs{x - \alpha\sqrt{t}} \leq 1$,
    \begin{enumerate}[itemsep=2pt,topsep=0pt]
    \item $\abs{\tilde{f}(t,x) - f(t, \alpha \sqrt{t})}
    \leq C_{\eps} \cdot (x - \alpha \sqrt{t})^2$; and
    \item $\abs{\partial_x \tilde{f}(t,x)} \leq C_{\eps} \cdot \abs{x - \alpha\sqrt{t}}$.
    \end{enumerate}
\end{lemma}
\begin{proof}
    The key observation is that $f(t, \alpha \sqrt{t})$ is already a first-order Taylor expansion of $\tilde{f}(t,x)$ (in $x$) about the point $\gamma \sqrt{t}$.
    Indeed, $\tilde{f}(t,\alpha \sqrt{t}) = f(t, \alpha \sqrt{t})$ and $(\partial_x \tilde{f})(t, \alpha, \sqrt{t}) = 0$.
    Hence, by Taylor's Theorem (see e.g.~\cite[Theorem~5.15]{RudinPrinciples})
    \[
        \abs{\tilde{f}(t,x) - f(t, \alpha\sqrt{t})}
        \leq
        \frac{1}{2} \cdot (x - \alpha\sqrt{t})^2 \cdot
        \sup_{t \geq \eps, \abs{x-\alpha\sqrt{t}} \leq 1} \abs{\partial_{xx} \tilde{f}(t,x)}
    \]
    By the second identity in \Lemma{uncap_rtg_deriv_cts}, we have
    \[
        \abs{\partial_{xx} \tilde{f}(t,x)} = \frac{\kappa_{\alpha}\exp(\sfrac{x^2}{2t})}{\sqrt{2t}}. 
    \]
    Since $t \geq \eps$ and $x \leq 1 + \alpha \sqrt{t}$, we have
    \begin{align*}
        \abs{\partial_{xx} \tilde{f}(t,x)}
        & \leq \frac{\kappa_{\alpha}\exp(\sfrac{(1+\alpha\sqrt{t})^2}{2t})}{\sqrt{2\eps}} \\
        & = \frac{\kappa_{\alpha}\exp(\alpha^2/2 + \alpha/\sqrt{t} + 1/t)}{\sqrt{2\eps}} \\
        & \leq
        \frac{\kappa_{\alpha}\exp(\alpha^2/2 + \alpha/\sqrt{\eps} + 1/\eps)}{\sqrt{2\eps}}.
    \end{align*}
    So one can take $C_{\eps} =
    \frac{\kappa_{\alpha}\exp(\alpha^2/2 + \alpha/\sqrt{\eps} + 1/\eps)}{\sqrt{2\eps}}$.
    This gives the first assertion.
    
    The second assertion is similar.
    Indeed, since $(\partial_x \tilde{f})(t,\alpha\sqrt{t}) = 0$, we have
    \begin{align*}
        \abs{(\partial_x \tilde{f})(t,x)}
        & =
        \abs{(\partial_x \tilde{f})(t,x)-(\partial_x \tilde{f})(t,\alpha\sqrt{t})} \\
        & \leq \abs{x-\alpha\sqrt{t}}\cdot \sup_{t \geq \eps, |x-\alpha\sqrt{t}| \leq 1} |\partial_{xx}\tilde{f}(t,x)| \\
        & \leq C_{\eps} \cdot \abs{x-\alpha\sqrt{t}}. \qedhere
    \end{align*}
\end{proof}

We also need a simple claim which bounds the value of $\abs{\phi_n'(x)}$ and $\abs{\phi_n''(x)}$.
\begin{claim}
    \ClaimName{phi_deriv_bound}
    There is an absolute constant $C > 0$ such that $\abs{\phi_n'(x)} \leq Cn$
    and $\abs{\phi_n''(x)} \leq Cn^2$.
\end{claim}
\begin{proof}
    Note that $\phi_n'(x) = n\cdot \phi'(x)$ and $n^2 \cdot \phi''(x)$.
    It is easy to see, from differentiating \Equation{phi_def} or by continuity and compact arguments,
    that there exists $C > 0$ such that $\abs{\phi'(x)}, \abs{\phi''(x)} \leq C$ for all $x \in \bR$.
\end{proof}

\begin{proofof}{\Lemma{fn_converge}}
We start with the second assertion.
The first assertion is similar but simpler.
We claim that there exists a constant $C'$ (depending on $\eps$ and $\alpha$) such that
\begin{equation}
    \EquationName{fn_converge_1}
    \left|\phi_n''(|B_t| - \alpha \sqrt{t}) \cdot \left( f(t, \alpha\sqrt{t}) - \tilde{f}(t, |B_t|) \right)\right|
    \leq C' \ind[ |B_t| - \alpha \sqrt{t} \in [0,\sfrac{1}{n}] ]
\end{equation}
Indeed, if $|B_t| - \alpha\sqrt{t} \notin [0,1/n]$ then $\phi_n''(|B_t| - \alpha\sqrt{t}) = 0$
so both sides of \Equation{fn_converge_1} are equal to $0$.
On the other hand, if $|B_t| - \alpha\sqrt{t} \in [0,1/n]$ then \Lemma{f_taylor_error}
shows that $|f(t,\alpha\sqrt{t}) - \tilde{f}(t, |B_t|)| \leq C_{\eps} / n^2$ where $C_{\eps}$ is the constant from \Lemma{f_taylor_error}.
Next, \Claim{phi_deriv_bound} gives $|\phi_n''(|B_t| - \alpha \sqrt{t})| \leq Cn^2$.
So taking $C' = C_{\eps} \cdot C$ gives \Equation{fn_converge_1}.
Hence,
\begin{multline*}
    \left|\int_{\eps}^T \phi_n''(|B_t| - \alpha\sqrt{t}) \cdot \left( f(t, \alpha\sqrt{t}) - \tilde{f}(t, |B_t|) \right) \, \dd t\right|
    \leq \int_{\eps}^T C' \cdot \ind[ |B_t| - \alpha\sqrt{t} \in [0,1/n]]\, \dd t \\
    = C'\cdot m\left( \setst{t \in [\eps, T]}{ |B_t| - \alpha \sqrt{t} \in [0,1/n]} \right),
\end{multline*}
where $m$ denotes the Lebesgue measure.
By continuity of measure, we have
\[
    \lim_n m\left( \setst{t \in [\eps, T]}{ |B_t| - \alpha \sqrt{t} \in [0,1/n]} \right)
    = \int_{\eps}^T \ind\left[ |B_t| = \alpha\sqrt{t} \right]\, \dd t = 0 \quad \text{a.s.}
\]
This proves the second assertion.

For the first assertion, we can use the bound (from \Lemma{f_taylor_error} and \Claim{phi_deriv_bound})
\begin{equation}
    \EquationName{fn_converge_2}
    \left|\phi_n'(x - \alpha \sqrt{t}) \cdot \left( f(t, \alpha\sqrt{t}) - \tilde{f}(t, x) \right)\right|
    \leq \frac{C'}{n} \ind[ x - \alpha \sqrt{t} \in [0,\sfrac{1}{n}] ] \leq \frac{C'}{n}.
\end{equation}
Hence,
\begin{align*}
    \left|\int_{\eps}^T \frac{\alpha}{2\sqrt{t}} \cdot \phi_n'(|B_t| - \alpha\sqrt{t}) \cdot
    \left( f(t, \alpha \sqrt{t}) - \tilde{f}(t, |B_t|) \right) \, \dd t \right|
    &
    \leq \int_{\eps}^T \frac{\alpha}{2\sqrt{t}} \frac{C'}{n} \, \dd t \\
    & \leq C'\alpha \sqrt{T} / n \rightarrow 0. \qedhere
\end{align*}
\end{proofof}
\begin{proofof}{\Lemma{stoch_int_converge}}

By \Equation{partialg_fn}, we have
\begin{equation}
\EquationName{partialg_fn2}
\begin{split}
    \partial_x f_n(t,x) - \partial_x f(t,x)
    & = \left(\partial_x \tilde{f}(t,x) \phi_n(x-\alpha\sqrt{t}) - \partial_x f(t,x)\right) \\
    & + \left(\phi_n'(x-\alpha\sqrt{t}) \cdot \left( \tilde{f}(t,x) - f(t,\alpha\sqrt{t}) \right)\right).
\end{split}
\end{equation}
For the first bracketed term, since $\partial_x \tilde{f}(t,x) = \partial_x f(t,x)$ when $x \leq \alpha\sqrt{t}$
and $\partial_x f(t,x)= 0$ when $x \geq \alpha\sqrt{t}$,
we have
\begin{align*}
    \left|\partial_x \tilde{f}(t,x) \phi_n(x-\alpha\sqrt{t})\right|
    & = \left|\partial_x \tilde{f}(t,x)\phi_n(x-\alpha\sqrt{t})\right| \ind[x-\alpha\sqrt{t} \in [0,1/n]] \\
    & \leq \frac{C'}{n},
\end{align*}
where the final inequality is by the second assertion in \Lemma{f_taylor_error}.
The second bracketed term has been bounded in \Equation{fn_converge_2},
and so we have proved
\begin{equation}
\label{detderbnd}
\Bigl|\partial_xf_n(t,x)-\partial_x f(t,x)\Bigr|\le \frac{C''}{n}\text{ for all }t\ge \varepsilon\text{ and all }x.
\end{equation}
Tanaka's formula (see \cite[Theorem~IV.43.3]{RogersWilliamsII}) states that
\begin{equation*}
    |B_t|=\int_0^t\text{sign}(B_s)\, \dd B_s+L_t \eqqcolon W_t+L_t,
\end{equation*}
where $L$ is the local time at zero of $B$ and $W$ is a Brownian motion.
Recall that $t \mapsto L_t$ is a continuous non-decreasing random process which increases only on the set $\setst{t}{B_t=0}$.
Therefore by the It\^{o} isometry property, for any $T\ge\varepsilon$, 
\begin{align*}
\operatorname{E} \Bigl[&\Bigl(\int_\varepsilon^T\partial_xf_n(t,|B_t|)\, \dd |B|_t-\int_\varepsilon^T\partial_xf(t,|B_t|)\, \dd |B|_t\Bigr)^2\Bigr]\\
&\le 2\expect{\Bigl(\int_\varepsilon^T(\partial_xf_n-\partial_x f)(t,|B_t|)\, \dd W_t\Bigr)^2}
+
2\expect{\Bigl(\int_\varepsilon^T(\partial_xf_n-\partial_x f)(t,|B_t|))\, \dd L_t\Bigr)^2}\\
&=2\expect{\int_\varepsilon^T(\partial_xf_n-\partial_xf)(t,|B_t|)^2\, \dd t}+
2\expect{\Bigl(\int_\varepsilon^T(\partial_xf_n-\partial_x f)(t,0)\, \dd L_t\Bigr)^2}.
\end{align*}
Now use \eqref{detderbnd} to bound the right-hand side by
\[
    2(C''/n)^2T+2(C''/n)^2\expect{L_T^2}\le C'''n^{-2}T,
\]
where the last inequality uses Tanaka's formula (and the fact that $W_t$
is also a standard Brownian motion) to bound
\[
    \expect{L_T^2} = \expect{(|B_T| - W_T)^2}
    \leq 2\expect{|B_T|^2} + 2\expect{|W_T|^2}
    = 4\expect{|B_T|^2} = O(T).
\]
The result follows.
\end{proofof}

\subsection{Discussion on the statement of  \protect{\Theorem{ItoFormula}}}
\AppendixName{ito_discussion}
In this paper, we use the version of It\^{o}'s formula that appears in Remark 1 after Theorem IV.3.3 in \cite{RY13}.
It states that if $f \in C^{1,2}$,
$X$ is a continuous semimartingale\footnote{
    A continuous semimartingale $X$ is a process
    that can be written as $X = M + N$ where $M$ is a
    continuous local martingale and $N$ is a continuous adapted
    process of finite variation.
}
and $A$ is a process with bounded variation then
\begin{equation}
\EquationName{RYIto}
\begin{split}
    f(A_T, X_T) - f(A_0, X_0)
    & = \int_0^T \partial_x f(A_t, X_t)\, \dd X_t
      + \int_0^T \partial_t f(A_t, X_t)\, \dd A_t \\
    & + \frac{1}{2} \int_0^T \partial_{xx} f(A_t, X_t) \, \dd \inner{X}{X}_t.
\end{split}
\end{equation}
In our setting, we take $X_t = |B_t|$ and $A_t = t$.
We now explain the notation $\inner{X}{X}$.
\begin{enumerate}[label=(\arabic*),noitemsep,topsep=0pt]
    \item For a continuous local martingale $M$, $\inner{M}{M}$ is the unique increasing continuous process vanishing at $0$ such that $M^2 - \inner{M}{M}$ is a martingale \cite[Theorem IV.1.8]{RY13}.
    \item If $X$ is a continuous semimartingale with $M$ being the (continuous) local martingale part
    then $\inner{X}{X} = \inner{M}{M}$ \cite[Definition IV.1.20]{RY13}.
\end{enumerate}
Tanaka's formula \cite[Theorem~IV.43.3]{RogersWilliamsII} asserts that $|B_t| = W_t + L_t$ where $W_t$ is a Brownian Motion
and $L_t$ is the local time of $B_t$ at 0, which is an increasing, continuous, adapated process.
Hence, $|B_t|$ is a semimartingale with $\inner{|B|}{|B|}_t = \inner{W}{W}_t = t$.
Plugging these into \Equation{RYIto} gives
\[
    f(T, \Abs{B_T}) - f(0, \Abs{B_0}) ~=~ \int_{0}^T \partial_x f(t, \Abs{B_t}) \dd \Abs{B_t} + \int_{0}^T \Big [
    \partial_t f(t, \Abs{B_t}) + \smallfrac{1}{2}   \partial_{xx}f(t, \Abs{B_t}) \Big] \dd t,
\]
which is what appears in \Theorem{ItoFormula}.

\subsection{Continuous regret against any continuous semi-martingale}
\AppendixName{cont_reg_semi_mart}
Recall that the continuous regret upper bound (\Theorem{ContsMainResult}) involved the adversary evolving the gap process as a reflected Brownian motion, which is a continuous semi-martingale. In this section, we generalize the definition of continuous regret to allow arbitrary, non-negative, continuous semi-martingales to control the gap process, and derive an analogue of \Theorem{ContsMainResult} in this generalized setting. We use the notation $[X]_t$ to refer to $\inner{X}{X}_t$, the quadratic variation process of $X$, which was introduced in \Appendix{ito_discussion}.

We begin with a generalized definition of continuous regret.

\begin{definition}[Continuous Regret]
\DefinitionName{ContRegret2}
Let $p : \bR_{>0} \times \bR_{\geq 0} \rightarrow [0,1]$ be a continuous function that satisfies $p(t,0) = 1/2$ for every $t > 0$.
Let $X_t$ be a continuous, non-negative, semi-martingale.
Then, the \emph{continuous regret} of $p$ with respect to $X$
is the stochastic integral
\begin{equation}
\EquationName{ContinuousRegret2}    
\ContReg{T,p,X} ~=~ \int_{0}^T p(t, X_t) \dd X_t.
\end{equation}  
\end{definition}

The main result for this generalized setting is as follows.

\begin{theorem}
\TheoremName{ContsMainResult2}
There exists a continuous-time algorithm $p^*$ such that for any continuous, non-negative, semi-martingale $X$,
\begin{equation}
    \EquationName{ContsMainResult2}
    \ContReg{T,p^*,X} ~\leq~ \frac{\gamma}{2}\sqrt{[X]_T}~
    \quad\forall T \in \bR_{\geq 0},
    ~\text{almost surely}.
 \end{equation}
\end{theorem}
We provide an overview of the proof of this result below. For the sake of exposition, we sketch the proof of \Theorem{ContsMainResult2} in the setting where we allow $p^*$ to take values in $(-\infty, 1].$ Truncating $p^*$ as was done in \Subsection{truncation} yields \Theorem{ContsMainResult2} as stated.

\begin{proofsketch}
Let $p^*(t,x) \coloneqq \tilde{p}_{\gamma}([X]_t, x)$ and $R(t,x) \coloneqq \tilde{R}_{\gamma}(t,x)$. (See Eq.~\Equation{palpha_def} and Eq.~\Equation{RtildeDef} for definitions of $\tilde{p}_{\gamma}$ and $\tilde{R}_{\gamma}).$ Recall the following three important properties of $R$ from \Lemma{RSatisfiesBHE}:
\begin{enumerate}[label=(\arabic*),noitemsep,topsep=0pt]
\item $R$ is $C^{1,2}$,
\item $R$ satisfies $\doob R = 0$ over $\bR_{>0} \times \bR$,
\item $\partial_x R(t,x) = \tilde{p}_{\gamma}(t,x)$.
\end{enumerate}
Since $R$ is $C^{1,2}$, we may apply It\^o's formula (specifically Eq.~\Equation{RYIto} with $A_t = [X]_t$, which is a bounded variation process since it is increasing) to obtain
\begin{align*}
    R([X]_T, X_T) 
        &~=~ \int_0^T \partial_x R([X]_t,X_t) \dd X_t + \int_0^T \partial_t R([X]_t,X_t) + \frac{1}{2} \partial_{xx} R([X]_t,X_t) \dd [X]_t \\
        &~=~ \int_0^T p^*(t,X_t) \dd X_t + \int_0^T \underbrace{\partial_t R([X]_t,X_t) + \frac{1}{2} \partial_{xx} R([X]_t,X_t)}_{= \doob R([X]_t, X_t)} \dd [X]_t \quad(\partial_x R = \tilde{p}_{\gamma})\\
        &~=~ \int_0^T p^*(t,X_t) \dd X_t \qquad(\doob R = 0)\\
        &~=~ \ContReg{T, p^*, X}.
\end{align*}  
Next, recall the upper bound on $R$ from Eq.~\Equation{R_bound}:
\[  R(t,x) ~=~ R_{\gamma}(t,x) ~\leq~ \left( \frac{\gamma}{2} + \kappa_{\gamma}M_0\left(  \frac{\gamma^2}{2} \right) \right)\sqrt{t} ~=~ \frac{\gamma}{2}\sqrt{t}, \]where the final equality is because $\gamma$ is a root of $M_0\left(\frac{x^2}{2}\right).$ Putting everything together, we have
\[  \ContReg{T, p^*, X} ~=~ R([X]_T, X_T) ~\leq~ \frac{\gamma}{2}\sqrt{[X]_T}, \]as desired.
\end{proofsketch}

%% file: app_oblivious.tex
\section{Remark on oblivious adversaries}
\AppendixName{oblivious}
In this section, we consider the following model.
At each time step $t$, the algorithm $\cA$ chooses a probability vector $x_t \in [0,1]^n$ and then draws a random expert $I_t \in [n]$ such that $\Pr[I_t = i] = x_{t,i}$.
The adversary $\cB$ then chooses a loss vector $\ell_t \in [0,1]^n$ given $\{x_s\}_{s \leq t}$, $\{\ell_s\}_{s \leq t}$, and $\{I_s\}_{s < t}$. (Crucially, $\cB$ does not know $I_t$ at time $t$.)
In this setting, we consider the following notion of regret defined as
\[
    \Regret{n, T, \cA, \cB} = \sum_{t=1}^T \ell_{t, I_t} - \min_{i \in [n]} \sum_{t=1}^T \ell_{t, i}.
\]
The following theorem shows that, with this definition, any algorithm must incur $\Omega(\sqrt{t \log \log t})$ anytime regret.
\begin{theorem}
    \TheoremName{oblivious_lb}
    For any algorithm $\cA$, there exists an adversary $\cB$ such that for all $T \geq 1$,
    \[
        \expect{ \sup_{t \geq T} \frac{\Regret{2, t, \cA, \cB}}{\sqrt{(t/2) \log \log (t/2)}}} \geq \frac{1}{2}.
    \]
\end{theorem}
For the rest of this section, we write $\Reg(T) = \Regret{2, T, \cA, \cB}$.
The adversary $\cB$ that achieves \Theorem{oblivious_lb} is extremely simple.
At time $t$, the adversary chooses an index $i^* \in \argmax_{i \in [2]} x_{t,i}$.\footnote{
    For concreteness, we break ties in lexicographical order but the exact tie-breaking does not play a role.}
It sets a cost of $1$ for expert $i^*$ and a cost of $0$ for expert $3-i^*$ (the other expert).
More precisely,
\[
    \ell_t =
    \begin{cases}
        \begin{smallbmatrix} 1 \\ 0 \end{smallbmatrix} & \text{if $x_{t,1} \geq x_{t,2}$} \\
        \begin{smallbmatrix} 0 \\ 1 \end{smallbmatrix} & \text{if $x_{t,1} < x_{t,2}$}
    \end{cases}.
\]

We now analyze this adversary.
To do so, we set up some notation that is reminiscent of that used in \Section{ub} and \Section{lb}.
Let $L_{t,1} = \sum_{s \leq t} \ell_{s,1}$ and $L_{t,2} = \sum_{s \leq t} \ell_{s,2}$.
Let $g_t = |L_{t,1} - L_{t,2}|$ be the gap between the cumulative losses of the two experts.
Note that $|g_t - g_{t-1}| = 1$ because for all $t$, exactly one of $\ell_{t,1}, \ell_{t,2}$ is equal to $1$ while the other is equal to $0$.
If $g_t = 0$, let $p_t = \max\{ x_{t,1}, x_{t,2} \}$ and if $g_t > 0$, let $p_t$ be the probability mass placed on the worst expert
(i.e.~the expert with the highest cumulative cost at time $t$).
More precisely, if $g_t > 0$, we set
\[
    p_t =
    \begin{cases}
        x_{t,1} & \text{if $L_{t,1} > L_{t,2}$} \\
        x_{t,2} & \text{if $L_{t,1} < L_{t,2}$}
    \end{cases}.
\]
Let $\Ber(p)$ denote the Bernoulli distribution with parameter $p$.
In particular, if $X \sim \Ber(p)$ then $\Pr[X = 1] = p$ and $\Pr[X=0] = 1-p$.
\begin{prop}
    \PropositionName{ObliviousRegretGap}
    Suppose that whenever $g_{t-1} = 0$, the adversary sets a loss of $1$ for an expert in $\argmax_{i \in [2]} \{x_{t,1}, x_{t,2}\}$ and a loss of $0$ for the other expert.
    Suppose further that $|g_t - g_{t-1}| = 1$ for all $t \geq 1$.
    Then $\Reg(T) = \sum_{t=1}^T (g_t - g_{t-1}) \cdot R_t$ where $R_t \sim \Ber(p_t)$.
\end{prop}
We remark that \Proposition{ObliviousRegretGap} does not use the adversary we have defined above and is true for any adversary that satifies the conditions of the hypothesis.
We also remark that the proof of \Proposition{ObliviousRegretGap} is nearly identical to the proof of \Proposition{RegretWithGaps}.
\begin{proof}
Define $\DeltaR(t) = \Regret{t}-\Regret{t-1}$.
The total cost of the best expert at time $t$ is $L_t^* \coloneqq \min \set{L_{t,1}, L_{t,2}}$.
The change in regret at time $t$ is the cost incurred by the algorithm
minus the change in the total cost of the best expert,
so $\DeltaR(t) = L_{I_t} - (L_t^* - L_{t-1}^*)$, where $I_t \in [2]$ indicates which expert was chosen by the algorithm at time $t$.

\paragraph{Case 1: $g_{t-1} \neq 0$.}
In this case, the best expert at time $t-1$ remains a best expert at time $t$. Note that this uses the assumption that $g_t - g_{t-1} \in \{\pm 1\}$ so $g_{t-1} \geq 1$.
If the worst expert incurs cost $1$ then with probability $p_t$ the algorithm follows the worst expert and incurs cost $1$ and with probability $1-p_t$,
the algorithm follows the best expert and incurs cost $0$.
In other words, the algorithm's cost is given by $R_t \sim \Ber(p_t)$.
On the other hand, the best expert incurs cost $0$, so $\DeltaR(t) = R_t$ and $g_t - g_{t-1} = 1$.

Next, if the best expert incurs cost $1$ then with probability $p_t$ the algorithm follows the worst expert and incurs cost $0$ and with probability $1-p_t$,
the algorithm follows the best expert and incurs cost $1$.
In this case, the algorithm's cost is $1-R_t$.
On the other hand, the best expert incurs cost $1$,
so $\DeltaR(t) = -R_t$ and $g_t-g_{t-1}=-1$.

For either choice of cost, we see that $\DeltaR(t) = R_t \cdot (g_t-g_{t-1})$.

\paragraph{Case 2: $g_{t-1} = 0.$}
Both experts are best, but one incurs no cost, so $L_t^* = L_{t-1}^*$.
Recall that $p_t = \max\{x_{t,1}, x_{t,2}\}$ and that we assume the adversary sets a loss of $1$ for an expert in $\argmax_{i \in [2]} \{x_{t,1}, x_{t,2}\}$.
Without loss of generality, we assume $p_t = x_{t,1}$.
Hence, the algorithm's cost is given by $R_t$; it is equal to $1$ with probability $p_t$ and and $0$ with probability $1-p_t$.
We conclude that $\Delta_R(t) = R_t = R_t \cdot (g_t - g_{t-1})$.
\end{proof}

For the remainder of this section, we work with the adversary that is described early in this section.
Recall that the adversary sets a loss of $1$ on an expert in $\argmax_{i \in [2]} \{x_{t,1}, x_{t,2}\}$ and a loss of $0$ on the other expert.
Before we proceed, we make a couple of simple observations.
First, the hypothesis of \Proposition{ObliviousRegretGap} holds and we make use of this below.
Second, we have that if $g_t = g_{t-1} + 1$ then $p_t \geq 1/2$ and if $g_t = g_{t-1} - 1$ then $p_t \leq 1/2$.
The observation is trivial when $g_{t-1} = 0$ since $g_t = 1$ and $p_t = \max\{x_{t,1}, x_{t,2}\} \geq 1/2$.
Now suppose that $g_{t-1} \geq 1$ and $g_t = g_{t-1} + 1$.
We claim that this implies $p_t \geq 1/2$.
For the sake of contradiction, suppose $p_t < 1/2$ (recall that $p_t$ is the mass on the \emph{worst} expert).
Then the adversary sets a loss of $1$ on the \emph{best} expert which \emph{decreases} the gap at from time $t-1$ to $t$ so that $g_{t} = g_{t-1} - 1$.
This contradicts that the gap increases from $t-1$ to $t$.
A similar argument shows that $g_t = g_{t-1} -1$ implies $p_t \leq 1/2$.

For notation, we also let $M_T = \card{\setst{t \in [T]}{g_t = g_{t-1} + 1}}$ be the number of times that the gap increases by time $T$
and $N_T = \card{\setst{t \in [T]}{g_t = g_{t-1} - 1}}$ be the number of times that the gap decreases by time $T$.
Note that at any time $T \geq 1$, we have $g_T = M_T - N_T$.
In particular, $M_T \geq N_T$ which implies the following proposition (since $M_T + N_T = T$ and $M_T, N_T$ are non-negative integers).
\begin{proposition}
    \PropositionName{MTNT}
    For any time $T \geq 1$, we have $M_T \geq \lceil T/2 \rceil$ and $N_T \leq \lfloor T/2 \rfloor$.
\end{proposition}

Next, we show that for this simple adversary, we have a simple lower bound on $\Reg(T)$.
\begin{claim}
    \ClaimName{ObliviousLB1}
    Let $X_1, X_2, \ldots$ and $Y_1, Y_2, \ldots$ be sequences of i.i.d.~$\Ber(1/2)$ random variables.
    There is a coupling between $\{\Reg(t)\}_{t \geq 1}$ and $\{(X_t), (Y_t)\}_{t \geq 1}$ such that for all $T \geq 1$,
    \begin{equation}
        \EquationName{ObliviousLB1}
        \Reg(T) \geq \sum_{t=1}^{M_T} X_t - \sum_{t=1}^{N_T} Y_t.
    \end{equation}
\end{claim}
\begin{proof}
    From \Proposition{ObliviousRegretGap}, we have that $\Reg(T) = \sum_{t=1}^T R_t \cdot (g_t - g_{t-1})$.
    We show that there is a coupling between $\{\Reg(t)\}_{t \geq 1}$ and $\{(X_t), (Y_t)\}_{t \geq 1}$ such that for all $t \geq 1$:
    \begin{enumerate}[noitemsep, topsep=0pt]
        \item if $g_t = g_{t-1} + 1$ then $R_t \cdot (g_t - g_{t-1}) = R_t \geq X_{M_t}$
        \item if $g_t = g_{t-1} - 1$ then $R_t \cdot (g_t - g_{t-1}) = -R_t \geq -Y_{N_t}$.
    \end{enumerate}
    Fix a $t \geq 1$.
    We start with the case where $g_t = g_{t-1} + 1$.
    As mentioned above, we have $p_t \geq 1/2$.
    To define the coupling, let $A_t \sim \Ber(1/2)$ and $B_t \sim \Ber(2p_t - 1)$ be independent.
    We then set $X_{M_t} = A_t$ and $R_t = A_t + (1-A_t) \cdot B_t$.
    Clearly, $R_t \cdot (g_t - g_{t-1}) = R_t \geq X_{M_t}$.
    So it remains to check that $R_t$ has the desired distribution.
    Indeed, $R_t = 0$ if and only if $A_t = 0$ and $B_t = 0$.
    So $\Pr[R_t = 0] = \Pr[A_t = 0] \Pr[B_t = 0] = 0.5 \cdot (2-2p_t) = 1-p_t$ and $\Pr[R_t = 1] = p_t$ as desired.

    Next, suppose $g_t = g_{t-1} - 1$ in which case $p_t \leq 1/2$.
    Let $A_t \sim \Ber(1/2)$ and $B_t \sim \Ber(2p_t)$ be independent.
    We set $Y_{N_t} = A_t$ and $R_t = A_t B_t$.
    Clearly, $R_t \leq Y_{N_t}$ (equivalently, $-R_t \geq -Y_{N_t}$).
    So it remains to check that $R_t$ has the desired distribution.
    Indeed, $R_t = 1$ if and only if $A_t = B_t = 1$.
    So $\Pr[R_t = 1] = \Pr[A_t = 1] \Pr[B_t = 1] = p_t$ and $\Pr[R_t = 0] = 1-p_t$ as desired.
\end{proof}
Although the RHS of \Equation{ObliviousLB1} seems simpler to work than $\Reg(T)$, one annoyance is that it still depends on how the adversary and the algorithm interact.
However, we can combine \Proposition{MTNT} and \Claim{ObliviousLB1} to establish a lower bound on $\Reg(T)$ which does \emph{not} depend on the interaction between the adversary and the algorithm.
\begin{claim}
    \ClaimName{ObliviousLB2}
    Let $X_1, X_2, \ldots$ and $Y_1, Y_2, \ldots$ be sequences of i.i.d.~$\Ber(1/2)$ random variables.
    There is a coupling between $\{\Reg(t)\}_{t \geq 1}$ and $\{(X_t), (Y_t)\}_{t \geq 1}$ such that for all $T \geq 1$,
    \[
        \Reg(T) \geq \sum_{t=1}^{\lceil T/2 \rceil} X_t - \sum_{t=1}^{\lfloor T/2 \rfloor} Y_t.
    \]
\end{claim}
\begin{proof}
    \Proposition{MTNT} shows that $M_T \geq \lceil T/2 \rceil$ and $N_T \leq \lfloor T/2 \rfloor$ while \Claim{ObliviousLB1} shows that
    \[
        \Reg(T) \geq \sum_{t=1}^{M_T} X_t - \sum_{t=1}^{N_T} Y_t.
    \]
    The claim follows from the fact that for $j \in \{1, 2, \ldots, M_T - \lceil T/2 \rceil\}$, we have $X_{\lceil T/2 \rceil + j} \geq -Y_{N_{T} + j}$.
\end{proof}
The following claim completes the proof of \Theorem{oblivious_lb} since
\[
    \expect{ \sup_{t \geq T} \frac{\Reg(t)}{\sqrt{(t/2) \log \log (t/2)}} }
    \geq \expect{ \sup_{t \geq T} \frac{ \sum_{s=1}^{\lceil t/2 \rceil} X_s - \sum_{s=1}^{\lfloor t/2 \rfloor} Y_s }{\sqrt{(t/2) \log \log(t/2)}}}.
\]
\begin{claim}
    Let $X_1, X_2, \ldots$ and $Y_1, Y_2, \ldots$ be sequences of i.i.d.~$\Ber(1/2)$ random variables.
    Then, for any $T \geq 1$, there is a stopping time $\tau \geq T$ such that
    \[
        \expect{\frac{\sum_{t=1}^\tau X_t - \sum_{t=1}^\tau Y_t}{\sqrt{\tau \log \log \tau}}} \geq 1/2.
    \]
\end{claim}
\begin{proof}
    Let $\tau = \inf\setst{t \geq T}{\sum_{s=1}^t (X_s-1/2) \geq \frac{1}{2} \sqrt{t \log \log t}}$.
    By the law of the iterated logarithm (\Theorem{LIL} below), $\tau$ is finite a.s.
    In addition, since $\{Y_t - 1/2\}_{t \geq 1}$ are mean-zero random variables that are independent of $\tau$, we have $\expect{ \frac{\sum_{t=1}^\tau (Y_t - 1/2)}{\sqrt{\tau \log \log \tau}} } = 0$.
    Hence,
    \begin{align*}
        \expect{\frac{\sum_{t=1}^\tau X_t - \sum_{t=1}^\tau Y_t}{\sqrt{\tau \log \log \tau}}}
        & =
        \expect{\frac{\sum_{t=1}^\tau (X_t-1/2)}{\sqrt{\tau \log \log \tau}}} \\
        & \geq
        \expect{\frac{\frac{1}{2} \sqrt{\tau \log \log \tau}}{\sqrt{\tau \log \log \tau}}} \\
        & = \frac{1}{2}. \qedhere
    \end{align*}
\end{proof}
\begin{theorem}[Law of the iterated logarithm {\protect \cite[Theorem 22.11]{Klenke}}]
    \TheoremName{LIL}
    Let $X_1, X_2, \ldots$ be i.i.d.~real random variables such that $\expect{X_1} = 0$ and $\Var{X_1} = 1$.
    Let $S_n = \sum_{i=1}^n X_i$ for $n \in \bN$.
    Then, almost surely,
    \[
        \limsup_{n \to \infty} \frac{S_n}{\sqrt{2n \log \log n}} = 1.
    \]
\end{theorem}

%% file: main.bbl
\begin{thebibliography}{10}

\bibitem{AbbasiBG17}
Yasin Abbasi{-}Yadkori, Peter~L. Bartlett, and Victor Gabillon.
\newblock Near minimax optimal players for the finite-time 3-expert prediction
  problem.
\newblock In {\em Advances in Neural Information Processing Systems 30: Annual
  Conference on Neural Information Processing Systems 2017, 4-9 December 2017,
  Long Beach, CA, {USA}}, pages 3033--3042, 2017.

\bibitem{AbernethyFW12}
Jacob~D. Abernethy, Rafael~M. Frongillo, and Andre Wibisono.
\newblock Minimax option pricing meets black-scholes in the limit.
\newblock In {\em Proceedings of the 44th Symposium on Theory of Computing
  Conference}, pages 1029--1040. {ACM}, 2012.

\bibitem{Abramowitz}
Milton Abramowitz and Irene~A. Stegun.
\newblock {\em Handbook of mathematical functions: with formulas, graphs, and
  mathematical tables}, volume~55.
\newblock Courier Corporation, 1965.

\bibitem{Andoni}
Alexandr Andoni and Rina Panigrahy.
\newblock A differential equations approach to optimizing regret trade-offs,
  May 2013.
\newblock arXiv:1305.1359.

\bibitem{AHK12}
Sanjeev Arora, Elad Hazan, and Satyen Kale.
\newblock The multiplicative weights update method: a meta-algorithm and
  applications.
\newblock {\em Theory of Computing}, 8(1):121--164, 2012.

\bibitem{BEZ20}
Erhan Bayraktar, Ibrahim Ekren, and Xin Zhang.
\newblock Finite-time 4-expert prediction problem.
\newblock {\em Communications in Partial Differential Equations}, pages 1--44,
  2020.

\bibitem{BEZ19}
Erhan Bayraktar, Ibrahim Ekren, and Yili Zhang.
\newblock On the asymptotic optimality of the comb strategy for prediction with
  expert advice.
\newblock {\em arXiv preprint arXiv:1902.02368}, 2019.

\bibitem{Breiman}
Leo Breiman.
\newblock First exit times for a square root boundary.
\newblock In {\em Proceedings of the Fifth Berkeley Symposium on Mathematical
  Statistics and Probability, Volume 2: Contributions to Probability Theory,
  Part 2}, pages 9--16. University of California Press, 1967.

\bibitem{BreimanBook}
Leo Breiman.
\newblock {\em Probability}.
\newblock SIAM, 1992.

\bibitem{BL02}
Monica Brezzi and Tze~Leung Lai.
\newblock Optimal learning and experimentation in bandit problems.
\newblock {\em Journal of Economic Dynamics and Control}, 27(1):87--108, 2002.

\bibitem{BubeckNotes}
S\'{e}bastien Bubeck.
\newblock Introduction to online optimization, December 2011.
\newblock unpublished.

\bibitem{Cesa99}
Nicol{\`{o}} Cesa-Bianchi.
\newblock Analysis of two gradient-based algorithms for on-line regression.
\newblock {\em Journal of Computer and System Sciences}, 59(3):392--411, 1999.

\bibitem{CFHHSW97}
Nicol{\`{o}} Cesa-Bianchi, Yoav Freund, David Haussler, David~P. Helmbold,
  Robert~E. Schapire, and Manfred~K. Warmuth.
\newblock How to use expert advice.
\newblock {\em Journal of the ACM (JACM)}, 44(3):427--485, 1997.

\bibitem{CBL}
{Nicol\`{o}} Cesa-Bianchi and G\'{a}bor Lugosi.
\newblock {\em Prediction, learning, and games}.
\newblock Cambridge University Press, 2006.

\bibitem{NormalHedge}
Kamalika Chaudhuri, Yoav Freund, and Daniel~J. Hsu.
\newblock A parameter-free hedging algorithm.
\newblock In {\em Advances in Neural Information Processing Systems 22}, pages
  297--305, 2009.

\bibitem{Chernoff}
Herman Chernoff.
\newblock Optimal stochastic control.
\newblock {\em {Sankhy{\={a}}: The Indian Journal of Statistics, Series A}},
  30:221--252, 1968.

\bibitem{CLRS}
Thomas~H. Cormen, Charles~E. Leiserson, Ronald~L. Rivest, and Clifford Stein.
\newblock {\em Introduction to Algorithms}.
\newblock MIT Press, third edition, 2009.

\bibitem{Cover66}
Thomas~M. Cover.
\newblock Behavior of sequential predictors of binary sequences.
\newblock In {\em Proceedings of the 4th Prague Conference on Information
  Theory, Statistical Decision Functions, Random Processes}. Publishing House
  of the Czechoslovak Academy of Sciences, Prague, 1965.

\bibitem{Davis}
Burgess Davis.
\newblock On the intergrability of the martingale square function.
\newblock {\em Israel Journal of Mathematics}, 8:187--190, 1970.

\bibitem{Davis76}
Burgess Davis.
\newblock On the {$L_p$} norms of stochastic integrals and other martingales.
\newblock {\em Duke Math. J}, 43(4):697--704, 1976.

\bibitem{DeMarzoKM06}
Peter~M. DeMarzo, Ilan Kremer, and Yishay Mansour.
\newblock Online trading algorithms and robust option pricing.
\newblock In {\em Proceedings of the 38th Annual {ACM} Symposium on Theory of
  Computing}, pages 477--486. {ACM}, 2006.

\bibitem{Donchev}
Doncho~S. Donchev.
\newblock Brownian motion hitting probabilities for general two-sided
  square-root boundaries.
\newblock {\em Methodology and Computing in Applied Probability}, 12:237--245,
  2010.

\bibitem{DoobBook}
J.~L. Doob.
\newblock {\em Classical Potential Theory and Its Probabilistic Counterparts}.
\newblock Springer-Verlag, 1984.

\bibitem{drenska}
Nadeja Drenska.
\newblock {\em A {PDE} approach to a Prediction Problem Involving Randomized
  Strategies}.
\newblock PhD thesis, New York University, 2017.

\bibitem{DK20}
Nadejda Drenska and Robert~V Kohn.
\newblock Prediction with expert advice: A {PDE} perspective.
\newblock {\em Journal of Nonlinear Science}, 30(1):137--173, 2020.

\bibitem{Durrett}
Rick Durrett.
\newblock {\em Probability: Theory and Examples}.
\newblock Cambridge University Press, fifth edition, 2019.

\bibitem{Feller}
William Feller.
\newblock {\em An Introduction to Probability Theory and Its Applications}.
\newblock John Wiley \& Sons, second edition, 1971.

\bibitem{Freund09}
Yoav Freund.
\newblock A method for {H}edging in continuous time.
\newblock {\em arXiv preprint arXiv:0904.3356}, 2009.

\bibitem{Fujita08}
Takahiko Fujita.
\newblock A random walk analogue of {L\'{e}vy's Theorem}.
\newblock {\em Studia Scientiarum Mathematicarum Hungarica}, 45(2):223--233,
  2008.

\bibitem{Ger11}
S\'{e}bastien Gerchinovitz.
\newblock {\em Prediction of individual sequences and prediction in the
  statistical framework: some links around sparse regression and aggregation
  techniques}.
\newblock PhD thesis, Universit\'{e} Paris-Sud, 2011.

\bibitem{GKP}
Ronald~L. Graham, Donald~E. Knuth, and Oren Patashnik.
\newblock {\em Concrete Mathematics}.
\newblock Addison-Wesley, second edition, 1994.

\bibitem{gravin2016towards}
Nick Gravin, Yuval Peres, and Balasubramanian Sivan.
\newblock Towards optimal algorithms for prediction with expert advice.
\newblock In {\em Proceedings of the twenty-seventh annual ACM-SIAM symposium
  on Discrete algorithms}, pages 528--547. SIAM, 2016.

\bibitem{GPS17}
Nick Gravin, Yuval Peres, and Balasubramanian Sivan.
\newblock {Tight Lower Bounds for Multiplicative Weights Algorithmic Families}.
\newblock In {\em 44th International Colloquium on Automata, Languages, and
  Programming (ICALP 2017)}, volume~80, pages 48:1--48:14, 2017.

\bibitem{GreenwoodPerkins}
Priscilla Greenwood and Edwin Perkins.
\newblock A conditioned limit theorem for random walk and brownian local time
  on square root boundaries.
\newblock {\em Annals of Probability}, 11:227--261, 1983.

\bibitem{GS}
Geoffrey Grimmett and David Stirzaker.
\newblock {\em Probability and Random Processes}.
\newblock Oxford University Press, third edition, 2001.

\bibitem{Hannan57}
James Hannan.
\newblock Approximation to {B}ayes risk in repeated play.
\newblock {\em Contributions to the Theory of Games}, 3:97--139, 1957.

\bibitem{KP17}
Anna~R. Karlin and Yuval Peres.
\newblock {\em Game Theory, Alive}.
\newblock American Mathematical Society, 2017.

\bibitem{Klenke}
Achim Klenke.
\newblock {\em Probability Theory: A Comprehensive Course}.
\newblock Springer, 2008.

\bibitem{KKW19a}
Vladimir~A. Kobzar, Robert~V. Kohn, and Zhilei Wang.
\newblock New potential-based bounds for prediction with expert advice.
\newblock {\em arXiv preprint arXiv:1911.01641}, 2019.

\bibitem{KKW19b}
Vladimir~A. Kobzar, Robert~V. Kohn, and Zhilei Wang.
\newblock New potential-based bounds for the geometric-stopping version of
  prediction with expert advice.
\newblock {\em arXiv preprint arXiv:1912.03132}, 2019.

\bibitem{Kudzma82}
R.~Kud\u{z}ma.
\newblock Ito's formula for a random walk.
\newblock {\em Lithuanian Mathematical Journal}, 22:302--306, 1982.

\bibitem{LW94}
Nick Littlestone and Manfred~K. Warmuth.
\newblock The weighted majority algorithm.
\newblock {\em Information and computation}, 108(2):212--261, 1994.

\bibitem{LuoSchapire14}
Haipeng Luo and Robert~E. Schapire.
\newblock Towards minimax online learning with unknown time horizon.
\newblock In {\em Proceedings of ICML}, 2014.

\bibitem{AdaNormalHedge}
Haipeng Luo and Robert~E. Schapire.
\newblock Achieving all with no parameters: {A}da{N}ormal{H}edge.
\newblock In {\em Proceedings of The 28th Conference on Learning Theory},
  volume~40, pages 1286--1304, 2015.

\bibitem{Nesterov09}
Yurii Nesterov.
\newblock Primal-dual subgradient methods for convex problems.
\newblock {\em Mathematical Programming}, 120(1):221--259, 2009.

\bibitem{Perkins}
Edwin Perkins.
\newblock On the {H}ausdorff dimension of the {B}rownian slow points.
\newblock {\em Zeitschrift f\"{u}r Wahrscheinlichkeitstheorie und Verwandte
  Gebiete}, 64:369--399, 1983.

\bibitem{Peskir}
Goran Peskir and Albert Shiryaev.
\newblock {\em Optimal Stopping and Free-Boundary Problems}.
\newblock Birkh{\"{a}}user Verlag, 2006.

\bibitem{RY13}
Daniel Revuz and Marc Yor.
\newblock {\em Continuous martingales and Brownian motion}, volume 293.
\newblock Springer Science \& Business Media, 2013.

\bibitem{RogersWilliamsI}
L.~C.~G. Rogers and David Williams.
\newblock {\em Diffusions, Markov Processes and Martingales. Volume 1:
  Foundations}.
\newblock Cambridge University Press, second edition, 2000.

\bibitem{RogersWilliamsII}
L.~C.~G. Rogers and David Williams.
\newblock {\em Diffusions, Markov Processes and Martingales. Volume 2: It{\^o}
  Calculus}, volume~2.
\newblock Cambridge University Press, second edition, 2000.

\bibitem{RudinPrinciples}
Walter Rudin.
\newblock {\em Principles of Mathematical Analysis}.
\newblock John Wiley \& Sons, third edition, 1976.

\bibitem{SS}
Shai Shalev-Shwartz.
\newblock {\em Online Learning and Online Convex Optimization}.
\newblock Foundations and Trends in Machine Learning, 2011.

\bibitem{Shepp}
L.~A. Shepp.
\newblock A first passage problem for the {W}iener process.
\newblock {\em The Annals of Mathematical Statistics}, 38(6):1912--1914, 1967.

\bibitem{Vovk90}
Volodimir~G. Vovk.
\newblock Aggregating strategies.
\newblock {\em Proc. of Computational Learning Theory, 1990}, 1990.

\bibitem{Williams}
David Williams.
\newblock {\em Probability with Martingales}.
\newblock Cambridge University Press, 1991.

\end{thebibliography}
